\documentclass[conference,compsoc]{IEEEtran}

\usepackage{threeparttable}
\usepackage{algorithm,algorithmicx,algpseudocode}
\usepackage{xspace}
\usepackage{tablefootnote}
\usepackage{amsthm}
\usepackage{graphicx}
\usepackage{textcomp}
\usepackage{xcolor}
\usepackage{booktabs}
\usepackage{multirow}
\usepackage{diagbox}
\usepackage{amsmath}
\usepackage{hyperref}
\newtheorem{theorem}{Theorem}
\newtheorem{definition}{Definition}
\newtheorem{lemma}{Lemma}
\newcommand{\framework}{ObjectSeeker\xspace}

\usepackage{amssymb}
\usepackage{pifont}
\newcommand{\cmark}{\ding{51}}%
\newcommand{\xmark}{\ding{55}}%

\algnewcommand{\LeftCommenta}[1]{\Statex \hspace{1.3em} \(\triangleright\) #1}
\algnewcommand{\LeftCommentb}[1]{\Statex \hspace{2.7em} \(\triangleright\) #1}

\definecolor{darkgreen}{RGB}{25, 136, 46}
\newcommand{\st}{\mathrm{~~s.t.~~}}

\newcommand{\cA}{{\mathcal A}}
\newcommand{\cB}{{\mathcal B}}

\newcommand{\cM}{{\mathcal M}}

\newcommand{\cO}{{\mathcal O}}
\newcommand{\cP}{{\mathcal R}}

\newcommand{\cX}{{\mathcal X}}
\newcommand{\cY}{{\mathcal Y}}

\newcommand{\bfb}{\mathbf{b}}

\newcommand{\bfm}{\mathbf{m}}

\newcommand{\bfp}{\mathbf{r}}

\newcommand{\bfx}{\mathbf{x}}

\newcommand{\bfbgt}{\mathbf{b}_{\text{gt}}}
\newcommand{\bfbm}{\mathbf{b}_{\text{m}}}
\newcommand{\bfbb}{\mathbf{b}_{\text{b}}}
\newcommand{\gammab}{\gamma_{\text{b}}}
\newcommand{\gammam}{\gamma_{\text{m}}}
    
\begin{document}

\pagestyle{plain}

\title{\framework: Certifiably Robust Object Detection against Patch Hiding Attacks via Patch-agnostic Masking}
%
\author{\IEEEauthorblockN{Chong Xiang}
\IEEEauthorblockA{
Princeton University\\
cxiang@princeton.edu}
\and
\IEEEauthorblockN{Alexander Valtchanov}
\IEEEauthorblockA{
Princeton University\\
alexvaltchanov@princeton.edu}
\and
\IEEEauthorblockN{Saeed Mahloujifar}
\IEEEauthorblockA{
Princeton University\\
sfar@princeton.edu}
\and
\IEEEauthorblockN{Prateek Mittal}
\IEEEauthorblockA{
Princeton University\\
pmittal@princeton.edu}
\and
}

\maketitle

\begin{abstract}

Object detectors, which are widely deployed in security-critical systems such as autonomous vehicles, have been found vulnerable to \textit{patch hiding attacks}. An attacker can use a single physically-realizable adversarial patch to make the object detector miss the detection of victim objects and undermine the functionality of object detection applications. 
In this paper, we propose \framework for certifiably robust object detection against patch hiding attacks. The key insight in \framework is \textit{patch-agnostic masking}: we aim to mask out the entire adversarial patch without knowing the shape, size, and location of the patch. This masking operation neutralizes the adversarial effect and allows \emph{any} vanilla object detector to safely detect objects on the masked images. 
Remarkably, we can evaluate \framework's robustness in a certifiable manner: we develop a certification procedure to formally determine if \framework can detect certain objects against \textit{any white-box adaptive} attack within the threat model, achieving certifiable robustness. Our experiments demonstrate a significant ($\sim$10\%-40\% absolute and $\sim$2-6$\times$ relative) improvement in certifiable robustness over the prior work, as well as high clean performance ($\sim$1\% drop compared with undefended models).\footnote{Our source code is available at \url{https://github.com/inspire-group/ObjectSeeker}.}

\end{abstract}

\section{Introduction}\label{sec-introduction}

\begin{figure*}[t]
    \centering
    \includegraphics[width=\linewidth]{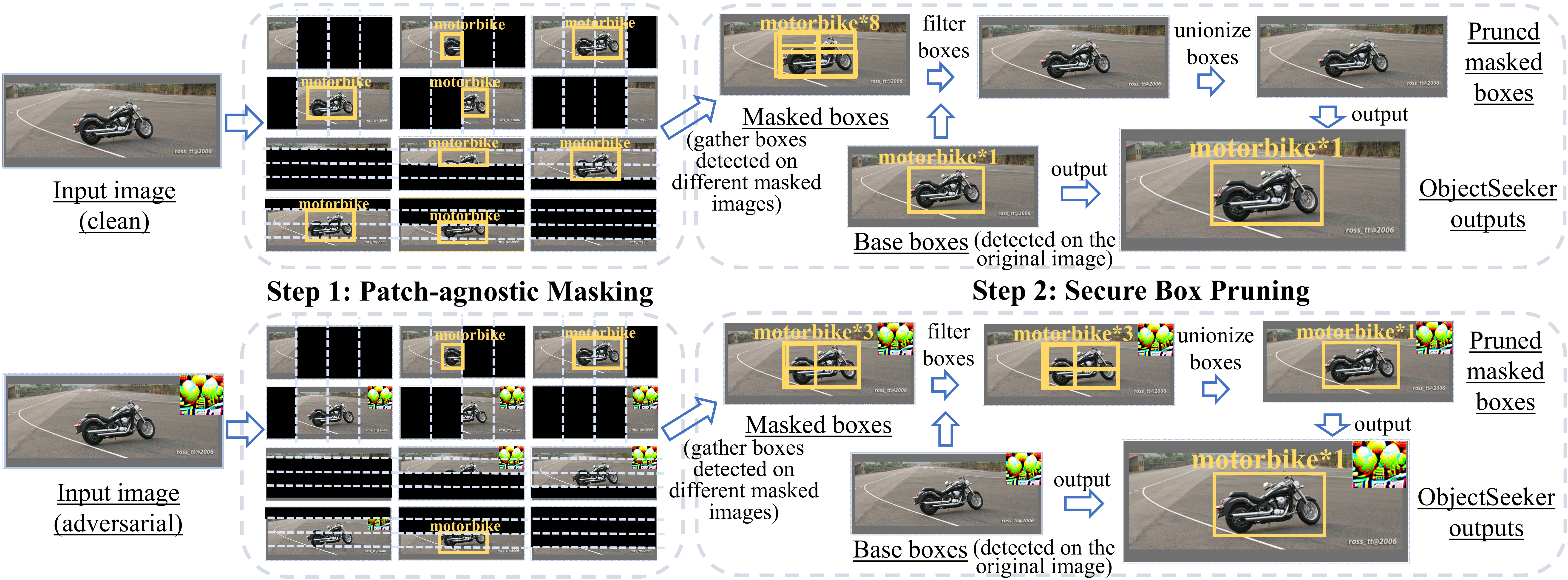}
    \vspace{-2em}
    \caption{\textbf{\framework Overview.}\textmd{ \textit{Step 1: patch-agnostic masking} -- we mask out image halves divided by a set of vertical and horizontal lines, and perform vanilla object detection on the remaining halves. This masking operation does not need to know the patch shape, size, or location. \textit{Step 2: secure box pruning} -- we gather boxes detected on different masked images (\textit{masked boxes}), use a ``robust" similarity score to \textit{filter} redundant masked boxes that are duplicates of boxes detected on the original image (\textit{base boxes}), and then \textit{unionize} the remaining masked boxes (if any). Finally, we combine pruned masked boxes and base boxes as the output. In the \textit{clean} setting (top of the figure), the motorbike is detected by one base box and eight masked boxes. All eight masked boxes are pruned, and we finally output the base box. In the \textit{adversarial} setting (bottom of the figure), the motorbike is only detected by three masked boxes (but not base box). No masked box is filtered; we unionize three masked boxes into one and output it as the robust prediction.}}
    \label{fig-overview}
\end{figure*}

Object detectors, which aim to output a list of bounding boxes detecting every object in a given image, have been shown vulnerable to\textit{ patch hiding attacks}~\cite{zhao2019seeing,thys2019fooling,xu2020adversarial,wu2019making,hu2021naturalistic,tan2021legitimate}. The patch hiding attack aims to make object detectors miss the detection of victim objects (e.g., making an autonomous vehicle miss a pedestrian) using an adversarial \textit{patch}~\cite{brown2017adversarial}. This attack can be realized in the physical world by attaching the adversarial patch to the real-world scene and thus imposes a notable threat to object detection applications like autonomous driving, augmented reality, and face recognition.

Unfortunately, strong robustness against patch hiding attacks has been hard to obtain. Most existing defenses~\cite{saha2020role,metzen2021meta,ji2021adversarial,liang2021we,chiang2021adversarial} are based on heuristics and lack formal security guarantees; their claimed robustness could be violated by stronger adaptive attacks. Xiang et al.~\cite{xiang2021detectorguard} designed the only certifiably robust defense against patch hiding attacks; nevertheless, the proposed DetectorGuard~\cite{xiang2021detectorguard} defense can only provide certifiable robustness for a small number of objects (e.g., $\sim$30\% of the objects from the VOC~\cite{voc} dataset and $\sim$10\% of the COCO~\cite{coco} objects even when the attacker is restricted to place a single 1\%-pixel patch \textit{far away} from the victim objects). In this paper, we propose a new approach called \framework, which can provide strong certifiable robustness for a significantly larger number of objects (e.g., $\sim$2-6$\times$ improvements over DetectorGuard in our experiments).

\textbf{\framework overview.} 
The core idea of \framework is to
apply pixel masks to the input image and perform object detection with a vanilla undefended model on the masked images. If the adversarial patch is masked, the attacker has no malicious influence over the model, and any vanilla object detector can safely make predictions. 
We focus on the setting that the patch hiding attacker can use one adversarial patch whose shape, size, and location are unknown to the defender. We aim to solve two major research questions: (1) How can we generate masks that can remove the patch without knowing the patch shape, size, and location as a priori (i.e., patch-agnostic)? (2) How can we achieve certifiable robustness while maintaining high clean performance? Our \framework solution involves two steps of \textit{patch-agnostic masking} and \textit{secure box pruning}; we provide a defense overview in Figure~\ref{fig-overview}.

\textbf{Step 1: patch-agnostic masking.} First, \framework applies a set of pixel masks to the input image and performs object detection on different masked images (left of Figure~\ref{fig-overview}). The mask set needs to satisfy two properties. First, some masks from the mask set need to remove the entire adversarial patch (i.e., patch-removing). Second, the mask set generation should not depend on any information on patch shape, size, and location (i.e., patch-agnostic). Once the mask set is generated, this fixed mask set should satisfy the patch-removing property for a patch of different shapes, sizes, and locations.

To attain these two properties, we propose to mask out \textit{image halves}: we divide the image into halves using a fixed set of vertical and horizontal lines, mask out each image half, and perform object detection on the other half.
Then, for a patch of any shape, size, and location, we can find lines that do not intersect with the patch to generate image halves that can mask out the entire patch. Intuitively, a vanilla object detector is likely to detect the object from the remaining benign pixels, as long as the patch is reasonably sized and does not corrupt the entire victim object. This lays a foundation for robust object detection.

\textbf{Step 2: secure box pruning.} Despite its robustness property, the pixel masking operation can introduce duplicate boxes and downgrade the precision of the model prediction; this is because an object can be detected multiple times on different masked images (e.g., the motorbike in Figure~\ref{fig-overview} is detected multiple times). The second step of \framework aims to remove these redundant boxes for precise object detection outputs while preserving the robustness property achieved by the masking operation. Towards this goal, we use a ``robust" box similarity score to identify similar boxes and prune redundant boxes accordingly. This process is illustrated in the right of Figure~\ref{fig-overview} and will be discussed in Section~\ref{sec-defense-merge}.

\textbf{Certifiable robustness evaluation.} Remarkably, our \framework design allows us to develop a certification procedure for formal robustness evaluation. The procedure will certify if \framework can robustly detect a given ground-truth object against any patch attacker within a  given threat model (e.g., any attacker placing a 1\%-pixel square patch anywhere on the image). Our formal proof in Section~\ref{sec-defense-certification} ensures that the certification results hold for any adaptive attack within the same threat model and thus allows us to evaluate robustness with absolute certainty.

\textbf{State-of-the-art certified robustness across datasets.} We implement \framework with two vanilla object detectors: YOLOR~\cite{yolor} and Mask~R-CNN~\cite{he2017mask} with a Swin Transformer backbone~\cite{liu2021swin}, and evaluate it on three object detection datasets: VOC~\cite{voc}, COCO~\cite{coco}, and KITTI~\cite{kitti}. We demonstrate that \framework achieves significant ($\sim$10\%-40\% absolute and $\sim$2-6$\times$ relative) improvements in certified robustness over DetectorGuard~\cite{xiang2021detectorguard}. In the meantime, \framework also has high clean performance similar to that of vanilla models ($\sim$1\% clean performance drops). Our contributions can be summarized as follows.
\begin{itemize}\setlength\itemsep{0em}
    \item We propose \framework as a certifiably robust defense framework for any vanilla object detector via a patch-agnostic masking strategy.
    \item We develop a certification procedure to determine if \framework is provably robust, for a given object, against any attack within a given threat model.
    \item We evaluate \framework on object detection benchmark datasets and demonstrate significant robustness improvements over DetectorGuard~\cite{xiang2021detectorguard}.
\end{itemize}

\section{Background and Problem Formulation}

In this section, we introduce the task of object detection, attack formulation, and defense formulation.

\subsection{Object Detection}\label{sec-formulation-obj-detection}
The task of object detection is to predict a list of bounding boxes that locate and classify each object in a given image. We use $\bfx \in \cX \subset [0,1]^{W\times H\times C}$ to denote the input image with width $W$, height $H$, and $C$ color channels. 
We use $\bfb = (x_{\min},y_{\min},x_{\max},y_{\max},\ell,c)$ to represent a bounding box, where four coordinates $(x_{\min},y_{\min},x_{\max},y_{\max})$ illustrate the box, $\ell$ denotes the class label, and $c\in[0,1]$ denotes the prediction confidence for this box detection. An object detector takes an image $\bfx$ as the input and outputs a list of boxes $\bfb_0,\bfb_1,\cdots,\bfb_n$. We further use an unordered set $\cB=\{\bfb_0,\bfb_1,\cdots,\bfb_n\}$ to denote the detection output and let $\cO$ denote the space of all possible $\cB$ (detection outputs). We can then formally represent an object detector as $\mathbb{F}(\bfx,\gamma):\cX\times[0,1]\rightarrow \cO$, where the detector only outputs boxes with confidence higher than $\gamma$. 
We do not make any assumption on the object detector $\mathbb{F}$, it can be any off-the-shelf model such as YOLO~\cite{redmon2016you,bochkovskiy2020yolov4,yolor}, Faster R-CNN~\cite{ren2015faster}, and Mask R-CNN~\cite{he2017mask}. 

\textbf{Box operation.} In some cases, we abuse the notation $\bfb$  by considering $\bfb$ as a set of image pixels within the box. Then, we can define several important box operations: $|\bfb|$ denotes the \textit{area} of the box $\bfb$; $\bfb_0\cap\bfb_1$ and $\bfb_0\cup\bfb_1$ denote the \textit{intersection} and \textit{union} of two boxes, respectively; $\bfb_0\setminus\bfb_1$ denotes the image region that belongs to $\bfb_0$ but not $\bfb_1$. We only consider operations between boxes when they share the same class labels. When $\bfb_0,\bfb_1$ have different class labels, $\bfb_0\cap\bfb_1=\varnothing$; $\bfb_0\cup\bfb_1$ and $\bfb_0\setminus\bfb_1$ are undefined.

\textbf{Performance evaluation.} To evaluate a conventional object detector, we consider a detected box is correct if (1) the box class label matches the ground-truth box label, and (2) Intersection over Union (IoU)  between the detected box and the ground-truth box, defined as $|\bfb_0\cap\bfb_1|/|\bfb_0\cup\bfb_1|$, exceeds a certain threshold~\cite{voc,coco}. We count a correct detected box as a true-positive (TP). On the other hand, any detected box that is not a TP is a false-positive (FP); any ground-truth box that is not correctly detected is a false-negative (FN). Furthermore, we define \textit{precision} as TP/(TP+FP) and \textit{recall} as TP/(TP+FN). We aim to build object detectors that have both high precision and recall.

\subsection{Attack Formulation}\label{sec-formulation-attack}
We focus on defending against patch hiding attacks.

\textbf{Objective: hiding attacks.}  The hiding attack aims to make the object detector miss the detection of the victim object, which increases FN errors and downgrades the detection recall. This attack can cause serious consequences in real-world scenarios such as an autonomous car failing to detect and consequently hitting a pedestrian.

\textbf{Means: patch attacks.} The patch attack aims to generate adversarial pixels within a local region (forming a patch) to cause mispredictions of the machine learning model. This attack can be realized in the physical world by printing and attaching the patch to the underlying scene and thus imposes an urgent threat to real-world machine learning applications. Formally, we denote the patch region with a binary tensor $\bfp\in\{0,1\}^{W\times H}$ of the same shape as the $W\times H$ input image $\bfx$. We set the elements within the patch region to zeros, and the rest to ones. We further use $\cP$ to denote a set of patch regions $\bfp$ that an attacker can use. Then, the constraint set of the patch attack can be represented as $\cA_\cP(\bfx)=\{\bfx^\prime = \bfx\odot\bfp+\bfx^{\prime\prime}\odot(\mathbf{1}-\bfp)|\bfx\in\cX,\bfp\in\cP\}$. $\odot$ refers to the element-wise product operator. $\bfx^{\prime\prime}\in[0,1]^{W\times H\times C}$ is the patch content \textit{arbitrarily} controlled by the attacker.

\textbf{Threat model.} We note that the patch region set $\cP$ is determined by the number of patches, patch shape, patch size, and patch location. 
In this paper, we primarily focus on the open research question of \textit{one} adversarial patch; nevertheless, we will quantitatively discuss defenses against multiple patches in Section~\ref{sec-discussion-limitation}. 

Furthermore, we allow the patch to have different \textit{shapes} such as square, rectangle, and circle. 
The patch \textit{size} can take any \textit{reasonable} value as long as the patch does not occlude the entire object; otherwise, the defense problem is meaningless since no one (not even a human) can detect a completely invisible object. 

Finally, we divide all possible patch \textit{locations} into three groups: {far-patch}, {close-patch}, and {over-patch} when the patch is far away from, close to, and over the object.  The attacker can pick \textit{any} location within a certain location set. These different location threat models represent different attackers in the physical world: for example, whether being able to place the patch over the victim object or not clearly exhibits different attacker capabilities. 

In our evaluation, we will consider different patch region sets $\cP$ such as a set of single 2\%-pixel square patches at all possible locations over the victim object or single 1\%-pixel rectangle patches anywhere far away from the object.


\subsection{Defense Formulation}\label{sec-formulation-defense}

In this section, we formulate the defense problem.

\textbf{Defender knowledge (patch-agnostic).} As discussed in Section~\ref{sec-formulation-attack}, we consider patch hiding attacks that use one adversarial patch. The defender only knows that there might be one patch on the input image, but does not know anything about the patch shape, size, and location (i.e., patch-agnostic). This requires the defender to build a defense with a \textit{fixed} set of parameters to deal with \textit{different} patch attackers that use different patch region sets $\cP$.

\textit{Note:} The assumption of a \textit{reasonable patch size} discussed in the last subsection will not give an additional advantage to our defense or violate the patch-agnostic property: it only ensures the object visibility that is required by any object detector (including humans).

\textbf{Robustness definition.} Facing a patch hiding attack, our defense aims to robustly predict bounding boxes that \textit{cover at least part of each object and have correct class labels}. To formally discuss the robustness objective, we introduce the following concept of Intersection over Area (IoA)

\begin{definition}[Intersection over Area (IoA)]
The IoA between two boxes $\bfb_0$ and $\bfb_1$ is the ratio of the intersecting area of two regions to the area of the first box $\bfb_0$. Formally, we have: $\textsc{IoA}(\bfb_0,\bfb_1) = {|\bfb_0\cap \bfb_1|}/{|\bfb_0|}$.
\end{definition}

\noindent We note that when two boxes have different class labels, we consider the intersection $\bfb_0\cap\bfb_1$ empty and the IoA is 0. 

With the concept of IoA, we can then reiterate the robustness definition as: given a ground-truth box $\bfbgt$, we aim to detect a box $\bfb^\prime$ such that $\textsc{IoA}(\bfbgt,\bfb^\prime)>T$, where $T\in[0,1]$ determines the strength of robustness. We term this as \textit{IoA robustness}. We choose this robustness objective because it aligns with our goal of mitigating hiding attacks; this objective is also similar to that of DetectorGuard~\cite{xiang2021detectorguard} and enables a fair comparison. In Appendix~\ref{apx-iou} and \ref{apx-taxonomy}, we provide discussions on other possible robustness definitions. 

\textbf{Certifiable robustness evaluation.} More importantly, we target \textit{certifiable robustness} -- evaluating robustness in a certifiable manner. We will develop a certification procedure to determine if \framework can robustly detect a given ground-truth object against a given threat model. We will formally prove that the certification procedure accounts for all possible attackers within the given threat model $\cA_\cP$, including an adaptive attacker with perfect knowledge of our defense setup. In our evaluation, we apply the certification procedure to datasets with annotated ground-truth bounding boxes (objects) and use \textit{certified recall}, the fraction of certified objects among all ground-truth objects, as our robustness metric. 

\textbf{Remark: patch-agnostic inference vs. certification.} We note that our patch-agnostic property is only for inference but not certification. The inference procedure (Algorithm~\ref{alg-inference}) is the defense that will be \textit{deployed in practice}. Patch-agnostic inference allows us to use the {same} inference setting to achieve non-trivial robustness against {different} patch shapes/sizes/locations (different patch region sets $\cP$). In contrast, the certification procedure (Algorithm~\ref{alg-certification}) is for \textit{evaluating robustness} against \textit{one} specific patch threat model $\cA_\cP$, and thus requires patch information. In fact, “patch-agnostic certification” is impossible: we cannot certify robustness against an unknown/unbounded attack capability.



\begin{table}[t]
    \centering
    \caption{Summary of important notation}
    \vspace{-1em}
 \resizebox{\linewidth}{!}
  { \begin{tabular}{l|l|l|l}
    \toprule
    \textbf{Notation} & \textbf{Description} & \textbf{Notation} & \textbf{Description} 
    \\
    \midrule
 $\mathbf{x}\in\mathcal{X}$ & input image &$\bfb\in\cB\in\cO$ & bounding box \\
 $\mathbf{m}\in\mathcal{M}$ & pixel mask &$\bfp\in\cP$ & patch region \\
     $\mathbb{F}:\cX\times[0,1]\rightarrow\cB$ &undefended model& $\gamma\in[0,1]$ & confidence thres.\\
     $\mathbb{S}:\cB\times\cB\rightarrow\mathbb{R}$&similarity score&$\tau\in[0,1]$&filtering thres.\\
     $k\in\mathbb{Z}^+$&\# lines&$T\in[0,1]$ &certification thres.\\
      \bottomrule
    \end{tabular}}
    \label{tab-notation}
\end{table}

\section{\framework Design}\label{sec-defense}
\textbf{Overview.} In this section, we introduce the design of \framework, a defense framework for building certifiably robust object detectors against patch hiding attacks. \framework operates in a two-step manner (recall the defense overview in Figure~\ref{fig-overview}). First, \framework applies a set of pixel masks to the input images and performs vanilla object detection on masked images (Section~\ref{sec-defense-mask}). Our masking strategy ensures that some of the masks remove the entire adversarial patch so that a vanilla object detector can detect victim objects from the unmasked regions of these images. This masking operation lays a foundation for robustness; however, it may introduce redundant boxes when the same object is repeatedly detected on multiple masked images. The second step of \framework aims to remove these redundant boxes via secure box pruning (Section~\ref{sec-defense-merge}); this will improve the detection precision while preserving the robustness property achieved by the masking operation. Notably, with a careful design of pixel masking and box pruning, we can develop a certification procedure to determine if  \framework has provable robustness guarantees for certain objects against any adaptive attacker within a given threat model (Section~\ref{sec-defense-certification}), achieving certifiable robustness.

We provide the pseudocode of \framework in Algorithm~\ref{alg-inference} and a summary of important notation in Table~\ref{tab-notation}. We will discuss the details of each defense module next.

\subsection{Patch-agnostic Masking}\label{sec-defense-mask}

\textbf{Intuition.} 
The objective of our pixel-masking defense is to find masks to mask out the entire patch (but not the entire victim objects) from the input image. Typically, vanilla object detectors can safely detect objects from the masked image once all adversarial pixels are removed.

\textbf{Challenge.} If we have information on patch shapes, sizes, and locations, it is easy to find masks to remove the patch. 
For example, if we know an attacker will use a $32\times32$ \textit{patch}, we can enumerate all possible $32\times32$ \textit{masks} at different locations, and one of the masks must remove the entire $32\times32$ \textit{patch}. However, when we do not have patch information, this enumeration no longer works due to an infinite number of patch regions $\bfp$. In \framework, we propose an alternative masking strategy that does not depend on the patch information (i.e., patch-agnostic).

\textbf{Masking strategy.} The robustness requirement of our masking algorithm is that: for each object, we have at least one mask that removes the entire patch while preserving most of the object pixels (so that we can safely detect each object).\footnote{Note that we can use a different mask for a different object to achieve this goal (i.e., detecting different objects from different masked images).} 
Towards this goal, we propose to mask out \textit{``image halves"} determined by a set of horizontal and vertical lines; we provide visualization for this strategy in Figure~\ref{fig-mask}. First, we select a set of evenly spaced horizontal and vertical lines across the image (top left of Figure~\ref{fig-mask}). Second, observing that each horizontal/vertical line splits the image into two halves, we mask out one half and perform vanilla object detection on the other half (top right of Figure~\ref{fig-mask}).

\begin{figure}[t]
    \centering
    \includegraphics[width=\linewidth]{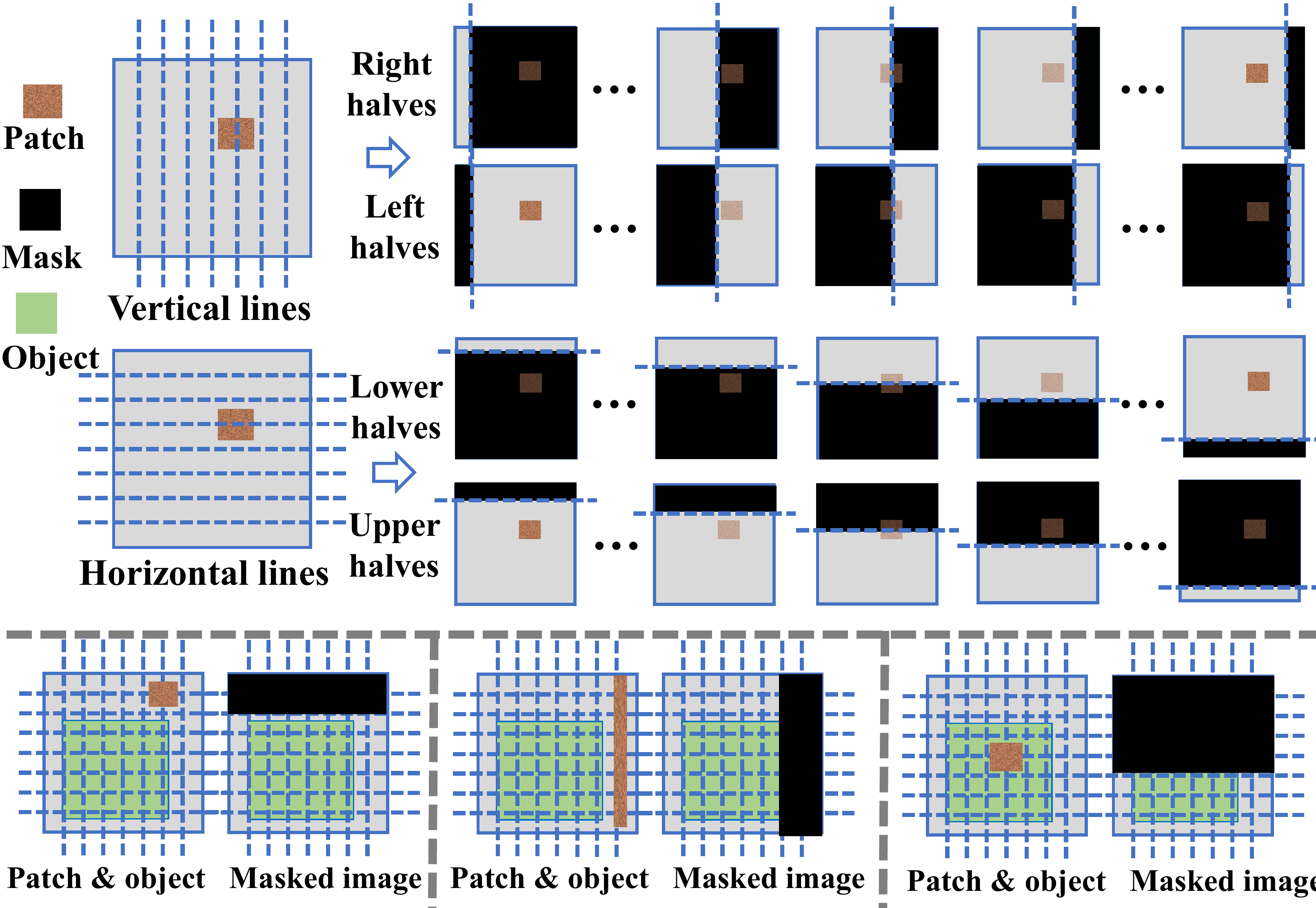}
    \vspace{-2em}
    \caption{\textbf{Visualization for patch-agnostic masking:} the algorithm masks out image halves divided by vertical/horizontal lines (top of the figure); masks from a fixed mask set can remove patches of different shapes, sizes, and locations (bottom of the figure).}
    \label{fig-mask}
\end{figure}

\textbf{The patch-agnostic property.} Clearly, our mask set generation strategy only requires the number of lines as its input, and does not rely on any information of patch shape/size/location (i.e., patch-agnostic). For a patch of any shape, size, and location, vertical/horizontal lines that \textit{do not intersect with the patch} can generate masks that remove the entire patch. At the bottom of Figure~\ref{fig-mask}, we visualize three patches of different shapes, sizes, and locations and the corresponding masks that can remove the patch. We can see that masks from our \textit{fixed} mask set can remove \textit{different} patches while preserving a large portion of object pixels. This lays a foundation for robust object detection. Finally, we highlight that the patch-agnostic property is a significant improvement from all existing masking-based certifiably robust defenses (for image classification)~\cite{mccoyd2020minority,xiang2021patchguard,xiang2021patchguard2,xiang2021patchcleanser}, which require patch information (e.g., sizes, shapes) for their inference procedure to achieve non-trivial certifiable robustness.

\textbf{Masking pseudocode.} Now, we discuss the implementation details of patch-agnostic masking. We present the pseudocode in Line~\ref{ln-inference-gen-mask}-\ref{ln-inference-mask-end} of Algorithm~\ref{alg-inference}. The first step of pixel masking is to generate a mask set (Line~\ref{ln-inference-gen-mask}). Formally, we represent each mask as a binary tensor $\mathbf{m} \in \cM \subset \{0,1\}^{W\times H}$ that has the same shape as the $W \times H$ images. We set the elements within the mask to $0$, and others to $1$. 

We use the sub-procedure $\textsc{MaskSet}(\cdot)$ for the mask set generation. It takes the image size $W\times H$ and the number of horizontal/vertical lines $k$ as inputs and outputs a mask set $\cM$. In this sub-procedure, we first generate two sets of coordinates $\cX,\cY$ that contain $k$ evenly spaced coordinates along two image axes (Line~\ref{ln-maskset-index}). Next, we generate a mask set $ \cM_\cX$, whose elements mask the left halves of the image (Line~\ref{ln-maskset-mx}). Similarly, we generate $\cM_\cY$ that masks the lower halves, $\bar{\cM}_\cX$ that masks the right halves, and $\bar{\cM}_\cY$ that masks the upper halves (Line~\ref{ln-maskset-my}-\ref{ln-maskset-mx-mx-bar}). The final mask set $\cM$ is the union of these four sets (Line~\ref{ln-maskset-union}).

With the generated mask set $\cM$, we iterate over each mask $\bfm \in \cM$ (Line~\ref{ln-inference-each-mask}), perform object detection with an undefended model $\mathbb{F}(\cdot,\gammam)$ on the masked image $\bfx\odot\bfm$ (Line~\ref{ln-inference-apply-mask}), and then gather detected boxes into the box set $\cB_{\text{mask}}$ (Line~\ref{ln-inference-add-detection}). 
After that, \framework will perform secure box pruning on these detected boxes, which will be discussed in the next subsection.


\begin{algorithm}[t]
    \centering
    \caption{\framework inference algorithm}\label{alg-inference}
    \begin{algorithmic}[1]
    \renewcommand{\algorithmicrequire}{\textbf{Input:}}
    \renewcommand{\algorithmicensure}{\textbf{Output:}}
    \Require Image $\mathbf{x}$, image size $(W,H)$, vanilla object detector $\mathbb{F}$, number of lines $k$ (for masking), masked box confidence threshold $\gammam$, base box confidence threshold $\gammab$, similarity score function $\mathbb{S}$, box filtering threshold $\tau$
    \Ensure  Robust detection results $\cB_{\text{robust}}$ 
    \Procedure{\framework}{$\mathbf{x},\mathbb{F},k,W,H,\gammam,\gammab,\mathbb{S},\tau$}
    \State $\cM\gets\textsc{MaskSet}(k,W,H)$\label{ln-inference-gen-mask}
    \State $\cB_{\text{mask}} \gets \varnothing$\label{ln-inference-empty-set}
    \For{$\bfm \in \cM$}\label{ln-inference-each-mask}
    \State $\cB_\bfm \gets \mathbb{F}(\bfx \odot \bfm,\gammam)$\label{ln-inference-apply-mask}
    \State $\cB_{\text{mask}}\gets\cB_{\text{mask}}\cup\cB_\bfm$\label{ln-inference-add-detection}
    \EndFor\label{ln-inference-mask-end}

\State $\cB_{\text{base}} \gets \mathbb{F}(\bfx,\gammab)$\label{ln-inference-base}
\State $\cB_{\text{robust}}\gets\textsc{BoxPrune}(\cB_{\text{mask}},\cB_{\text{base}},\mathbb{S},\tau)$\label{ln-inference-box-prune}
    \State \Return $\cB_{\text{robust}}$ \label{ln-inference-final-output}
    \EndProcedure
    
    \item[]
\Procedure{MaskSet}{$k,W,H$}
\State$\cX=\{\lceil \frac{W}{k+1}\rceil\cdot t | t\in \mathbb{Z}^+_k\};\ \cY=\{\lceil \frac{H}{k+1}\rceil\cdot t | t\in \mathbb{Z}^+_k\}$\label{ln-maskset-index}
\State$\cM_\cX \gets \{\bfm |  \bfm[i,j]=0, i\in[0,x);  \bfm[i,j]=1,i \in [x,W), x\in\cX\}$\label{ln-maskset-mx}
\State $\cM_\cY \gets \{\bfm |  \bfm[i,j]=0, j \in [0,y)  ; \bfm[i,j]=1, j\in [y,H) ,y\in\cY\}$\label{ln-maskset-my}
\State$\bar{\cM}_\cX \gets\{\mathbf{1}-\bfm | \bfm\in\cM_\cX\} ;\ \bar{\cM}_\cY =\{\mathbf{1}-\bfm | \bfm\in\cM_\cY\}$\label{ln-maskset-mx-mx-bar}
\State\Return$\cM_\cX \cup \cM_\cY \cup \bar{\cM}_\cX \cup \bar{\cM}_\cY$  \label{ln-maskset-union}
    \EndProcedure
\item[]
\Procedure{BoxPrune}{$\cB_{\text{mask}},\cB_{\text{base}},\mathbb{S},\tau$}
    \State ${\cB}^{\text{filtered}}_{\text{mask}}\gets\{\bfbm \in \cB_{\text{mask}} | \nexists\ \bfbb\in\cB_{\text{base}}: \mathbb{S}(\bfbm,\bfbb)>\tau\}$\label{ln-inference-filter}
    \State $\cB^{\text{pruned}}_{\text{mask}}\gets\{\textsc{Rep}(\hat{\cB},\mathbb{S}))  \ |\ \hat{\cB} \in \textsc{Cluster}({\cB}^{\text{filtered}}_{\text{mask}},\mathbb{S}) \}$\label{ln-inference-union}
    \State $\cB_{\text{robust}}\gets\cB_{\text{base}}\cup\cB^{\text{pruned}}_{\text{mask}}$\label{ln-inference-combine}
\State\Return$\cB_{\text{robust}}$\label{ln-inference-robust}
\EndProcedure
\end{algorithmic}
\end{algorithm}
\subsection{Secure Box Pruning}\label{sec-defense-merge}
\textbf{Intuition.} Our masking operation allows us to safely see benign pixels/objects from some of the masked images that contain no adversarial pixels.
With this robustness property, a naive defense strategy is to take all boxes detected on masked images as the final output. Since the attacker has no influence over the detection results when the patch is completely masked, we will detect most ground-truth objects and achieve a high robust \textit{recall} against the hiding attack.

\textbf{Challenge.} However, this approach is not ideal: we will have many  duplicate boxes for one object since an object might be detected multiple times in different masked images (see ``masked boxes" in Figure~\ref{fig-overview} for visual examples); redundant duplicate boxes will be considered as false-positive errors and hurt the detection \textit{precision}. Therefore, we need to further identify and prune these redundant boxes in \framework. We note that box pruning is a non-trivial task since duplicate boxes can look very different (e.g., in Figure~\ref{fig-overview}, some boxes detect the left part of the motorbike while some detect the right part). Moreover, the box pruning should be done in a secure manner since an adaptive attacker might introduce malicious boxes in some masked images (where the patch is not completely removed) to interfere with the box pruning. In summary, we need to perform secure box pruning to improve prediction \textit{precision} while preserving the robust \textit{recall} against the patch hiding attack.

\textbf{Box pruning algorithm.} The high-level idea of our box pruning is to use a score function $\mathbb{S}:\cB\times\cB\rightarrow\mathbb{R}$ to robustly measure the similarity between detected boxes so that we can identify and prune redundant boxes. We note that different similarity scores can give different robustness guarantees. Here, we focus on using IoA as the similarity score for achieving IoA robustness (recall Section~\ref{sec-formulation-defense}); we will also discuss an alternative choice in Appendix~\ref{apx-iou}.

We visualize the pruning algorithm in the right of Figure~\ref{fig-overview} and present the pseudocode in Line~\ref{ln-inference-base}-\ref{ln-inference-box-prune} and Line~\ref{ln-inference-filter}-\ref{ln-inference-combine} of Algorithm~\ref{alg-inference}. To start with, we perform one vanilla model prediction on the original unmasked image $\bfx$ to obtain the (potentially vulnerable) base detection results $\cB_{\text{base}}=\mathbb{F}(\bfx,\gammab)$ (Line~\ref{ln-inference-base}). 
We term boxes detected on the original unmasked image \textit{base boxes} and boxes detected on masked images as \textit{masked boxes}. Then, we call the sub-procedure $\textsc{BoxPrune}(\cB_{\text{mask}},\cB_{\text{base}},\mathbb{S},\tau)$ for box pruning, which first \textit{filters} out redundant masked boxes that are duplicates of base boxes and then \textit{unionizes} unfiltered masked boxes for the final prediction output.

\textit{Box filtering.} The first pruning operation aims to filter out redundant masked boxes that are highly similar to the base boxes (measured by $\mathbb{S}$). Specifically, we calculate the pairwise similarity scores between masked boxes and base boxes and remove masked boxes $\bfbm$ whose similarity with a particular base box $\bfbb$ exceeds a filtering threshold $\tau$, i.e., $\mathbb{S}(\bfbm,\bfbb)>\tau$. This filtering operation is illustrated in Line~\ref{ln-inference-filter} of Algorithm~\ref{alg-inference}. In the top right of Figure~\ref{fig-overview}, we can see that duplicate boxes are effectively removed when base boxes also detect the object.

\textit{Box unionizing.} Intuitively, if base box predictions are accurate, most masked boxes will be removed during the box filtering operation, and we will directly output high-quality base boxes (top right of Figure~\ref{fig-overview}). However, when a patch hiding attack happens, the attacker makes base box predictions disappear, and thus no masked boxes will be filtered (bottom right of Figure~\ref{fig-overview}). As a result, we need to further prune/unionize the remaining redundant masked boxes. The high-level idea of box unionizing is to use similarity score $\mathbb{S}$ to cluster similar boxes via $\textsc{Cluster}(\cdot)$ and output one box representative for each cluster via $\textsc{Rep}(\cdot)$. The collection of these box representatives is the final \textit{pruned masked boxes}. This process occurs in Line~\ref{ln-inference-union} of Algorithm~\ref{alg-inference} and is illustrated in the bottom right of Figure~\ref{fig-overview}. 

\textit{\framework output.} After the box filtering and unionizing operations, we combine the pruned masked boxes ${\cB}^{\text{pruned}}_{\text{mask}}$ with base boxes $\cB_{\text{base}}$ as the final output of \framework (Line~\ref{ln-inference-robust} and Line~\ref{ln-inference-final-output}). 

\textit{Instantiation with IoA.} In our implementation when we take IoA as $\mathbb{S}$, we instantiate $\textsc{Cluster}(\cdot)$ using a distance-based clustering algorithm DBSCAN~\cite{ester1996density}; the ``distance" between boxes $\bfb_0,\bfb_1$ is calculated as $1-\max(\textsc{IoA}(\bfb_0,\bfb_1),\textsc{IoA}(\bfb_1,\bfb_0))$. 
For each box cluster, we generate the box representative as $\textsc{Rep}(\hat{\cB}) = \bigcup_{\bfb\in\hat{\cB}} \bfb$, i.e., taking the mathematical union of all boxes as the representative of the cluster. In the rest of this paper, we will refer to this instance of box pruning as \textbf{$\textsc{IoA-BoxPrune(}\cB_{\text{mask}},\cB_{\text{base}},\tau)$} and call the corresponding \framework instance \textbf{$\textsc{IoA-\framework}(\bfx)$}.

\subsection{Robustness Certification}\label{sec-defense-certification}

So far, we have discussed how to use patch-agnostic masking to neutralize the adversarial patch and how to securely perform box pruning to remove duplicate boxes. In this subsection, we develop the robustness certification procedure (Algorithm~\ref{alg-certification}) for provable robustness evaluation. Given a victim object and a specific threat model (e.g., a specific patch size, shape, and location set), the certification procedure aims to determine if \framework can robustly detect the object against all possible attackers within a given threat model. 
Here, we will focus on the case where we implement $\mathbb{S}$ with IoA for certifiable IoA robustness. We note that the high-level idea of certification is similar for different $\mathbb{S}$; we will discuss a different $\mathbb{S}$ and its certification in Appendix~\ref{apx-iou}.

We first formally reiterate our IoA robustness objective.
\begin{definition}[Certifiable IoA robustness]\label{dfn-certifiable-ioa}
We consider \framework is certifiably IoA-robust for a ground-truth object $\bfbgt$ on an image $\bfx$ against a patch attacker $\cA_\cP$ (discussed in Section~\ref{sec-formulation-attack}), if we can always predict a box $\bfb^\prime$ satisfying $\textsc{IoA}(\bfbgt,\bfb^\prime)>T$, where $T\in[0,1]$ is a certification threshold. Formally, we have: $\forall\ \bfx^\prime\in\cA_\cP(\bfx),\exists\ \bfb^\prime\in\cB_{\text{robust}}=\textsc{IoA-\framework}(\bfx^\prime)\st \textsc{IoA}(\bfbgt,\bfb^\prime)>T$.
\end{definition}

\noindent Next, we discuss our certification intuition, present our certification procedure in Algorithm~\ref{alg-certification}, and prove its soundness in Theorem~\ref{thm}.

\textbf{Certification intuition.} In our \framework framework, we aim to detect victim objects on \textit{masked images without adversarial pixels} and combine pruned masked boxes and base boxes as the final output $\cB_{\text{robust}}$. Intuitively, if we have a ``good" masked box in $\cB_{\text{mask}}$, this good masked box is likely to ``survive" the box pruning procedure and become one good box in $\cB_{\text{robust}}$. Our certification aims to search for ``good" boxes on masked images without adversarial pixels. 

\begin{algorithm}[t]
    \centering
    \caption{Certification algorithm for IoA robustness}\label{alg-certification}
    \begin{algorithmic}[1]
    \renewcommand{\algorithmicrequire}{\textbf{Input:}}
    \renewcommand{\algorithmicensure}{\textbf{Output:}}
    \Require Image $\mathbf{x}$, the object (bounding box) $\bfbgt$ to be certified,  valid patch region set $\cP$, defense setup $(\mathbb{F},\cM,\gammam,\tau)$, certification threshold $T$.
    \Ensure Whether \framework has certifiable IoA robustness for $\bfbgt$ against $\cA_\cP$
    \Procedure{Certify}{$\mathbf{x},\bfb,\cP,\mathbb{F},\mathcal{M},\gammam,\tau$}
    \For{every $\bfp \in \cP$}
    \State $f_{\bfp}\gets\texttt{False}$\label{ln-certification-init}
    \For{$\bfm \in \{\bfm \in \cM  |  \bfm[i,j]\leq \bfp[i,j],\forall (i,j)\}$}
        \State $\cB_\bfm \gets \mathbb{F}(\bfx \odot \bfm,\gammam)$\label{ln-certification-predict}
        \If{$\exists \bfb_\bfm \in \cB_\bfm$s.t. $  \frac{|\bfbm|\cdot \tau - |\bfbm\setminus\bfbgt|}{|\bfbgt|}>T$}\label{ln-certification-condition}
       
            \State $f_{\bfp}\gets\texttt{True}$; \textbf{break}\label{ln-certification-condition2}
        \EndIf
    \EndFor\label{ln-certification-endcheck}
    
    \If{$f_{\bfp} = \texttt{False}$}
        \State\Return \texttt{False}\label{ln-certification-vulnerable}
    \EndIf
    \EndFor
    \State \Return \texttt{True}\label{ln-certification-robust}
    \EndProcedure
    \end{algorithmic}
\end{algorithm}

\textbf{Certification algorithm.} We provide the certification pseudocode in Algorithm~\ref{alg-certification}. It determines if \framework has certifiable IoA robustness for the given ground-truth object $\bfbgt$ against the given attack threat model $\cA_\cP$. 

Overall, the certification algorithm will iterate every valid patch region $\bfp\in\cP$ (e.g., every valid patch location) to determine if \framework can certify the robustness for the ground-truth object $\bfbgt$ against every $\bfp$ (Line~\ref{ln-certification-init}-\ref{ln-certification-endcheck}).

For each patch region $\bfp$ that represents a specific patch shape, size, and location, we initialize the robustness flag $f_{\bfp}$ to \texttt{False} (Line~\ref{ln-certification-init}). Next, we will examine every mask $\bfm \in \cM$ that can remove the entire patch (i.e., $\bfm[i,j]\leq \bfp[i,j],\forall (i,j)$). For each mask $\bfm$, we perform object detection on the masked image $\bfx\odot\bfm$ to get masked boxes $\cB_\bfm$ (Line~\ref{ln-certification-predict}). Then, we will search for any ``good" box in $\cB_\bfm$. If a masked box $\bfb_\bfm\in\cB_\bfm$ satisfies $ \frac{|\bfbm|\cdot \tau - |\bfbm\setminus\bfbgt|}{|\bfbgt|}>T$, we consider this box ``certifiably good" for robustly detecting object $\bfbgt$ against patch region $\bfp$ (will be proved in Theorem~\ref{thm}). On the other hand, if we try all possible masks that can remove the patch, and no masked box satisfies this condition, the object $\bfbgt$ might be vulnerable to this patch region $\bfp$ and the procedure returns \texttt{False} (Line~\ref{ln-certification-vulnerable}). 

Finally, if the certification procedure enumerates \textit{every} valid patch $\bfp\in\cP$ and does not return \texttt{False}, it implies that \framework has certified robustness for the object $\bfbgt$ against all possible $\bfp$ within the threat model $\cA_\cP$. The algorithm returns \texttt{True} (Line~\ref{ln-certification-robust}).

\textbf{Soundness of Algorithm~\ref{alg-certification}.} We present Theorem~\ref{alg-certification} below to prove the soundness of our certification.

\begin{theorem}\label{thm}
Given a ground-truth object box $\bfbgt$ in the input image $\bfx$, defense setup $(\mathbb{F},\cM,\gammam,\tau)$, certification threshold $T$, and a set of valid patch regions $\cP$, if Algorithm~\ref{alg-certification} returns \texttt{True}, \textsc{IoA-\framework} has certifiable IoA robustness for the object $\bfbgt$ against any attacker in $\cA_\cP$.
\end{theorem}

\begin{proof}
Recall our certification intuition: if we can detect a ``good" masked box $\bfbm\in\cB_{\text{mask}}$ on images without adversarial pixels, this box is likely to ``survive" box pruning and become a good box in $\cB_{\text{robust}}$.
In this proof, we will first define the property of a ``good"  box as \textit{pruning-safe} masked box, and then discuss how to find pruning-safe boxes.

\begin{definition}[pruning-safe masked box]\label{dfn-pruning-safe}
Let $\bfbm$ be a masked box that is part of the box pruning input $\cB_{\text{mask}}$. We call $\bfbm$ a pruning-safe masked box for the ground-truth object $\bfbgt$ and pruning procedure $\textsc{IoA-BoxPrune}(\cdot,\cdot,\tau)$, if there is always a box $\bfb^\prime$ in the box pruning output satisfying $\textsc{IoA}(\bfbgt,\bfb^\prime)>T$, regardless of the rest of box pruning inputs $\cB_{\text{mask}}\setminus \{\bfbm\},\cB_{\text{base}}$. Formally, we have: $    \forall\ \cB_{\text{mask}}\st\bfbm\in\cB_{\text{mask}},\
\forall\ \cB_{\text{base}},\ 
\exists\ \bfb^\prime\in\cB_{\text{robust}}=\textsc{IoA-BoxPrune}(\cB_{\text{mask}},\ \cB_{\text{base}},\tau) \st \textsc{IoA}(\bfbgt,\bfb^\prime)>T.$
\end{definition} 
\noindent We note that this definition considers all possible $\cB_{\text{mask}}$ that contain $\bfbm$ and all possible $\cB_{\text{base}}$. This captures adaptive attackers' capability to maliciously manipulate some of the masked boxes in $\cB_{\text{mask}}$ (when the patch is not removed by the masks) and all base boxes in $\cB_{\text{base}}$ to interfere with the box pruning procedure. With Definition~\ref{dfn-pruning-safe}, we can discuss a sufficient condition of certifiable robustness in Lemma~\ref{lemma-robust}.


\begin{lemma}\label{lemma-robust}
Given a ground-truth box $\bfbgt$, one patch region $\bfp$, if we detect a pruning-safe masked box $\bfb_\bfm$ from a masked image $\bfx\odot\bfm$ with no adversarial pixels (i.e., $\bfm[i,j]\leq \bfp[i,j],\forall (i,j)$), \textsc{IoA-\framework} has certifiable IoA robustness for object $\bfbgt$ against attacker $\cA_{\{\bfp\}}$.
\end{lemma}

\begin{proof}
Since the pruning-safe masked box $\bfb_\bfm$ is detected from a masked image without adversarial pixels, we have $\bfb_\bfm\in\cB_{\text{mask}}$ no matter what a patch attacker does using $\bfp$. From the definition of pruning-safe mask box, we have the guarantee that $\exists\ \bfb^\prime\in\cB_{\text{robust}}\st \textsc{IoA}(\bfbgt,\bfb^\prime)>T$, which implies certifiable IoA robustness. 
\end{proof}

With Lemma~\ref{lemma-robust}, an IoA robustness certification procedure only needs to look for pruning-safe masked boxes detected on masked images without adversarial pixels. Next, we will present two lemmas discussing how to identify pruning-safe masked boxes. 

Recall that the box pruning involves two steps of box filtering (Line~\ref{ln-inference-filter} of Algorithm~\ref{alg-inference}) and box unionizing (Line~\ref{ln-inference-union} of Algorithm~\ref{alg-inference}). Lemma~\ref{lemma-ioa} will give a lower bound of the IoA guarantee for the first step (box filtering). Lemma~\ref{lemma-ioa-prune} will use this lower bound to further derive the sufficient condition of being a pruning-safe masked box for the entire box pruning procedure (box filtering and box unionizing).

\begin{lemma}\label{lemma-ioa}
Given any ground-truth object box $\bfbgt$, if a detected masked box $\bfbm$ is filtered by a box $\bfbb$ during box filtering, i.e., $\textsc{IoA}(\bfbm,\bfbb)>\tau$, we have:
$$
\resizebox{\hsize}{!}{$\textsc{IoA}(\bfbgt,\bfbb) > \frac{|\bfbm|\cdot \tau - |\bfbm\setminus\bfbgt|}{|\bfbgt|},\ \forall\ \bfbb \st \textsc{IoA}(\bfbm,\bfbb)>\tau$}
$$
\end{lemma}
\begin{proof}
From $\textsc{IoA}(\bfbm,\bfbb) = |\bfbm\cap\bfbb|/|\bfbm|>\tau$, we have $|\bfbm\cap\bfbb|>|\bfbm|\cdot \tau$. 
With this condition, we can derive an inequality as follows:
$|\bfbgt\cap\bfbb| = |(\bfbgt\cap\bfbb)\cap\bfbm| + |(\bfbgt\cap\bfbb) \setminus \bfbm| \geq |\bfbm\cap\bfbb\cap\bfbgt| =  |\bfbm\cap\bfbb| - |(\bfbm\cap\bfbb) \setminus \bfbgt| \geq |\bfbm\cap\bfbb| - |\bfbm \setminus \bfbgt|> \bfbm\cdot \tau - |\bfbm \setminus \bfbgt|$ (two equal signs are based on basic set operations). Finally, we have $\textsc{IoA}(\bfbgt,\bfbb) =  |\bfbgt\cap\bfbb| / |\bfbgt| > ({|\bfbm|\cdot \tau - |\bfbm\setminus\bfbgt|})/{|\bfbgt|}$.
\end{proof}

\noindent{We use $ \mathbb{L}_{\textsc{IoA}}(\bfbgt,\bfbm,\tau) = ({|\bfbm|\cdot \tau - |\bfbm\setminus\bfbgt|})/{|\bfbgt|}$ to denote the lower bound. Lemma~\ref{lemma-ioa-prune} will demonstrate that $\mathbb{L}_{\textsc{IoA}}>T$ implies that $\bfbm$ is a pruning-safe masked box.}

\begin{lemma}\label{lemma-ioa-prune}
Given a ground-truth object $\bfbgt$ and the box filtering threshold $\tau$, if there is one masked box $\bfbm\in\cB_{\text{mask}}$ satisfying that $\mathbb{L}_{\textsc{IoA}}(\bfbgt,\bfbm,\tau)>T$, this box is a pruning-safe masked box for object $\bfbgt$ and $\textsc{IoA-BoxPrune}(\cdot,\cdot,\tau)$.
\end{lemma}
\begin{proof}

The detected masked box $\bfbm\in\cB_{\text{mask}}$ will go through box filtering and box unionizing to generate the final output. We will demonstrate that this masked box $\bfbm$ will ``survive" these two operations and become part of the final output, ensuring the IoA robustness.

\textit{Box filtering.} Recall that, in Line~\ref{ln-inference-filter} of Algorithm~\ref{alg-inference}, we remove a box $\bfbm$ if there is another box $\bfbb\in\cB_{\text{base}}$ satisfying $\textsc{IoA}(\bfbm,\bfbb)>\tau$, and the masked box set $\cB_{\text{mask}}$ becomes ${\cB}^{\text{filtered}}_{\text{mask}}$ after the filtering. We can prove that there is always a box $\bfb^* \in{\cB}^{\text{filtered}}_{\text{mask}}\cup\cB_{\text{base}}$  such that $\textsc{IoA}(\bfbgt,\bfb^*)>T$. There are two possible scenarios. 
\begin{enumerate}  \setlength\itemsep{0em}
    \item If $\bfbm$ is not filtered, we will know there is a box $\bfb^* = \bfbm \in{\cB}^{\text{filtered}}_{\text{mask}}$ such that $\textsc{IoA}(\bfbgt,\bfb^*) = |{\bfbgt} \cap \bfbm|/|\bfbgt| = ({|\bfbm| - |\bfbm\setminus\bfbgt|})/{|\bfbgt|} \geq \mathbb{L}_{\textsc{IoA}}(\bfbgt,\bfbm,\tau)>T$.
    \item If $\bfbm$ is filtered, there is a box $\bfb^* \in \cB_{\text{base}}$ such that $\textsc{IoA}(\bfbm,\bfb^*)>\tau$. From Lemma~\ref{lemma-ioa}, we know that this box $\bfb^* \in \cB_{\text{base}}$ satisfies $\textsc{IoA}(\bfbgt,\bfb^*)> \mathbb{L}_{\textsc{IoA}}(\bfbgt,\bfbm,\tau)>T$. 
\end{enumerate} 

\textit{Box unionizing.} Recall that we will perform clustering over filtered masked boxes ${\cB}^{\text{filtered}}_{\text{mask}}$ and take the mathematical union of each cluster of boxes as the representative of each cluster to get ${\cB}^{\text{pruned}}_{\text{mask}}$ (Line~\ref{ln-inference-union} of Algorithm~\ref{alg-inference}). Since the mathematical union operation will not decrease IoA, we know that there is a box $\bfb^{\prime}\in{\cB}^{\text{pruned}}_{\text{mask}}\cup\cB_{\text{base}}=\cB_{\text{robust}}$ such that $\textsc{IoA}(\bfbgt,\bfb^{\prime}) \geq \textsc{IoA}(\bfbgt,\bfb^*)>T$, which implies the certifiable robustness for the object $\bfbgt$.
\end{proof}

\noindent\textbf{Certification.} Lemma~\ref{lemma-robust} and Lemma~\ref{lemma-ioa-prune} together provide a simple way to certify IoA robustness: if we can detect a masked box $\bfbm$ on a masked image with no adversarial pixel ($\bfm[i,j]\leq \bfp[i,j],\forall (i,j)$), and this box is pruning-safe ($\mathbb{L}_{\textsc{IoA}}>T$), we have certifiable IoA robustness for the object $\bfbgt$ against the patch region $\bfp$. Recall that Algorithm~\ref{alg-certification} only returns \texttt{True} when these conditions are satisfied for all possible patch regions $\bfp\in\cP$ (e.g., patches at different locations). Therefore, our certification in Algorithm~\ref{alg-certification} has accounted for all possible attackers within $\cA_\cP$.
\end{proof}

\textbf{Remark: usage of Algorithm~\ref{alg-certification}.} Algorithm~\ref{alg-certification} and Theorem~\ref{thm} allow us to determine the certifiable robustness of \framework for a ground-truth object $\bfbgt$ in a given image $\bfx$ against a given threat model $\cA_\cP$. In our evaluation, we will report the fraction of certified objects in the annotated test set of benchmark datasets as the robustness metric. We note that the certification procedure (Algorithm~\ref{alg-certification}) is only used for robustness \textit{evaluation} and thus requires ground-truth annotations and specific patch information. When we deploy \framework in the wild, we use the inference procedure (Algorithm~\ref{alg-inference}) instead and thus do not need any ground-truth annotation or patch information.

\section{Evaluation}\label{sec-eval}
In this section, we implement \framework with two vanilla object detectors and evaluate the defense performance on two object detection datasets (we include a third dataset in Appendix~\ref{apx-eval}). 
We demonstrate a significant robustness improvement ($\sim$10\%-40\% absolute and $\sim$2-6$\times$ relative) over the prior work DetectorGuard~\cite{xiang2021detectorguard} as well as similarly high clean performance ($\sim$1\% drop compared with vanilla undefended models). 

\subsection{Setup}\label{sec-eval-setup}

In this subsection, we introduce our evaluation setup, including datasets, object detectors, evaluation metrics, robustness evaluation setup, and defense setup.

\noindent\textbf{Datasets.}

\underline{VOC}~\cite{voc}. The detection challenge of the PASCAL Visual Object Classes (VOC) project has annotations for 20 different object classes. We combine the \texttt{trainval2007} set (5k images) and the \texttt{trainval2012} set (11k images) for training and evaluate \framework on the \texttt{test2007} set (5k images), which is a conventional usage of the PASCAL VOC dataset~\cite{liu2016ssd,zhang2019towards}.

\underline{COCO}~\cite{coco}. The Microsoft Common Objects in COntext (COCO) dataset is a challenging object detection dataset with 80 annotated object classes. We use the \texttt{COCO2017} split for training (117k images) and validation (5k images). 

\noindent \textbf{Object detectors.}

\underline{YOLOR}~\cite{yolor} is a popular one-stage object detector that achieves a good balance between inference speed and accuracy. We choose YOLOR-S~\cite{yolor} as the base object detector in \framework. 

\underline{Swin Transformer}~\cite{liu2021swin} adopts the representative two-stage detector Mask R-CNN~\cite{he2017mask} architecture and uses the Swin Transformer~\cite{liu2021swin} as the backbone. Its largest model achieves state-of-the-art detection performance on COCO. We choose Swin-S for our experiments. 

\noindent\textbf{Evaluation metrics.}

\underline{Clean performance: Average Precision (AP).} We use AP as our evaluation metric for clean performance, which follows object detection benchmark competitions~\cite{voc,coco} and relevant research papers~\cite{redmon2017yolo9000,redmon2018yolov3,bochkovskiy2020yolov4,ren2015faster,he2017mask,lin2017focal,tan2020efficientdet,liu2016ssd,he2017mask,yolor}. An object detector will have different precision and recall values as we change its confidence threshold $\gamma$ (recall that $\mathbb{F}(\bfx,\gamma)$ only outputs boxes with confidence values larger than $\gamma$). AP is defined as the average of precision values at different recalls. Intuitively, AP considers model performance at different confidence thresholds, precision values, and recall values; thus, it provides a view of the model's overall performance. We note that AP is high only when \textit{both precision and recall} are high. In our evaluation, we report $\text{AP}_{0.5}$ (AP evaluated with an IoU threshold of 0.5). We provide more details of the AP calculation in Appendix~\ref{apx-setup-ap}.



\begin{table*}[t]
    \centering
    \caption{Performance of vanilla undefended models, \framework, and DetectorGuard~\cite{xiang2021detectorguard}}  \label{tab-main-eval}
    \vspace{-1em}
    \resizebox{\linewidth}{!}
{ \small
\begin{tabular}{c|l|c|c|c|c|c|c|c|c|c|c|c}
    \toprule
     \multicolumn{3}{c|}{Dataset} & \multicolumn{5}{c|}{PASCAL VOC~\cite{voc}}&\multicolumn{5}{c}{MS COCO~\cite{coco}}\\
    \midrule
    
 &\multirow{2}{*}{\diagbox{Model}{Metric}}&Certify &\multirow{2}{*}{$\text{AP}_{0.5}$} & \multirow{2}{*}{FAR}   &  \multicolumn{3}{c|}{Certified recall (@0.8)}&\multirow{2}{*}{$\text{AP}_{0.5}$} & \multirow{2}{*}{FAR}  &  \multicolumn{3}{c}{Certified recall (@0.6)}\\
  &&class?&&&{far-patch}& {close-patch}& {over-patch} & & &{far-patch}& {close-patch}& {over-patch} \\
    \midrule
\multirow{4}{*}{\rotatebox{90}{{\footnotesize YOLOR~\cite{yolor}}}}& Vanilla (undefended) & -- & 94.3\% & -- & -- & -- & -- & 70.3\% & -- & -- & -- & -- \\
 &\framework&{\cmark}&92.9\%&--&58.9\%&46.7\%&18.0\%& 69.3\%&--&41.5\%& 28.8\%& 15.5\%\\
 &\framework&{\xmark}&93.2\%&--&61.9\%& 49.8\%& 21.5\%&69.8\%&--&44.6\%& 31.8\%& 18.1\%\\
  &DetectorGuard~\cite{xiang2021detectorguard}&{\xmark}&93.0\%&5.2\%&32.7\%&25.1\%&12.0\%&69.6\%&2.4\%&13.7\%&8.0\%&3.0\%\\
  
 \midrule
 \multirow{4}{*}{\rotatebox{90}{Swin~\cite{liu2021swin}}}& Vanilla (undefended) & -- & 93.9\% & -- & -- & -- & -- & 69.6\% & -- & -- & -- & -- \\
 &\framework&\cmark&92.6\%&--&68.0\%& 55.2\% & 22.5\%&68.9\%&--&34.9\%& 24.8\%& 11.7\%\\
  
 &\framework& \xmark & 92.9\%&--& 70.8\%& 58.0\%& 26.9\%&69.2\%&--&37.5\%& 27.4\%& 14.2\%\\

 &DetectorGuard~\cite{xiang2021detectorguard}& \xmark &92.8\%&3.6\%&31.2\%&23.0\%&10.4\%&69.2\%&1.4\%& 11.4\%&6.8\%&2.4\%\\

      \bottomrule
    \end{tabular}}
  
\end{table*}

\underline{Robustness performance: Certified Recall (CertR).} We use \textit{certified recall} to evaluate defense robustness against patch hiding attacks. The certified recall is defined as the fraction of ground-truth \textit{objects} whose IoA robustness can be certified by our defense (for which Algorithm~\ref{alg-certification} returns \texttt{True}). We note that the certified recall changes as the \textit{clean} recall of the object detector changes (when we use a different confidence threshold $\gamma$). To enable a fair comparison, we report certified recall at a particular clean recall (CertR@0.x). We report CertR@0.8 for VOC and CertR@0.6 for COCO, which follows DetectorGuard~\cite{xiang2021detectorguard}.

\textbf{Robustness evaluation setup.} To evaluate the certifiable robustness of \framework and to fairly compare with DetectorGuard~\cite{xiang2021detectorguard}, we choose a square patch that takes 1\% of the image pixels and report certified recalls for a certification threshold $T=0$.
We will also analyze defense performance when we use different patch sizes (e.g., 1-100\% pixels), patch shapes (e.g., rectangles), and large certification thresholds $T$ (e.g., 0-0.9) in Section~\ref{sec-eval-detailed} and \ref{sec-eval-shape-size}.

As discussed in Section~\ref{sec-formulation-attack}, we consider three location models. We consider a patch location as a far-patch when the smallest distance along the height (and width) axis between any adversarial pixel and object pixel is larger than 10\% of the image height (and width). We count an over-patch when there are adversarial pixels within the object bounding box. We consider the remaining patch locations as close-patch. We find that certified robustness against \textit{all possible locations} is identical to that for over-patch locations, implying that over-patches are the hardest cases for defenders.

\textit{Note: the patch-agnostic property.} We note that the setup of \framework is agnostic to the shape, size, and location of the adversarial patch; the specified patch information is only used for robustness evaluation/certification. In other words, our defense algorithm and setup do not change when we consider a different patch shape, size, or location, though the robustness performance can change for different patches. 



\textbf{Defense setup.} In our default setting, we set the number of vertical/horizontal lines $k=30$ and the filtering threshold $\tau = 0.6$. We use different confidence thresholds $\gammab,\gammam$ for base boxes and masked boxes to adjust the \textit{clean recall} of \framework for AP and CertR@0.x evaluation; we provide additional details in Appendix~\ref{apx-setup-ap}. We note that we choose these parameters because they can give similarly small clean performance drops ($\sim$1\%) compared with undefended models on a validation set. In Appendix~\ref{apx-more-discussion}, we further demonstrate that using different randomly selected validation sets gives identical parameter selection results, and thus our default parameters are not ``overfitted" to the evaluation setup of this section. In Section~\ref{sec-eval-detailed}, We will further analyze the impact of different defense parameters. 

We will also evaluate the performance of DetectorGuard~\cite{xiang2021detectorguard} using their official open-source code. We use YOLOR and Swin as its vanilla object detectors for a fair performance comparison. 
We further report \underline{false alert rate (FAR)} for DetectorGuard since it is an attack-detection defense. FAR is the fraction of clean images for which DetectorGuard issues a false alert. Note that \framework is alert-free so it always has a zero FAR.

\begin{figure*}
\centering
\begin{minipage}[b]{0.32\linewidth}
\includegraphics[width=\linewidth]{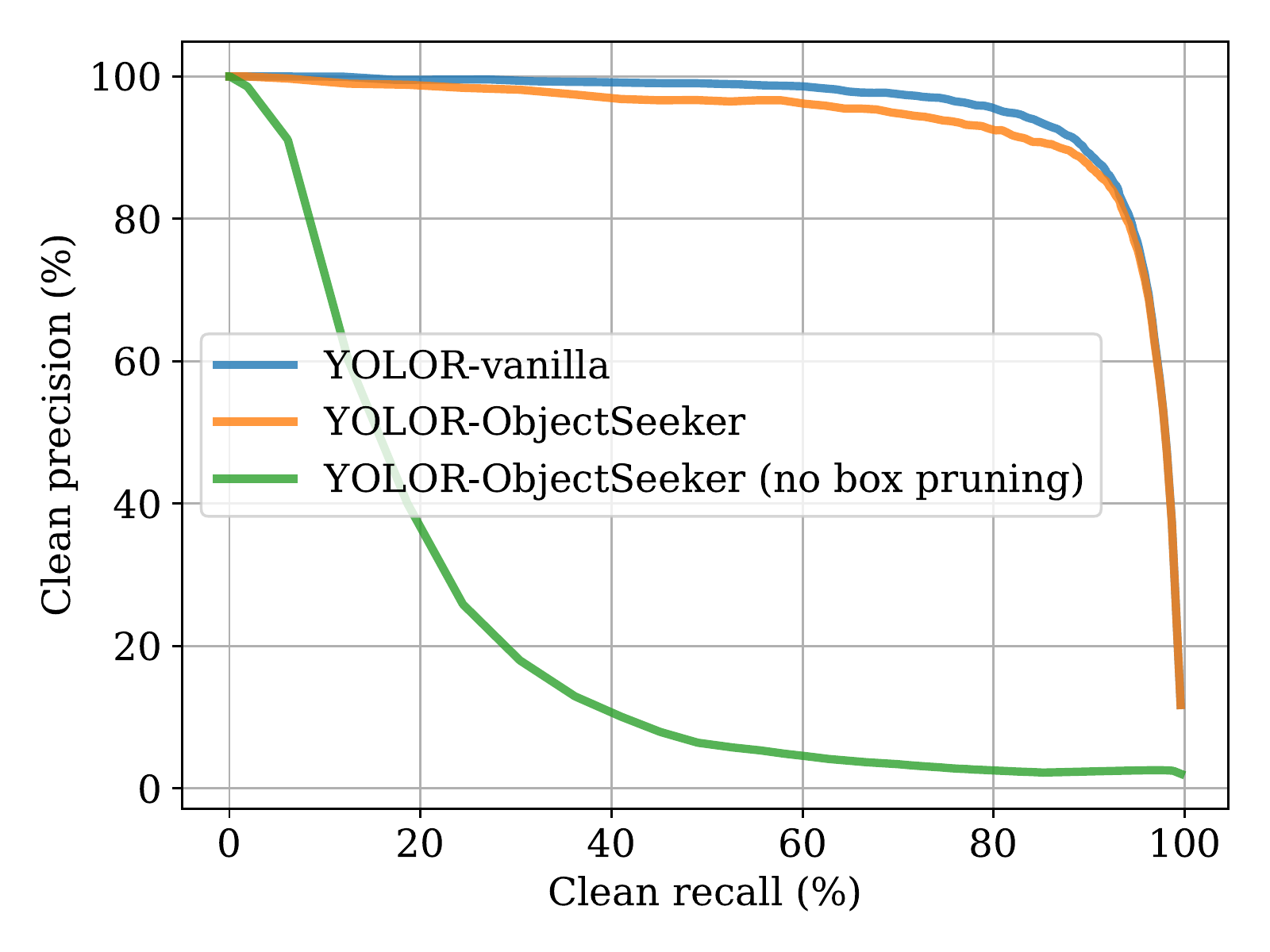}
\vspace{-2em}
    \caption{Clean precision vs. clean recall}
    \label{fig-prec-voc}
\end{minipage}%
\quad
\begin{minipage}[b]{0.32\linewidth}
\includegraphics[width=\linewidth]{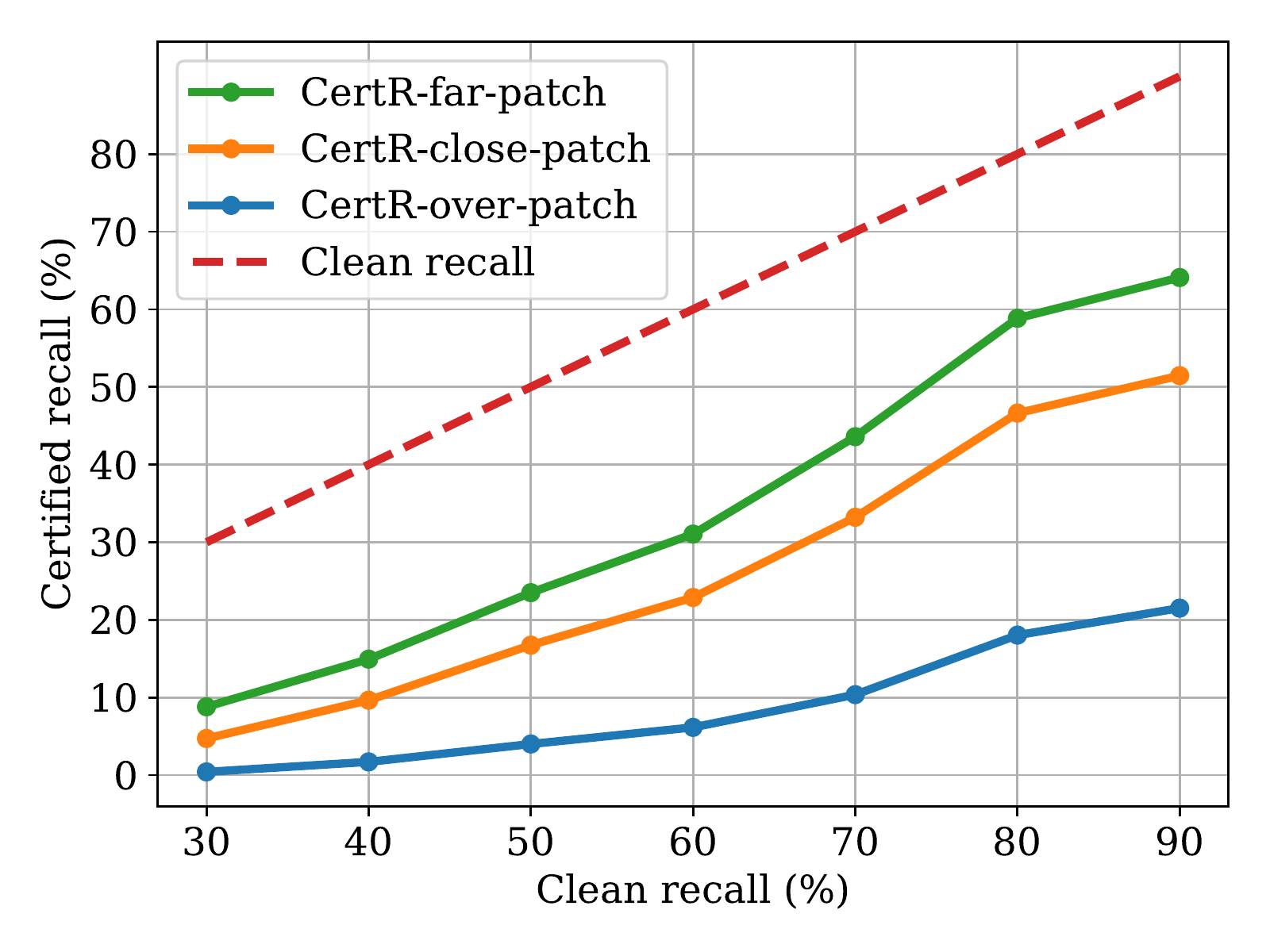}
\vspace{-2em}
    \caption{Certified recall vs. clean recall}
    \label{fig-cr-voc}
\end{minipage}%
\quad
\begin{minipage}[b]{0.32\linewidth}
\includegraphics[width=\linewidth]{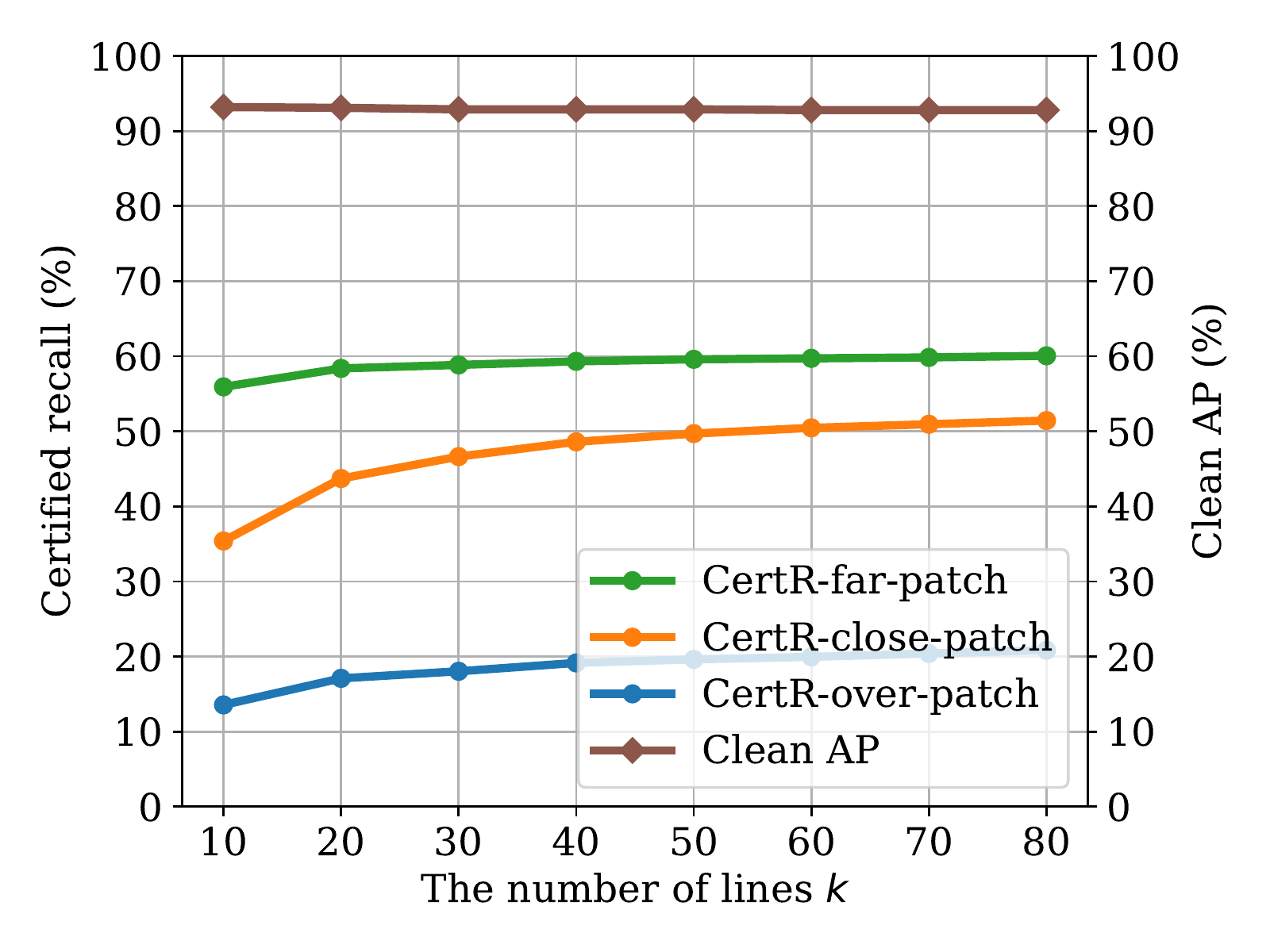}
\vspace{-2em}
    \caption{Effect of the number of lines $k$}
    \label{fig-k-voc}
\end{minipage}%
\end{figure*}

\begin{figure*}
\centering
\begin{minipage}[b]{0.32\linewidth}
\includegraphics[width=\linewidth]{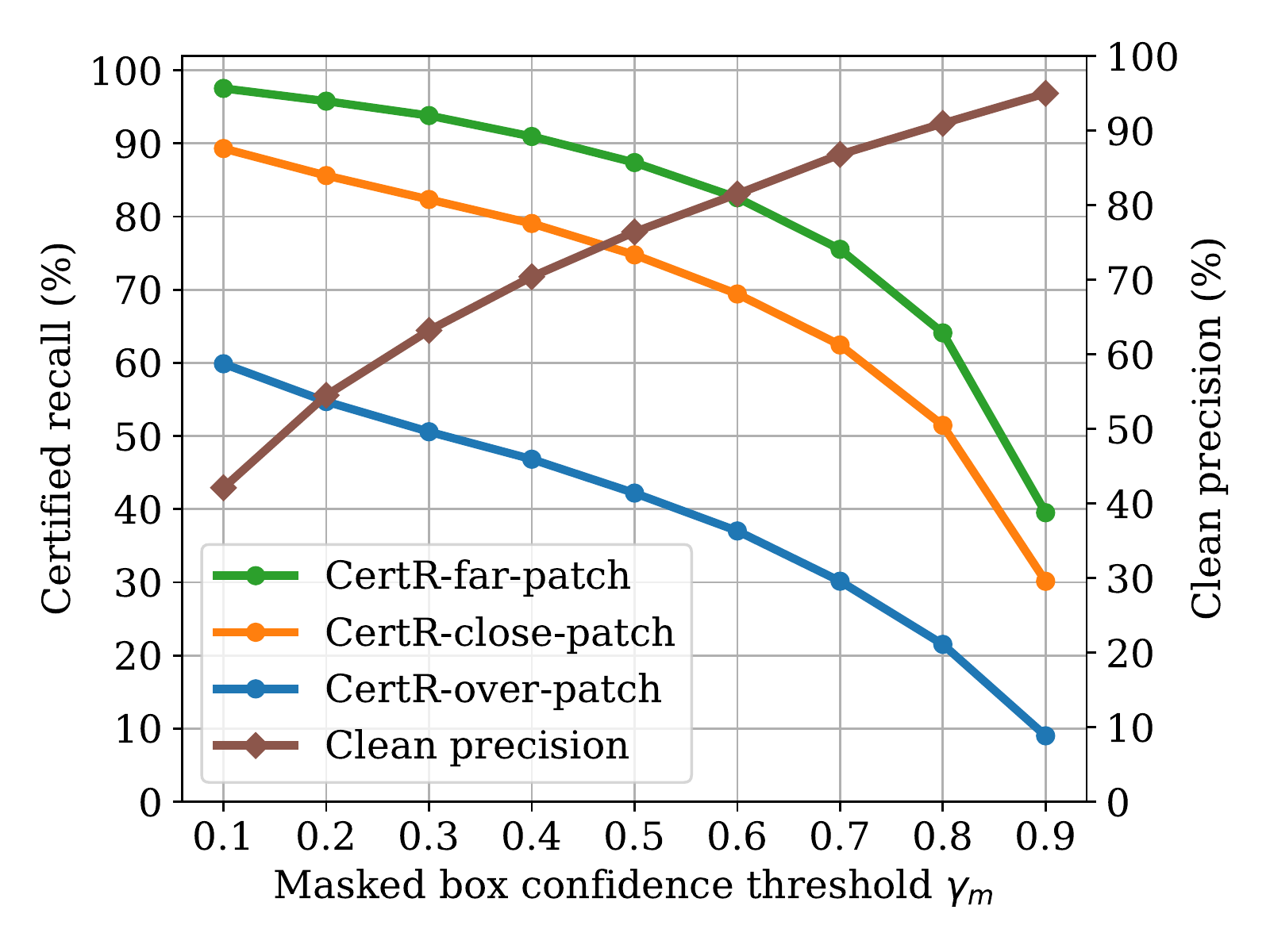}
\vspace{-2em}
    \caption{Effect of masked box confidence threshold $\gammam$}
    \label{fig-mconf-voc}
\end{minipage}%
\quad
\begin{minipage}[b]{0.32\linewidth}
\includegraphics[width=\linewidth]{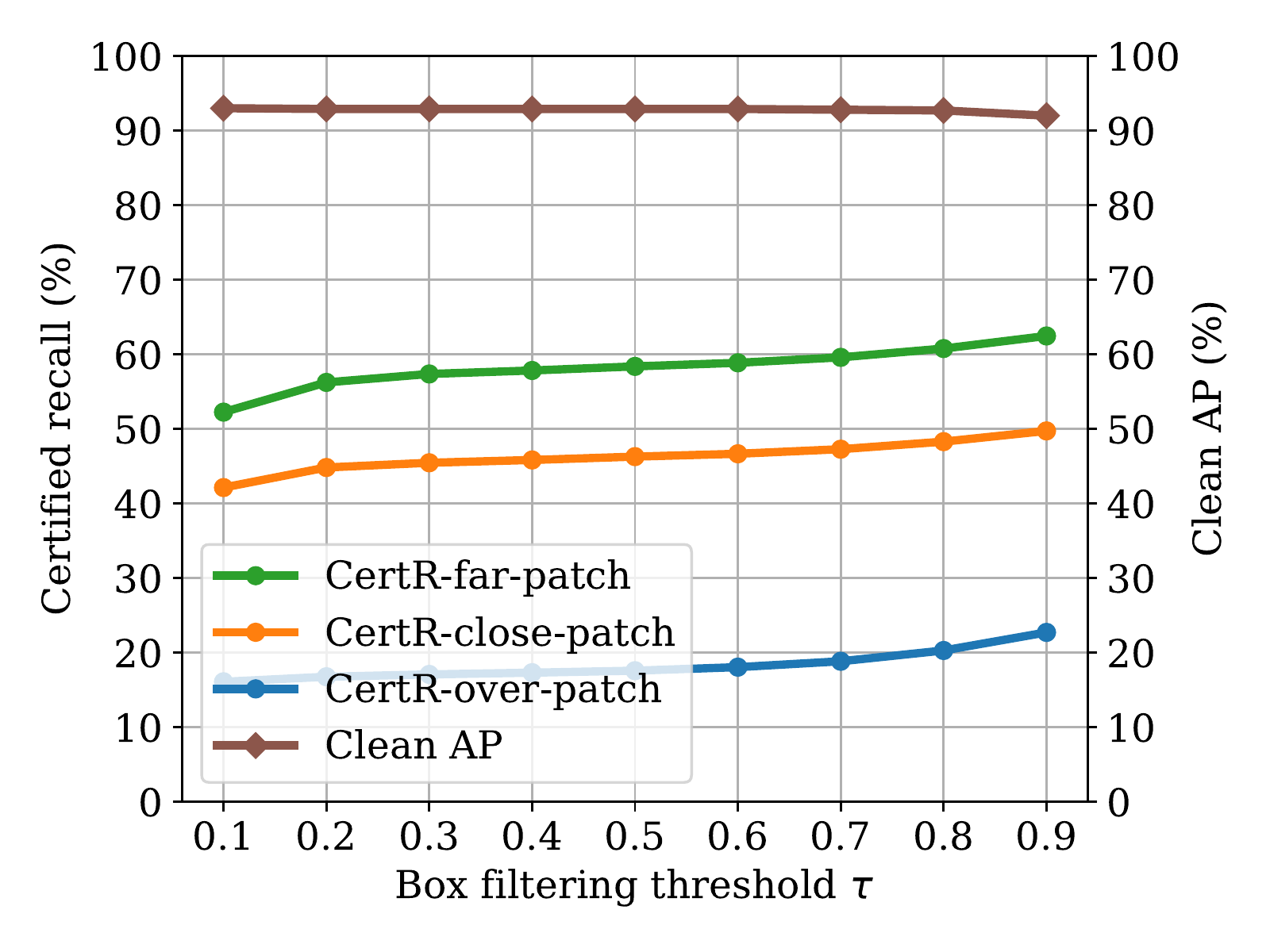}
\vspace{-2em}
    \caption{Effect of box filtering threshold $\tau$}
    \label{fig-tau-voc}
\end{minipage}%
\quad
\begin{minipage}[b]{0.32\linewidth}
\includegraphics[width=\linewidth]{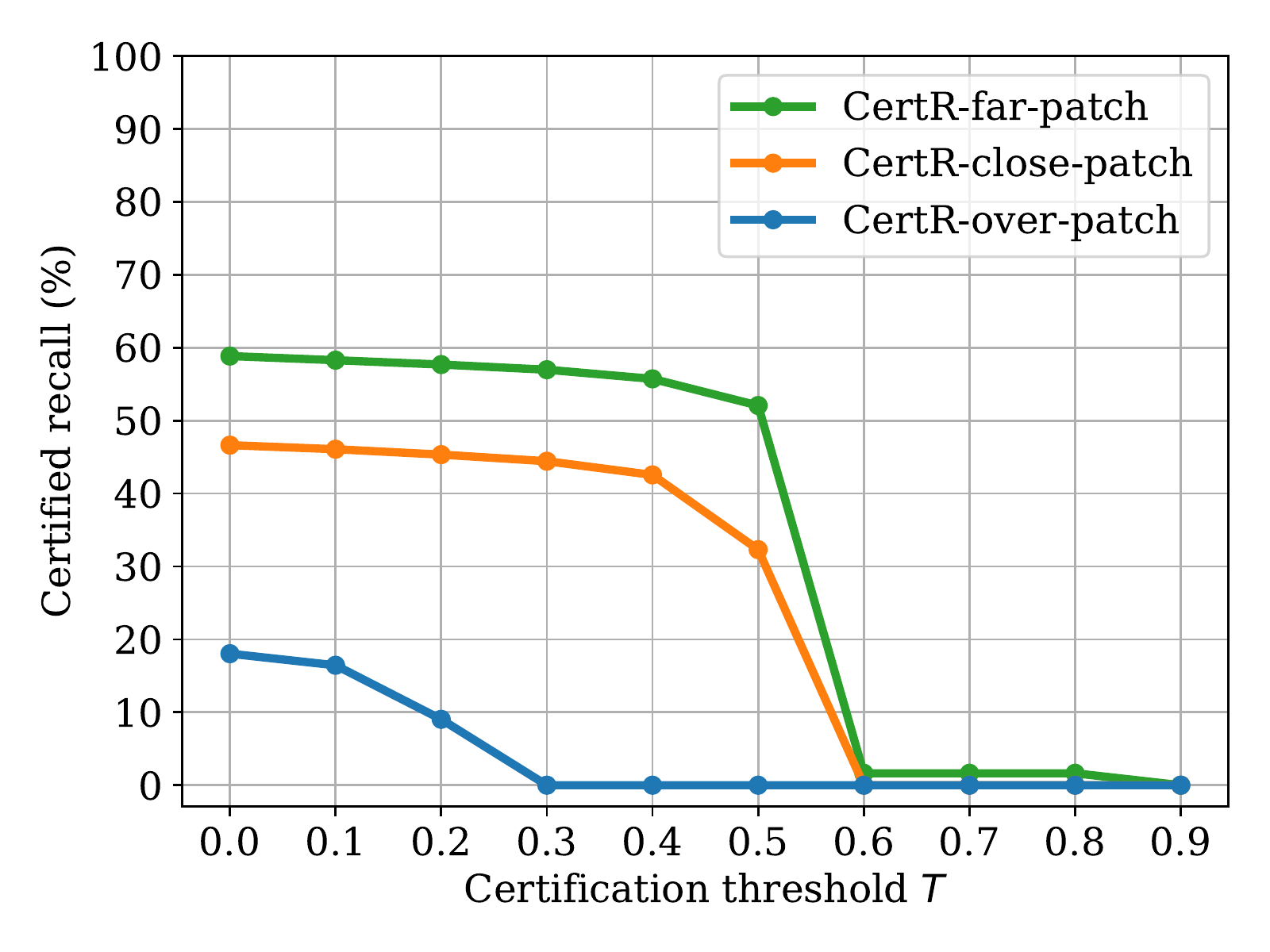}
\vspace{-2em}
    \caption{\framework performance for different certification thresholds $T$}
    \label{fig-certify-voc}
\end{minipage}%
\end{figure*}

\subsection{State-of-the-art Performance of \framework}\label{sec-eval-main}

In Table~\ref{tab-main-eval}, we report the defense performance of \framework and compare it with DetectorGuard~\cite{xiang2021detectorguard}. We note that \framework is able to certify the correct class label of the detected box (in addition to detecting the object); in contrast, DetectorGuard has no guarantee for the label. To enable a fair comparison, we additionally report performance for a variant of \framework that ignores the class labels in the step of secure box pruning. We also include the AP metric of vanilla undefended models to understand the clean performance of defenses.

\textbf{\framework achieves high certified recall across different datasets and threat models.} Table~\ref{tab-main-eval} demonstrates that \framework achieves high certified recalls. For example, \framework-Swin (certify class) achieves a 68.0\% certified recall for the far-patch model on the VOC dataset. That is, for 68.0\% of the objects in the test set of the VOC dataset, no patch hiding attacker using a 1\%-pixel square patch that is far away from the object can bypass our defense (i.e., hide the object). We can also see similarly high numbers across different object detectors, datasets, and threat models.

\textbf{\framework has a similarly high clean performance as vanilla object detectors.} Comparing APs of vanilla models and \framework in Table~\ref{tab-main-eval}, we can see that \framework achieves a similar clean AP as the vanilla object detectors -- the AP drops are only $\sim$1\%. The high clean performance can foster the real-world deployment of our defense. We note that the clean performance of \framework that certifies class labels is slightly worse than that without class certification. This is because the vanilla object detector can make mistakes between similar object classes (e.g., motorbike vs. bicycle, car vs. bus) when the object is partially masked. Despite the small clean AP drop, we note that \framework achieves a stronger robustness notion by certifying class labels.

\textbf{\framework achieves significant performance improvements compared with DetectorGuard~\cite{xiang2021detectorguard}.} We also compare defense performance between \framework and DetectorGuard~\cite{xiang2021detectorguard}. 
First, we can see that \framework achieves a significant improvement in certified recalls. For example, \framework-Swin (not certify class) improves the certified recall by 2$\times$ across three different location models on VOC (16.5\%-39.6\% absolute CertR improvements). The relative improvement on COCO is even larger: \framework-YOLOR improves the certified recall by 6$\times$ against the over-patch attacker. We note the significant improvement holds even when we require \framework to certify class label: we have achieved much higher certified recalls for an even stronger robustness notion. All these results demonstrate the strength of \framework.

Second, \framework also has similarly high clean performance as DetectorGuard. When we do not certify class label, \framework has slightly higher clean APs than DetectorGuard. When we require \framework to protect class labels, the clean APs are slightly lower. Moreover, we want to note that DetectorGuard is an attack-detection defense: it alerts and abstains from making predictions when it detects an attack. This design can cause non-trivial false alerts on the clean images and downgrade the user experience. In contrast, \framework is an alert-free defense and thus has more advantages in real-world deployment.

\textbf{Summary.} In this subsection, we demonstrate that \framework achieves high certified recalls while maintaining high clean APs as vanilla undefended models. Through a comparison with the only prior work DetectorGuard~\cite{xiang2021detectorguard}, we further demonstrate \framework's state-of-the-art defense performance against patch hiding attacks.

\begin{figure*}
\centering
\begin{minipage}[b]{0.38\linewidth}
    \includegraphics[width=\linewidth]{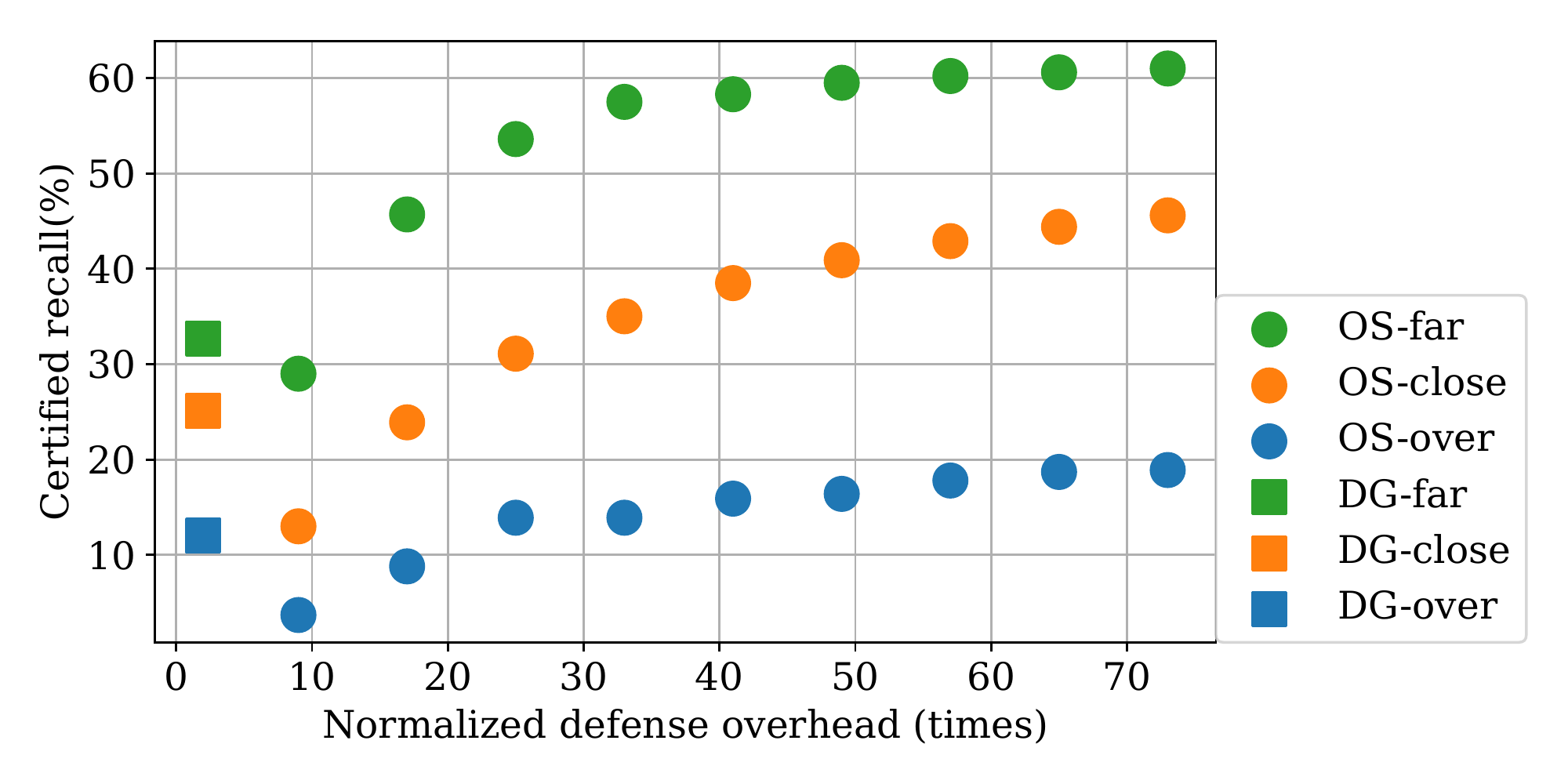}
    \vspace{-2em}
    \caption{Trade-off between defense overhead and certified robustness (OS: \framework; DG: DetectorGuard~\cite{xiang2021detectorguard})}
    \label{fig-efficiency-voc}
\end{minipage}%
\quad
\begin{minipage}[b]{0.28\linewidth}
    \includegraphics[width=\linewidth]{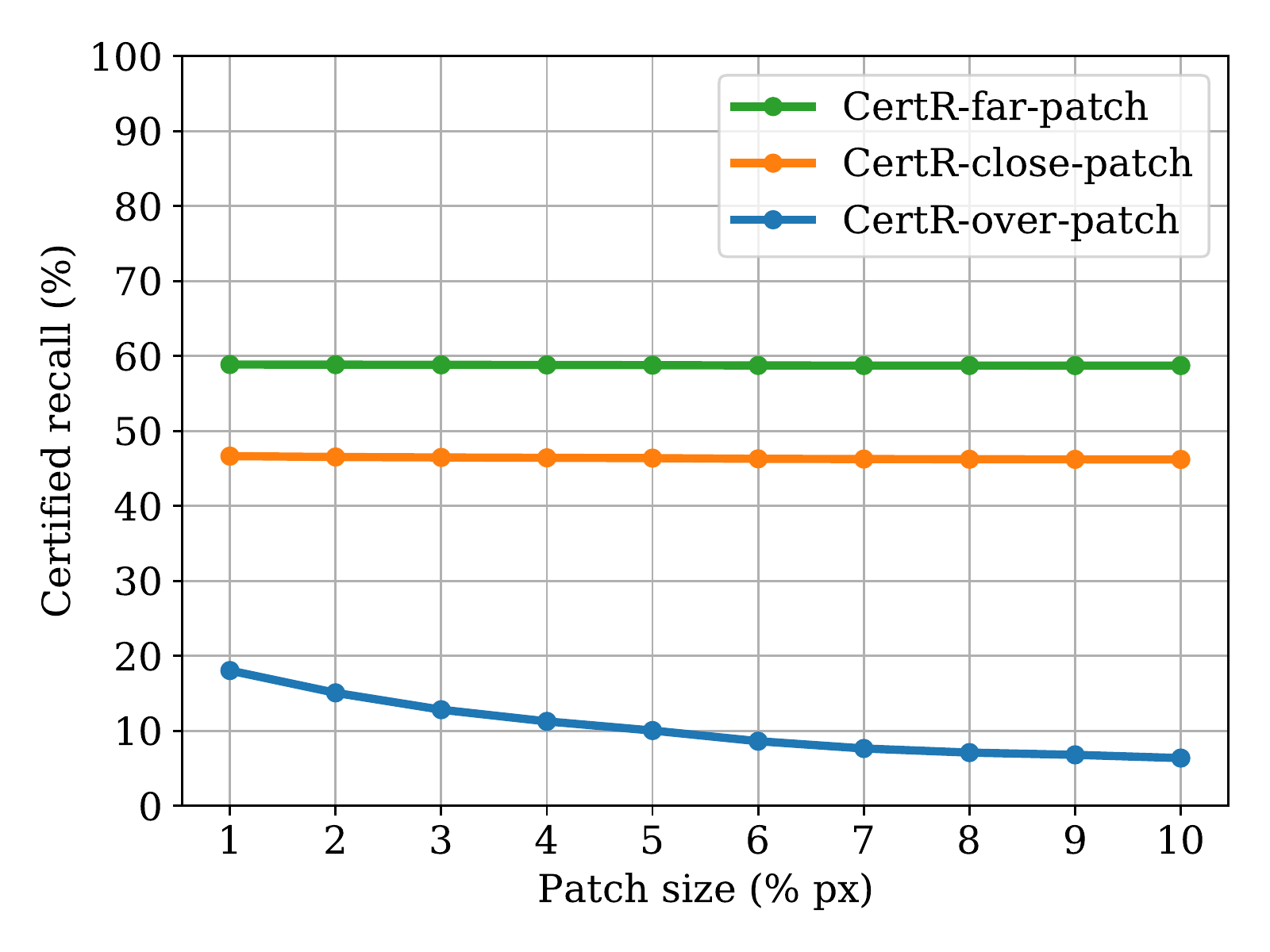}
    \vspace{-2em}
    \caption{Robustness against different patch sizes ($k=30$)}
    \label{fig-patch-voc}
\end{minipage}%
\quad
\begin{minipage}[b]{0.28\linewidth}
    \includegraphics[width=\linewidth]{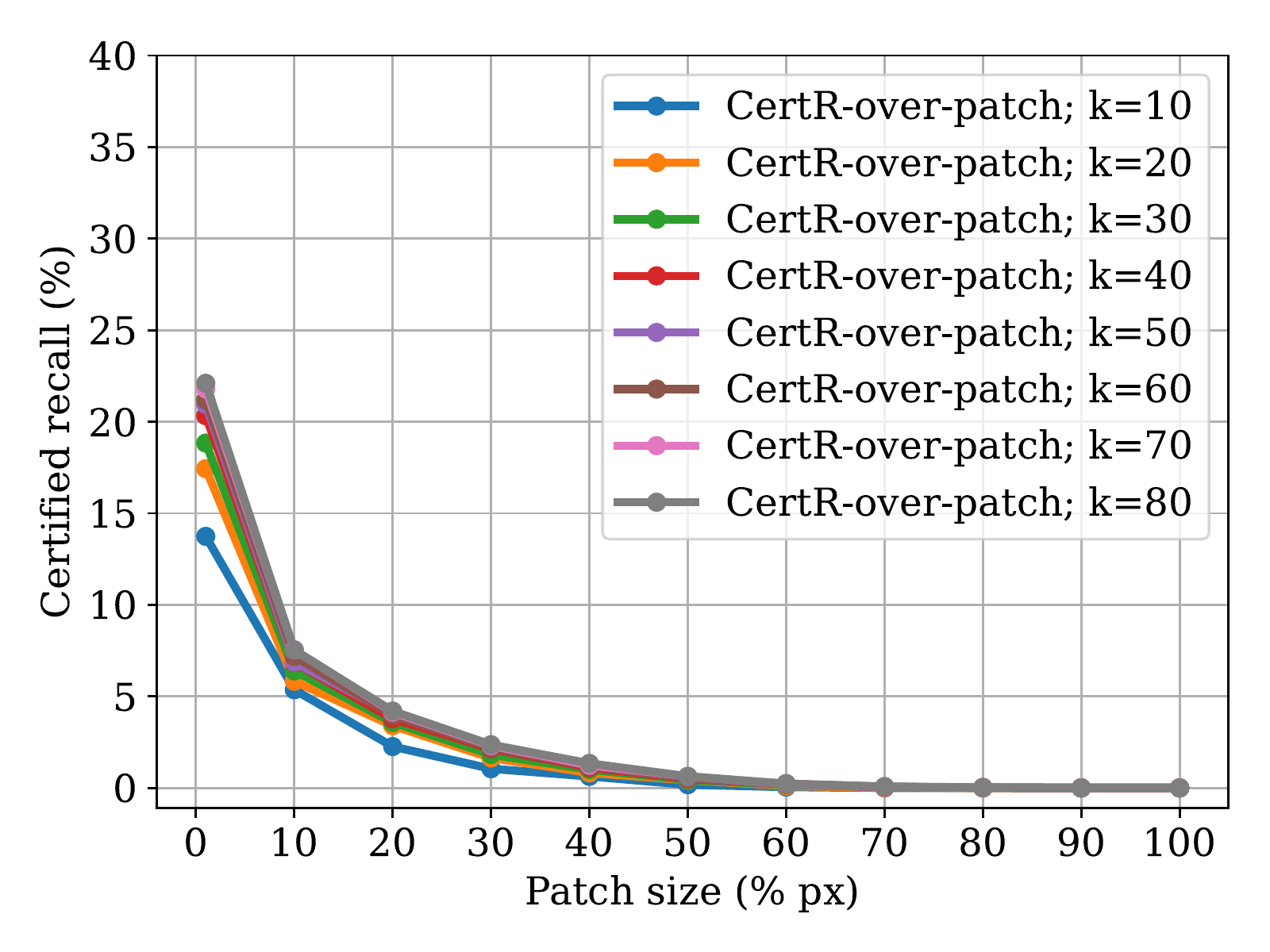}
    \vspace{-2em}
    \caption{Robustness against different over-patch sizes with different $k$}
    \label{fig-patch-k-voc}
\end{minipage}%
\end{figure*}

\subsection{Detailed Analysis of \framework}\label{sec-eval-detailed}
In this subsection, we perform detailed analyses using the YOLOR detector and the VOC dataset. We will report defense performance at different clean recalls, discuss the effect of different defense parameters, and analyze the defense overhead. We report similar analysis results for COCO in Appendix~\ref{apx-eval}.

\textbf{Clean precision vs. clean recall.} In Figure~\ref{fig-prec-voc}, we plot the clean precision-recall curves for vanilla YOLOR and \framework-YOLOR. First, we can see that two curves are close to each other, explaining similar clean APs reported in Table~\ref{tab-main-eval}. Second, \framework-YOLOR has slightly lower precision than vanilla YOLOR when the precision value is higher than 85\%; this explains the slight AP drop in Table~\ref{tab-main-eval}. Third, we additionally plot the curve for \framework without the secure box pruning module. We can see that this variant has a very low precision, which demonstrates the necessity and effectiveness of our box pruning module. Finally, we note that the precision of both vanilla YOLOR and \framework-YOLOR starts to drop quickly as the clean recall increases over 80\%. Therefore, we choose to study model robustness at a clean recall of 0.8 (CertR@0.8) for VOC in Table~\ref{tab-main-eval}, when the model has both high clean precision and high clean recall. 

\textbf{Certified recall vs. clean recall.} In Figure~\ref{fig-cr-voc}, we report certified recall at different clean recall values (recall that we only report CertR@0.8 in Table~\ref{tab-main-eval}). As shown in the figure, the certified recall increases as clean recall increases. This is expected since the more objects we can detect in the clean setting, the more objects we can try to provide certifiable robustness for. Moreover, we note that the gap between clean recall and certified recall is stable as we vary the clean recall values. This implies that \framework is compatible with detectors at different clean recall levels. How to further close this gap is an important future research question.

\textbf{Effect of the number of lines $k$.} In this analysis, we vary the number of lines $k$ used for the mask set generation, and plot the defense performance 
in Figure~\ref{fig-k-voc}. As shown in the figure, when we use a larger $k$, the robustness (CertR) gradually improves because we have a finer granularity of masks to bound the corrupted image region (recall Figure~\ref{fig-mask} in Section~\ref{sec-defense-mask}). Meanwhile, we can see a slight drop ($<0.5\%$) in clean AP as we increase $k$. This is because we have more masked images and more masked boxes, which leads to some unsuccessfully pruned boxes and hurts the precision of object detection. Furthermore, we find that the gain in certified robustness becomes minimal when $k$ is larger than 30. Therefore, we choose $k=30$ in our default setting to avoid excessive computational overhead.

\textbf{Effect of the masked box confidence threshold $\gammam$.} In Figure~\ref{fig-mconf-voc}, we study the effect of masked box confidence threshold $\gammam$ (with a fixed $\gammab$). Recall that we only consider masked boxes whose prediction confidence exceeds the threshold $\gammam$. Figure~\ref{fig-mconf-voc} demonstrates that the threshold $\gammam$ greatly affects the clean precision and the certified recall of \framework. When we use a small threshold $\gammam$, we have more masked boxes and have a better chance for successful robustness certification. However, a lower threshold leads to less confident and less precise predictions, which results in more incorrect boxes and decreases the clean precision.

\textbf{Effect of box filtering threshold $\tau$.} In Figure~\ref{fig-tau-voc}, we analyze the effect of box filtering threshold $\tau$. 
As we use a smaller $\tau$, the AP of \framework improves since there are fewer boxes left after the box filtering operation. However, the certified robustness is downgraded with a smaller $\tau$ (we note that CertR drops are larger if we consider a larger T; see Appendix~\ref{apx-more-discussion} for more results). This is because the lower bound proved in Lemma~\ref{lemma-ioa} gets lower with a smaller $\tau$, making the robustness certification harder to succeed.

\textbf{\framework performance for different certification thresholds $T$.} In our default setting, we set the certification threshold $T=0$. That is, we consider the defense is robust if we can detect even just a tiny part of the object. This setup follows DetectorGuard~\cite{xiang2021detectorguard} and enables a fair comparison in Table~\ref{tab-main-eval}. Here, we study the defense performance when we require the defense to detect at least $T$ of the object. We report the results in Figure~\ref{fig-certify-voc}. We can see that the CertR for over-patch is greatly affected by the certification threshold $T$. This is because we allow the adversary to place a patch at any location over the object. A worst-case attacker can put the patch at the center of the object to minimize the IoA guarantee (sometimes even occluding the entire small objects). 
Furthermore, we can see that the robustness for close-patch and far-patch are generally stable until the threshold $T$ hits a large value. This demonstrates that \framework provides stronger robustness when the patch does not overlap with objects. In Appendix~\ref{apx-more-discussion}, we further discuss the practical implications of different T.

\textbf{Defense overhead.} In Figure~\ref{fig-efficiency-voc}, we plot the defense overheads (normalized by the runtime of the base vanilla object detector) versus certified recalls for DetectorGuard~\cite{xiang2021detectorguard} and \framework. As shown in the figure, when the computational budget is small (lower than 20$\times$), DetectorGuard has better robustness (with a small overhead of 2$\times$). However, as we have more computational resources, \framework's performance gradually improves and eventually outperforms DetectorGuard by a large margin. This trade-off between defense overhead and defense performance should be carefully balanced when deploying the \framework defense. Moreover, we note that our approach can be trivially parallelized with multiple GPUs (performing vanilla predictions for different masked images simultaneously). In Appendix~\ref{apx-more-discussion}, we provide a quantitative analysis of absolute wall-clock runtime. We will show that using 8 GPUs can give $6.6\times$ speedup. Together with other implementation-level optimizations, we can run \framework on VOC with a latency of 40ms (25fps).

\subsection{Different Patch Sizes and Shapes}\label{sec-eval-shape-size}

In this subsection, we analyze defense performance against different patch sizes and shapes. Note that the performance is evaluated using the same defense setup (recall the patch-agnostic property).

\textbf{\framework performance against different patch sizes.} In Figure~\ref{fig-patch-voc}, we plot the defense performance against patches with different sizes (from 1\% to 10\% image pixels). As we use a large patch size, the robustness against the over-patch model gradually drops. This is expected since a larger patch has a greater chance to occlude the salient part of the object. Intriguingly, we find that the robustness for close-patch and far-patch exhibits a different behavior: the certified recalls barely change when faced with a larger patch. This analysis further demonstrates that our defense has constrained the adversarial effect to a local region: it is harder for the attacker to hide objects that do not overlap with the patch. This property of \framework can be helpful in practice when the attacker is not always able to place the patch over the victim object.

Furthermore, we report CertR for over-patch whose size ranges from 1-100\% when we use different $k$ in Figure~\ref{fig-patch-k-voc}. We can see that the CertR curves for different $k$ are highly similar across different patch sizes. This further demonstrates the patch-agnostic property of \framework: the default $k=30$ used in the paper is not implicitly optimized for small patches (e.g., occupying 1\% image pixels).




\begin{table}[t]
    \centering
    
        \caption{Certified recalls (\%) against one 1\%-pixel patch of different rectangle shapes}\label{tab-shape-voc}
            \vspace{-1em}
   \resizebox{\linewidth}{!} {\small
    \begin{tabular}{l|c|c|c|c|c|c|c|c|c}
\toprule
 
aspect ratio& 16:1 &8:1&4:1&2:1&1:1&1:2&1:4&1:8&1:16\\
\midrule
    far-patch      & 58.8&58.8 &58.9&58.9&58.9&58.9&58.9&58.9&58.9\\
close-patch        & 46.7&46.7&46.7&46.7&46.7&46.7&46.6&46.7&46.7\\
over-patch         & 21.3&19.2&19.6&18.8&18.0&18.1&16.8&18.3&18.3\\
\bottomrule
    \end{tabular}}

\end{table}

\textbf{\framework performance against different patch shapes.} In this analysis, we study the defense performance against different patch shapes. In Table~\ref{tab-shape-voc}, we report certified recalls for different rectangular shapes that take up 1\% of the image pixels (the aspect ratio ranging from 16:1 to 1:16). The results demonstrate that \framework is effective against different patch shapes: CertRs for far-patch and close-patch barely change; CertRs for over-patch only change slightly. We note that these results are obtained \textit{with the same set of defense parameters}, which further validates the patch-agnostic property of our masking defense.

Furthermore, we note that robustness certified in Table~\ref{tab-shape-voc} directly applies to other shapes that can be completely covered by the rectangles. For example, since a 1\%-pixel square can cover a $(\pi/4)$\%-pixel circle, certified robustness for a 1\%-pixel square also holds for a $(\pi/4)$\%-pixel circle.

\section{Discussion}\label{sec-discussion-limitation}
In this section, we discuss the limitations and future work directions of \framework.

\textbf{Robustness against multiple patches.} In this paper, we focus on the setting of \textit{one} adversarial patch with unknown content, shape, size, and location because it is an unresolved research question. However, the design of \framework is general: we only require that some masks from the mask set $\cM$ can remove all adversarial pixels. If we have a new mask set that can mask out all (multiple) patches, we can simply plug this new mask set $\cM$ into \framework. 

To provide a proof-of-concept, we experiment with two 0.5\%-pixel patches on the VOC dataset. First, we generate a mask set $\cM^\prime$ (for one patch) as discussed in Section~\ref{sec-defense-mask}. Second, we generate a new mask set $\cM$ that contains all possible two-mask combinations from the mask set $\cM^\prime$. Formally, we have $\cM=\{\bfm_0\odot \bfm_1 | (\bfm_0,\bfm_1)\in \cM^\prime\times\cM^\prime\}$. Third, we use $\cM$ to instantiate \framework. We report the defense performance in Table~\ref{tab-two-patch}.\footnote{We use a mask set $\cM^\prime$ that has 40 masks to generate the two-mask set $\cM$ for \framework. Since the number of all two patch locations can also be too large to evaluate, e.g., more than $(600\times1000)^2>10^{11}$ for images with $600\times1000$ pixels, we only select 200 VOC test images and select $(1/50)^2$ of all possible two-patch locations for robustness certification.} We can see that our defense has high clean performance and achieves non-trivial certified recall against two-patch attacks. 



\begin{table}[t]
    \centering
        \caption{Defense performance against two 0.5\%-pixel square patches for 200 VOC test images}
    \label{tab-two-patch}
        \vspace{-1em}
       \resizebox{\linewidth}{!}{
    \begin{tabular}{c|c|c|c|c}
    \toprule
 \multirow{2}{*}{\diagbox{Model}{Metric}} &\multirow{2}{*}{$\text{AP}_{0.5}$}   &  \multicolumn{3}{c}{Certified recall (@0.8)}\\

  &&{far-patch}& {close-patch}& {over-patch}  \\
   \midrule
    YOLOR&  94.4\%&--&--&--\\
    \framework-YOLOR& 94.4\%&51.7\%&24.6\%& 5.1\%\\
         \bottomrule
    \end{tabular}}

\end{table}


\textbf{Improving \framework runtime.} \framework needs to perform vanilla object detection on $4k$ masked images. As analyzed in Figure~\ref{fig-k-voc} and Figure~\ref{fig-efficiency-voc}, $k$ needs to be large enough to achieve good defense robustness. As a result, the \framework defense incurs a non-negligible overhead. We note that it is worthwhile to spend more computation for high certifiable robustness for applications whose robustness is important. For example, for the security-critical video content analysis, we can apply \framework to offline video to ensure robustness. Nevertheless, it is also important to study how to reduce defense overhead with both algorithm and implementation-level improvements. An algorithm-level optimization for real-time systems could be applying \framework to a subset of frames to balance the efficiency and robustness. Implementation-level optimizations could include parallelizing inference on masked images with multiple GPUs and resizing input images; we provide quantitative examples in Appendix~\ref{apx-more-discussion}.


\textbf{Improving robustness against over-patch.} In our evaluation, we follow DetectorGuard~\cite{xiang2021detectorguard} to use three patch location models to analyze \framework against different attack capabilities. Despite the large \textit{relative} improvements from DetectorGuard~\cite{xiang2021detectorguard}, we acknowledge that \framework's \textit{absolute} CertR for over-patch remains limited. Further enhancing robustness against the challenging over-patch attackers is an important future research objective.

\textbf{Further exploration of similarity score functions $\mathbb{S}$ and robustness notions.} In our \framework design, we use a ``robust" similarity score function $\mathbb{S}$ to prune redundant boxes, and we note that using different similarity scores $\mathbb{S}$ can provide different types of robustness notions. In this paper, we use IoA as the similarity score function $\mathbb{S}$ and focus on the concept of ``IoA robustness" (i.e., $\textsc{IoA}(\bfbgt,\bfb^\prime)=|\bfbgt\cap\bfb^\prime|/|\bfbgt|>T$). This is because IoA captures how much of the ground-truth box is detected and aligns with the objective of defending against patch hiding attacks (following DetectorGuard~\cite{xiang2021detectorguard}). In Appendix~\ref{apx-iou}, we discuss one alternative \framework instance: we use IoU as $\mathbb{S}$ and achieve ``IoU robustness" for far-patch attackers: we aim to predict a box $\bfb^\prime$ for each ground-truth box $\bfbgt$ such that $\textsc{IoU}(\bfbgt,\bfb^\prime)=|\bfbgt\cap\bfb^\prime|/|\bfbgt\cup\bfb^\prime|>T$. We demonstrate that, against a far-patch attacker, we can achieve certified recall of $\sim$50\% for VOC and $\sim$40\% for COCO with an IoU certification threshold of 0.5. In Appendix~\ref{apx-taxonomy}, we further provide a taxonomy of different robustness notions against patch hiding attacks. 

\textbf{Accounting for physical-world attack constraints.} The physically realizable nature of patch attacks imposes a threat to real-world object detectors and motivates the design of our defense. In \framework, however, we did not model physically realizable constraints such as printability and lighting conditions, but simply assumed that the attacker can introduce arbitrary pixel values. As a result, \framework's certification is over-conservative against physical-world attacks. How to further leverage physical-world constraints and improve defense robustness and efficiency could also be an interesting future work direction.


\section{Related Work}\label{sec-related-work}

\subsection{Adversarial Patch Attacks}
\textbf{Image classification.}
The adversarial patch attack was first introduced for\textit{ image classification}. Brown et al.~\cite{brown2017adversarial} demonstrated that, by constraining all adversarial pixels within a local restricted region, an attacker can carry out the patch attack in the physical world. The physically realizable nature of the patch attack imposed a huge threat to real-world computer vision systems. Follow-up papers further studied variants of patch attacks against image classifiers with different threat models~\cite{karmon2018lavan,yang2020patchattack,liu2019perceptual,liu2020bias,doan2021tnt}.

\textbf{Object detection.} 
Numerous patch attacks against \textit{object detection} have been proposed. Liu et al.~\cite{liu2019dpatch} proposed the first patch attack against object detectors. 
Lu et al.~\cite{lu2017adversarial}, Chen et al.~\cite{chen2018shapeshifter}, Eykholt et al.~\cite{eykholt2018physical}, and Zhao et al.~\cite{zhao2019seeing} proposed different physical attacks against traffic sign recognition. Thys et al.~\cite{thys2019fooling}, Xu et al.~\cite{xu2020adversarial},  and Wu et al.~\cite{wu2019making} studied how to use adversarial patches to evade person detection. Recently, Hu et al.~\cite{hu2021naturalistic} and Tan et al.~\cite{tan2021legitimate} propose natural-looking patch hiding attacks against object detectors.

\subsection{Defenses against Adversarial Patches}
\noindent\textbf{Image classification.} 
To counter the threat of adversarial patch attacks, there have been a large number of heuristic-based image classification defenses proposed in recent years~\cite{hayes2018visible,naseer2019local,wu2019defending,rao2020adversarial,mu2021defending}. 
Unfortunately, many of these defenses are found broken when there is an adaptive attacker who knows about the defense setup~\cite{chiang2021adversarial}. 
To provide a strong provable robustness guarantee for patch attacks, the research community has proposed a number of certifiably robust defenses~\cite{chiang2020certified,zhang2020clipped,levine2020randomized,mccoyd2020minority,xiang2021patchguard,metzen2021efficient,xiang2021patchguard2,xiang2021patchcleanser,han2021scalecert,salman2021certified} for image classification models. In this paper, \framework focuses on \textit{a harder task of object detection}.

\textbf{Object detection.} 
How to secure object detectors is a challenging and under-studied research question. Saha et al.~\cite{saha2020role} studied how to constrain the use of spatial context in YOLOv2~\cite{redmon2017yolo9000} and improved the robustness against a patch at the image corner. Metzen et al.~\cite{metzen2021meta} studied meta-learning techniques to improve model robustness. Ji et al.~\cite{ji2021adversarial} added adversarial patches to the training dataset and taught the model to detect patches. Liang et al~\cite{liang2021we} proposed two heuristic-based defenses to detect a patch hiding attack. Chiang et al.~\cite{chiang2021adversarial} designed a defense model to detect and remove the adversarial pixels. Despite their contributions to robust object detection research, these defenses are all based on heuristics and do not have any formal security guarantee. 

In contrast, Xiang et al.~\cite{xiang2021detectorguard} proposed DetectorGuard as the only certifiably robust defense against patch hiding attacks, which aimed to provide a robustness guarantee for certain certified objects against any adaptive attacker within the threat model. DetectorGuard designed an objectness explaining strategy to build certifiably robust object detectors using off-the-shelf certifiably robust image classifiers. However, DetectorGuard only achieved limited certified robustness (recall Table~\ref{tab-main-eval}). Xiang et al.~\cite{xiang2021detectorguard} pointed out that the bottleneck for the defense performance is the imperfection of existing certifiably robust image classifies: the incorrect classification outputs from the robust image classifier resulted in a heavy trade-off between robustness and clean performance.  
In this paper, we design \framework solely based on vanilla undefended object detectors and easily bypass the bottleneck of imperfect image classifiers. In Section~\ref{sec-eval}, we have demonstrated the significant improvements ($\sim$2-6$\times$) in robustness performance over DetectorGuard. Moreover, DetectorGuard is an attack-detection defense that has troublesome false alerts in the clean setting while we design \framework as an alert-free defense. Finally, DetectorGuard cannot protect the class labels due to its design limitations while \framework achieves high certified robustness for bounding box labels in addition to securing the class label (recall Table~\ref{tab-main-eval} in Section~\ref{sec-eval-main}).

\textbf{Masking-based defenses against adversarial patches.} We note that the idea of masking out adversarial patches has been studied in existing defenses for image classification~\cite{hayes2018visible,mu2021defending,mccoyd2020minority,xiang2021patchguard,xiang2021patchguard2,xiang2021patchcleanser} and object detection~\cite{chiang2021adversarial}. However, they either lack certifiable robustness~\cite{hayes2018visible,mu2021defending,chiang2021adversarial}, or require additional information on patch shapes/sizes~\cite{mccoyd2020minority,xiang2021patchguard,xiang2021patchguard2,xiang2021patchcleanser}. Our \framework proposes the first patch-agnostic masking strategy that achieves certifiable robustness.

\subsection{Other Adversarial Example Attacks and Defenses}
Adversarial example attacks and defenses for computer vision tasks with different threat models have been extensively studied~\cite{barreno2010security,szegedy2013intriguing,biggio2013evasion,goodfellow2014explaining,papernot2016limitations,meng2017magnet,xu2017feature,carlini2017towards,madry2017towards,papernot2016distillation,raghunathan2018certified,wong2017provable,lecuyer2019certified,cohen2019certified,salman2019provably,gowal2018effectiveness,mirman2018differentiable,xiang2019generating}. We focus on the adversarial patch attacks because they are physically-realizable and impose an urgent threat to real-world computer vision systems.

\section{Conclusion}
In this paper, we propose \framework as a certifiably robust defense against patch hiding attacks. \framework is a two-step defense framework: we first perform patch-agnostic masking to neutralize the adversarial effect (without knowing patch size, shape, and location); we next perform secure box pruning for a precise and robust prediction output. We can certify that for certain objects, \framework can always detect the object no matter what an adaptive attacker within the threat model does. Our certifiable evaluation demonstrates a significant improvement ($\sim$10\%-40\% absolute and $\sim$2-6$\times$ relative) in certified robustness over the only prior work DetectorGuard~\cite{xiang2021detectorguard}, as well as the high clean performance of \framework ($\sim$1\% drops compared with vanilla undefended models). 

\section*{Acknowledgements}
We are grateful to the anonymous shepherd and reviewers from IEEE S\&P 2023 program committee for their insightful comments and helpful suggestions. We would like to thank Shawn Shan, Sihui Dai, and Vikash Sehwag for providing early feedback on the manuscript draft. This work was supported in part by the National Science Foundation under grant CNS-2131859, the ARL’s Army Artificial Intelligence Innovation Institute (A2I2), Schmidt DataX award, and Princeton E-ffiliates Award.

\bibliographystyle{IEEEtran}
\bibliography{reference.bib}
\appendices

\section{Additional Details of Implementation and Evaluation}
\label{apx-setup}
In this section, we discuss the setup of confidence thresholds $\gammam,\gammab$, the patch size for robustness evaluation, and the DBSCAN clustering algorithm used in box unionization. We release our source code at \url{https://github.com/inspire-group/ObjectSeeker} for reproducibility.

\subsection{Average Precision and Confidence Thresholds}\label{apx-setup-ap}
\textbf{AP in conventional object detection research.} As discussed in Section~\ref{sec-eval-setup}, different confidence thresholds would give different precision and recall for an object detector; thus, it is not representative enough to simply evaluate model performance at one fixed confidence threshold. To overcome this challenge, conventional object detection research uses Average Precision (AP) as the main evaluation metric: we vary the confidence threshold from zero to one, record the precision-recall pairs at different thresholds, and finally calculate AP as the averaged precision value across different recall values. AP can also be considered as the Area under Curve (AUC) of the precision-recall curve and thus provides a global view of the model performance. Furthermore, note that we consider a detected box a true-positive only when its IoU with the ground-truth box exceeds a certain threshold, and TP is used for calculating precision and recall. In our evaluation, we set this IoU threshold to 0.5 and report $\text{AP}_{0.5}$.

\textbf{Confidence thresholds and AP calculation in \framework.} As presented in Algorithm~\ref{alg-inference}, we only consider boxes whose confidence values exceed certain thresholds; we have two separate confidence thresholds $\gammam,\gammab$ for masked boxes and base boxes, respectively. Here, we provide additional implementation details and insights of this design.

In our implementation of \framework, we set the masked box threshold $\gammam$ as a function of the base box threshold $\gammab$. Specifically, we have $\gammam = \max(\alpha,\gammab + (1-\gammab)\cdot \beta ), \alpha,\beta\in[0,1]$. First, this ensures that low-quality and low-confidence (smaller than $\alpha$) masked boxes will be discarded. Second, we will use a higher $\gammam$ for a higher $\gammab$ to ensure high precision of \framework. 
To calculate AP for \framework, we vary the base confidence threshold $\gammab$ from 0 to 1 with a step of 0.01, change the $\gammam$ correspondingly, and record the precision and recall values of \framework. Then we calculate the averaged precision at different recall values in the conventional way.

We further note that different object detectors (e.g., YOLO vs. Faster RCNN) have different confidence value distributions. For example, the confidence values of most boxes predicted by YOLO could be lower than 0.95 while most Faster RCNN boxes could be more confident than 0.95.
As a result, we need to use the clean recall (instead of the confidence threshold) as a universal normalizer when evaluating averaged precision and certified recall. Moreover, we need to adjust $\alpha,\beta$ for different object detectors and different datasets when deploying the \framework defense. In our experiments, we set $\alpha=0.8,\beta=0.8$ for YOLOR on VOC, $\alpha = 0.7,\beta=0.5$ for YOLOR on COCO, $\alpha = 0.8,\beta=0.5$ for Swin on VOC, and $\alpha=0.9,\beta=0.8$ for Swin on COCO.

Finally, we note that, since clean recall is the universal normalizer of confidence threshold $\gammab$, we analyzed \framework performance at ``normalized" clean recalls instead of ``raw" $\gammab$ in Section~\ref{sec-eval}.

\begin{figure}[t]
    \centering
    \includegraphics[width=0.66\linewidth]{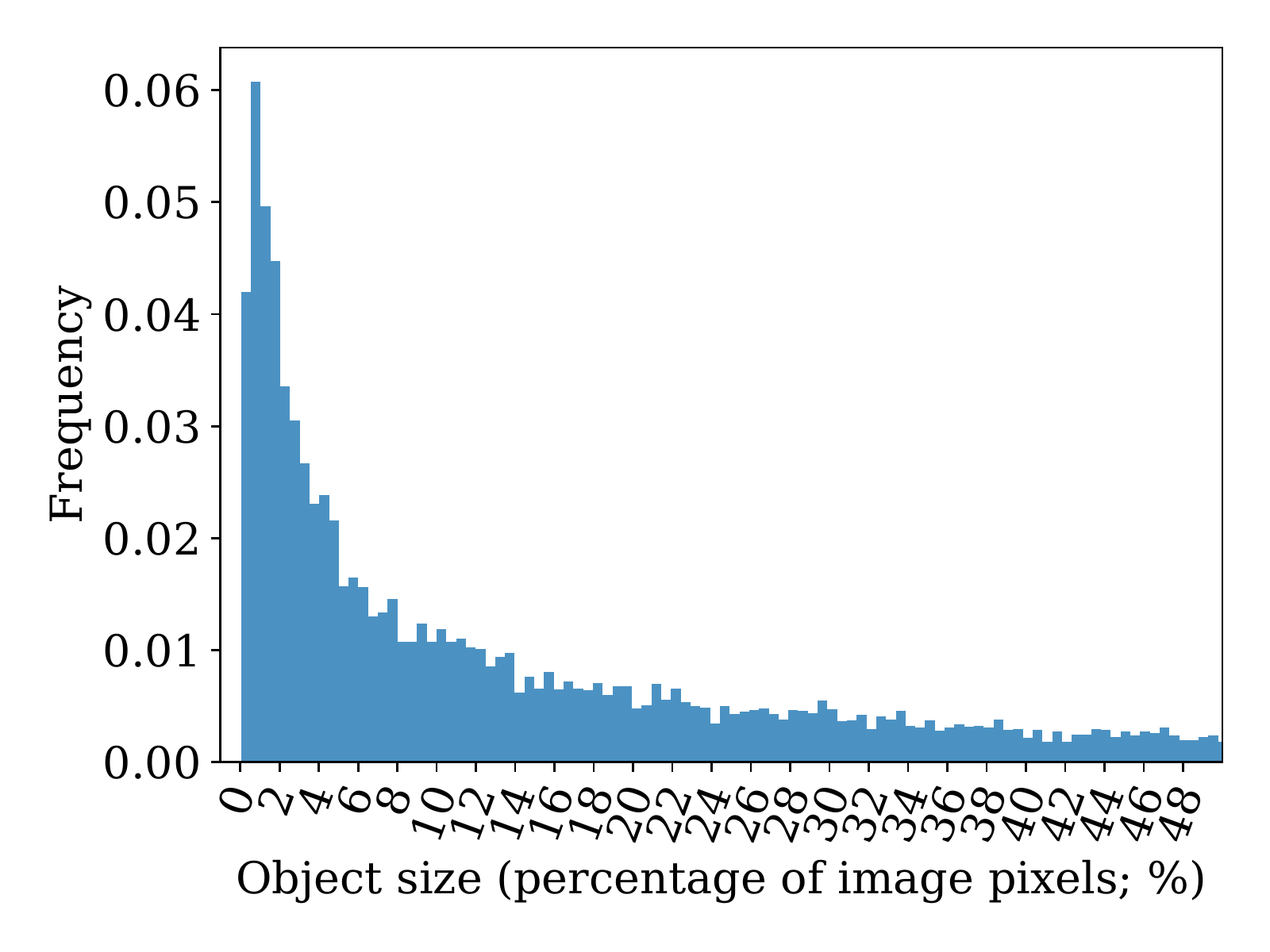}
    \vspace{-1em}
    \caption{Histogram of object sizes of VOC (in percentage of image pixels)}
    \label{fig-size-voc}
\end{figure}

\begin{figure}[t]
    \centering
    \includegraphics[width=0.66\linewidth]{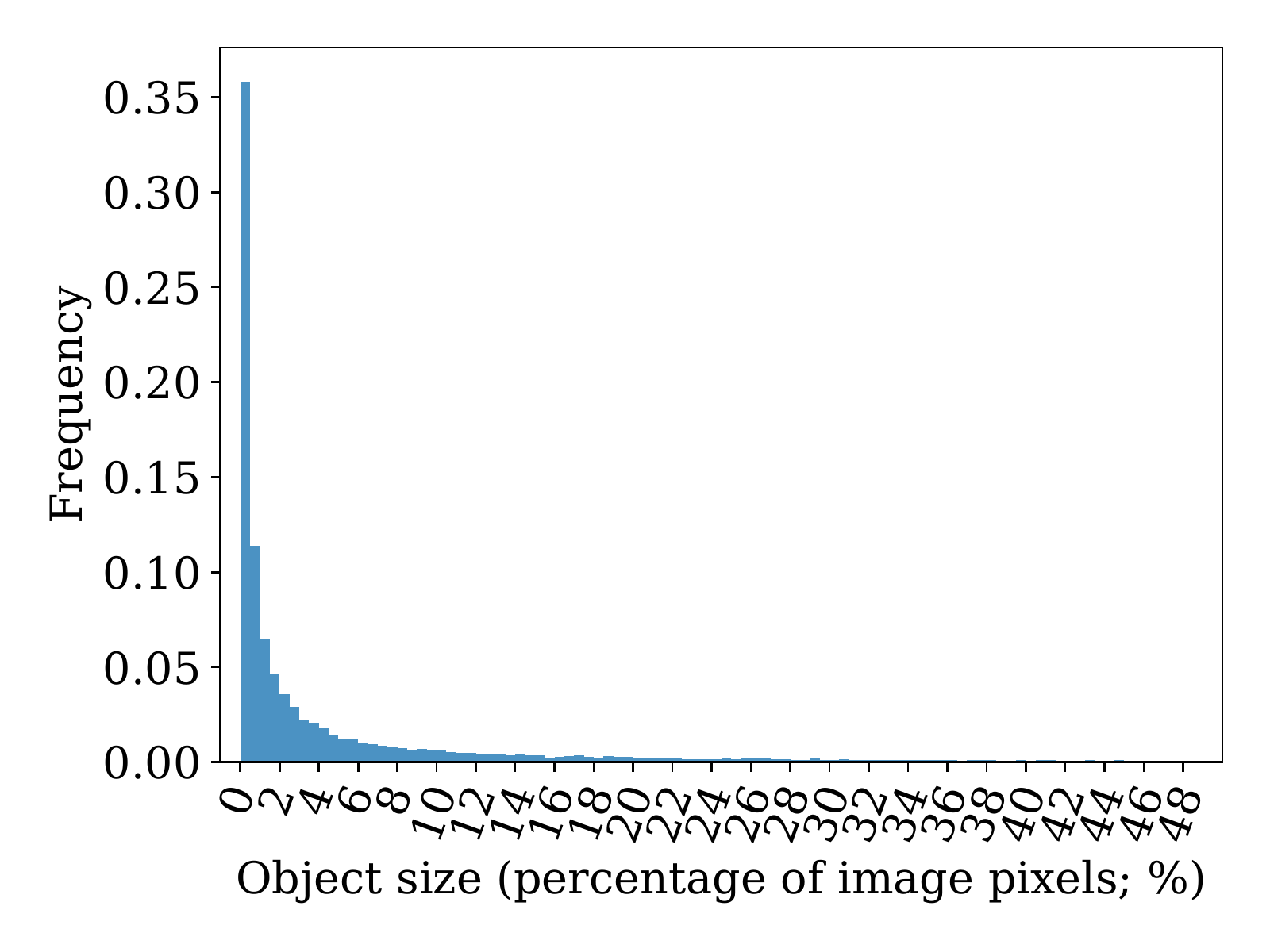}
    \vspace{-1em}
    \caption{Histogram of object sizes of COCO (in percentage of image pixels)}
    \label{fig-size-coco}
\end{figure}

\subsection{Patch Sizes for Evaluation} \label{apx-setup-patch}

\textbf{Patch sizes.} We generally follow DetectorGuard~\cite{xiang2021detectorguard} to choose patch sizes. DetectorGuard~\cite{xiang2021detectorguard} considers a 32$\times$32 square patch on 416$\times$416 images, which consists of 0.6\% image pixels. In this paper, we choose a larger patch consisting of 1\% image pixels to account for a stronger attacker. We note that DPatch~\cite{liu2019dpatch} demonstrated that even a small patch (less than 0.5\% pixels) at the image corner can achieve an effective attack. Therefore, high certifiable robustness against 1\%-pixel patch anywhere on the image is a non-trivial achievement.

Furthermore, in Figure~\ref{fig-size-voc} and Figure~\ref{fig-size-coco}, we further analyze the object sizes on test images of the VOC and COCO datasets. We find that 9.7\% of VOC objects and 47.2\% COCO of objects are smaller than 1\% pixels. Note that we allow the patch to be anywhere on the image; as a result, an over-patch can simply block these small objects and make the defense impossible. This further explains why the certified recall for over-patch is relatively low.

\subsection{DBSCAN Details} \label{apx-setup-dbscan}
As discussed in Section~\ref{sec-defense-merge}, we use DBSCAN~\cite{ester1996density} to cluster unfiltered masked boxes for box unionizing. DBSCAN is an efficient distance-based clustering algorithm with two parameters $\epsilon,n$. It labels a point as a core point when there are at least $n$ other points whose distances with it are smaller than $\epsilon$. All core points and their corresponding neighbors are considered as clusters, and the remaining points are labeled as outliers. In our implementation, we calculate the ``distance" between two boxes $\bfbm,\bfbb$ as $1-\max(\textsc{IoA}(\bfbm,\bfbb),\textsc{IoA}(\bfbb,\bfbm))$. We take the maximum because $\textsc{IoA}(\cdot,\cdot)$ is an asymmetric operator; we add a negative sign to change the similarity score to distance score. We set $\epsilon=0.1,n=1$ (thus no outliers). We do not report defense performance with different DBSCAN parameters because it only has a small effect on the clean performance ($\sim$0.2\%) and no effect on the robustness guarantee.

\section{IoU as Similarity Score Function $\mathbb{S}$}\label{apx-iou}

\begin{algorithm}[t]
    \centering
    \caption{\framework inference algorithm with IoU}\label{alg-inference-iou}
    \begin{algorithmic}[1]
    \renewcommand{\algorithmicrequire}{\textbf{Input:}}
    \renewcommand{\algorithmicensure}{\textbf{Output:}}
    \Require Image $\mathbf{x}$, image size $(W,H)$, base object detector $\mathbb{F}$, the number of lines $k$ (for masking), masked boxes confidence threshold $\gammam$, base boxed confidence threshold $\gammab$, box pruning threshold $\tau_{\text{IoU}},\tau_{\text{IoA}}$
    \Ensure  Robust detection results $\cB_{\text{robust}}$ 
    \Procedure{\framework}{}
    \State $\cM\gets\textsc{MaskSet}(k,W,H)$
    \State $\cB_{\text{mask}}^{\text{no}},\cB_{\text{mask}}^{\text{o}} \gets \varnothing,\varnothing$\label{ln-inferenceiou-init}
    \For{$\bfm \in \cM$}
    \State $\cB^{\text{no}}_\bfm,\cB^{\text{o}}_\bfm \gets \textsc{Split}(\mathbb{F}(\bfx \odot \bfm,\gammam),\bfm)$\label{ln-inferenceiou-mask-prediction}
    \State $\cB_{\text{mask}}^{\text{no}},\cB_{\text{mask}}^{\text{o}}\gets\cB_{\text{mask}}^{\text{no}}\cup\cB^{\text{no}}_\bfm,\cB_{\text{mask}}^{\text{o}}\cup\cB^{\text{o}}_\bfm$\label{ln-inferenceiou-gather-box}
    \EndFor
    \State $\cB_{\text{base}} \gets \mathbb{F}(\bfx,\gammab)$\label{ln-inferenceiou-base}
    \State ${\cB}^{\text{no-filtered}}_{\text{mask}}\gets\{\bfbm \in \cB^{\text{no}}_{\text{mask}} |\nexists\ \bfbb\in\cB_{\text{base}}: \textsc{IoU}(\bfbm,\bfbb)>\tau_{\text{IoU}}\}$\label{ln-inferenceiou-no-filter}
    \State ${\cB}^{\text{no-pruned}}_{\text{mask}}\gets\textsc{NMS}({\cB}^{\text{no-filtered}}_{\text{mask}},\tau_{\text{IoU}})$\label{ln-inferenceiou-no-union}
    \State ${\cB}^{\text{o-filtered}}_{\text{mask}}\gets\{\bfbm \in \cB^{\text{o}}_{\text{mask}} | \nexists\ \bfbb\in\cB_{\text{base}}: \textsc{IoA}(\bfbm,\bfbb)>\tau_{\text{IoA}}\}$ \label{ln-inferenceiou-o-filter}
    \State ${\cB}^{\text{o-pruned}}_{\text{mask}}\gets\{ \bigcup_{\hat{\bfb}\in\hat{\cB}}\hat{\bfb}  \ |\ \hat{\cB} \in \textsc{DBSCAN}({\cB}^{\text{o-filtered}}_{\text{mask}}) \}$\label{ln-inferenceiou-o-union}
   \State $\cB_{\text{robust}}\gets\cB_{\text{base}}\cup{\cB}^{\text{no-pruned}}_{\text{mask}}\cup{\cB}^{\text{o-pruned}}_{\text{mask}}$\label{ln-inferenceiou-output}
    \State \Return $\cB_{\text{robust}}$ 
    \EndProcedure
    
\end{algorithmic}
\end{algorithm}

In Section~\ref{sec-defense}, we discussed the use of box similar score $\mathbb{S}$ and noted that different $\mathbb{S}$ can provide different robustness notions. In this section, we discuss a variant of \framework that uses IoU as the similarity score $\mathbb{S}$ and can certify IoU robustness for the \textit{far-patch model}. 

The core idea of this \framework variant is to split boxes detected on the masked image (masked boxes) into two groups and apply different pruning strategies (with different $\mathbb{S}$) to two groups. The first group contains the boxes that are far away from and do not overlap with the masks (termed as \textit{non-overlapping boxes}); the second group considers the boxes that are close to or overlap with the masks (termed as \textit{overlapping boxes}). 

When the mask is far away from the given object, the detection of this object is barely affected. As a result, detected non-overlapping masked boxes are almost the same and thus have high pair-wise IoU with each other. We can then easily identify and prune redundant boxes by looking at the IoU. In Lemma~\ref{lemma-iou} and Lemma~\ref{lemma2-iou}, we can further prove that the IoU-based pruning strategy can provide a certifiable guarantee for IoU certification. 
After pruning the non-overlapping boxes, we use a similar IoA pruning strategy discussed in the main body to prune overlapping boxes. These boxes can be useful for certifying IoA robustness. We note that we cannot use IoU to prune overlapping boxes because a good number of overlapping boxes only detect part of the object. Therefore, their pair-wise IoU can be too low for effective pruning.

\textbf{Pseudocode.} We present the algorithm pseudocode in Algorithm~\ref{alg-inference-iou}. We first generate the mask set $\cM$ and initialize two empty sets $\cB_{\text{mask}}^{\text{no}},\cB_{\text{mask}}^{\text{o}}$ for holding non-overlapping and overlapping masked boxes (Line~\ref{ln-inferenceiou-init}). Next, we iterate over every mask $\bfm\in\cM$. For each mask, we perform object detection on the masked image $\mathbb{F}(\bfx\odot\bfm,\gammam)$ and split the detected boxes into non-overlapping boxes $\cB_{\bfm}^{\text{no}}$ and overlapping boxes $\cB_{\bfm}^{\text{o}}$ (Line~\ref{ln-inferenceiou-mask-prediction}). We then add $\cB_{\bfm}^{\text{no}}$ and $\cB_{\bfm}^{\text{o}}$ to $\cB_{\text{mask}}^{\text{no}}$ and $\cB_{\text{mask}}^{\text{o}}$, respectively (Line~\ref{ln-inferenceiou-gather-box}).
\begin{table*}[t]
    \centering
    \caption{Defense performance of \framework with IoU robustness (Algorithm~\ref{alg-inference-iou})}  \label{tab-iou-eval}
    \vspace{-1em}
{ \small
\begin{tabular}{l|c|c|c|c|c}
    \toprule
     \multicolumn{2}{c|}{Dataset} & \multicolumn{2}{c|}{PASCAL VOC}&\multicolumn{2}{c}{MS COCO}\\
    \midrule
 \multirow{2}{*}{\diagbox{Defense}{Metric}}&Certify &\multirow{2}{*}{$\text{AP}_{50}$} & {IoU-CertR@0.8 ($T=0.5$)}   &\multirow{2}{*}{$\text{AP}_{50}$} & {IoU-CertR@0.6 ($T=0.5$)}  \\
   &class?&&{far-patch} & &{far-patch}  \\
    \midrule
 
 \framework-YOLOR&\multirow{2}{*}{\cmark}&92.8\%&48.3\%&69.2\%&37.8\%\\
 \framework-Swin &&92.0\%&34.9\%&68.6\%&28.6\%\\
      \bottomrule
    \end{tabular}}
  
\end{table*}
After gathering all masked boxes, we further get base boxes $\cB_{\text{base}}$ prediction on the original image (Line~\ref{ln-inferenceiou-base}) and perform secure box pruning on $\cB_{\text{mask}}^{\text{no}}$ (Line~\ref{ln-inferenceiou-no-filter}-\ref{ln-inferenceiou-no-union}) and $\cB_{\text{mask}}^{\text{o}}$ (Line~\ref{ln-inferenceiou-o-filter}-\ref{ln-inferenceiou-o-union}) separately. For non-overlapping boxes $\cB_{\text{mask}}^{\text{no}}$, we first perform box filtering with IoU as the box similarity score $\mathbb{S}$ (Line~\ref{ln-inferenceiou-no-filter}). Next, we perform non-maximum suppression for box clustering and box representing (Line~\ref{ln-inferenceiou-no-union}). The non-maximum suppression works as follows. First, it picks the box $\bfb_0$ with the highest confidence value, finds all boxes $\bfb_1$ satisfying $\textsc{IoU}(\bfb_0,\bfb_1)>\tau_{\text{IoU}}$, forming a cluster with these boxes. Second, it repeats the first step on remaining unclustered boxes until all boxes are clustered. Third, for each box cluster, it takes the box with the highest confidence as the representative. For overlapping boxes $\cB_{\text{mask}}^{\text{o}}$, we perform box filtering (Line~\ref{ln-inferenceiou-o-filter}) and box unionizing (Line~\ref{ln-inferenceiou-o-union}) with IoA, similar to what we discussed in Section~\ref{sec-defense-merge}. 

Finally, we combine base boxes  $\cB_{\text{base}}$, pruned non-overlapping boxes ${\cB}^{\text{no-pruned}}_{\text{mask}}$ and overlapping boxes ${\cB}^{\text{o-pruned}}_{\text{mask}}$ together as the final prediction output $\cB_{\text{robust}}$ (Line~\ref{ln-inferenceiou-output}).

\textbf{Certification.} The certification with IoU robustness is similar to what we discussed for IoA in Section~\ref{sec-defense-certification}. From Line~\ref{ln-inferenceiou-no-filter}-\ref{ln-inferenceiou-no-union} of Algorithm~\ref{alg-inference-iou}, a masked box $\bfbm$ will only be removed when there is another box $\bfbb$ that has a large IoU with $\bfbm$. Therefore, given a masked box $\bfbm$ and a ground-truth box $\bfbgt$, we only need to prove the lower bound of $\textsc{IoU}(\bfbgt,\bfbb)$ for any box $\bfbb$ that can remove $\bfbm$ (i.e., $\textsc{IoU}(\bfbm,\bfbb)>\tau$).

\begin{lemma}\label{lemma-iou}
For any ground-truth box $\bfbgt$, detected masked box $\bfbm$, and box $\bfbb$ with $\textsc{IoU}(\bfbm,\bfbb)>\tau$, we have:

$$
\resizebox{\hsize}{!}{$\textsc{IoU}(\bfbgt,\bfbb) > \frac{\tau B+(\tau-1)\cdot C}{A+B+C},\ \forall\ \bfbb \st \textsc{IoU}(\bfbm,\bfbb)>\tau$}
$$
where $A=|\bfbgt\setminus\bfbm|,B=|\bfbgt\cap\bfbm|,C=|\bfbm\setminus\bfbgt|$.
\end{lemma}
\begin{proof}
In this proof, we are going to find the worst-case $\bfbb$ that satisfies the condition of $\textsc{IoU}(\bfbm,\bfbb)>\tau$ but gives the lowest $\textsc{IoU}(\bfbgt,\bfbb)$.

First, let us divide the box $\bfbb$ into four disjoint parts using set operations.
\begin{align}
    \bfbb=&(\bfbb\cap(\bfbgt\cup\bfbm))\cup(\bfbb\setminus(\bfbgt\cup\bfbm))\nonumber\\
    =&(\bfbb\cap(\bfbgt\cap\bfbm))\cup(\bfbb\cap(\bfbm\setminus\bfbgt))\cup(\bfbb\cap(\bfbgt\setminus\bfbm))\nonumber\\
    &\cup(\bfbb\setminus(\bfbgt\cup\bfbm))\nonumber
\end{align}
To reason the worst-case $\textsc{IoU}(\bfbgt,\bfbb)$, we can set the second part $|\bfbb\cap(\bfbm\setminus\bfbgt)|$ to $|(\bfbm\setminus\bfbgt)|=C$ because increasing its area increases $\textsc{IoU}(\bfbm,\bfbb)$ but not $\textsc{IoU}(\bfbgt,\bfbb)$. We set the third $|\bfbb\cap(\bfbgt\setminus\bfbm)|=0$ because decreasing its area decreases $\textsc{IoU}(\bfbgt,\bfbb)$ but not $\textsc{IoU}(\bfbm,\bfbb)$. 

Next, we set the remaining two parts as $|\bfbb\cap(\bfbgt\cap\bfbm)|=\alpha,|\bfbb\setminus(\bfbgt\cup\bfbm)|=\beta$. We can then write our constraint sets as
\begin{align}
    \textsc{IoU}(\bfbm,\bfbb) = \frac{|\bfbm\cap\bfbb|}{|\bfbm\cup\bfbb|}=\frac{C+\alpha}{B+C+\beta}\geq\tau
    \\\alpha\geq0,\beta\geq0
\end{align}
We can further write the target as:
\begin{align}
   t:= \min_{(\alpha,\beta)}\textsc{IoU}(\bfbgt,\bfbb)=\frac{\alpha}{A+B+C+\beta}
\end{align}
Now it is a linear programming problem for two variables $\alpha,\beta$. We can solve it and get the optimal solution as:
\begin{align}
    t^*=\max(0,\frac{\tau B+(\tau-1)\cdot C}{A+B+C}):=\mathbb{L}_{\textsc{IoU}}(\bfbgt,\bfb_\bfm,\tau)\\
    A=|\bfbgt\setminus\bfbm|,B=|\bfbgt\cap\bfbm|,C=|\bfbm\setminus\bfbgt|\nonumber
    \end{align}
when $\beta=0,\alpha =\max(0,\tau B+(\tau-1)\cdot C)$.
\end{proof}

We use $\mathbb{L}_{\textsc{IoU}}(\bfbgt,\bfb_\bfm,\tau)$ to denote this new bound for IoU. Next, we present the following lemma to demonstrate that $\mathbb{L}_{\textsc{IoU}}>T$ is the sufficient condition of being a pruning-safe masked box.

\begin{lemma}\label{lemma2-iou}
Given a ground-truth object $\bfbgt$ and the IoU pruning threshold $\tau$, if there is one masked box $\bfbm\in\cB^{\text{no}}_{\text{mask}}$ satisfying that $\mathbb{L}_{\textsc{IoU}}(\bfbgt,\bfbm,\tau)>T$, this box is a pruning-safe masked box for object $\bfbgt$ in \textsc{IoU-BoxPrune}, i.e., $\exists\ \bfb^\prime\in\cB_{\text{robust}}\st\textsc{IoU}(\bfbgt,\bfb^\prime)>T$.
\end{lemma}
\begin{proof}
The detected masked box $\bfbm\in\cB^{\text{no}}_{\text{mask}}$ will go through box filtering and box unionizing to generate the final output. 
There are three possible cases.
\begin{enumerate}  \setlength\itemsep{0em}
\item $\bfbm$ is filtered in the box filtering process (Line~\ref{ln-inferenceiou-no-filter}). Then there is a base box $\bfbb\in\cB_{\text{base}}$ satisfying $\textsc{IoU}(\bfbm,\bfbb)>\tau$. From Lemma~\ref{lemma-iou}, we know there is a box $\bfb^\prime=\bfbb\in\cB_{\text{base}}\subset\cB_{\text{robust}}$ satisfies $\textsc{IoU}(\bfbgt,\bfb^\prime)> \mathbb{L}_{\textsc{IoU}}(\bfbgt,\bfbm,\tau)>T$. 
\item $\bfbm$ is removed in the box unionizing step (Line~\ref{ln-inferenceiou-no-union}). This implies that there is another masked box $\bfbm^\prime\in\cB^{\text{no-filtered}}_{\text{mask}}$ satisfying $\textsc{IoU}(\bfbm,\bfbm^\prime)>\tau$. From Lemma~\ref{lemma-iou}, we know there is a box $\bfb^\prime=\bfbm^\prime\in\cB^{\text{no-pruned}}_{\text{mask}}\subset\cB_{\text{robust}}$ satisfies $\textsc{IoU}(\bfbgt,\bfb^\prime)> \mathbb{L}_{\textsc{IoU}}(\bfbgt,\bfbm,\tau)>T$. 
\item $\bfbm$ is not removed and becomes a part of $\cB^{\text{no-pruned}}_{\text{mask}}\subset\cB_{\text{robust}}$. Then we know there is a box $\bfb^\prime=\bfbm\in\cB_{\text{robust}}$ such that $\textsc{IoU}(\bfbgt,\bfb^\prime) = |{\bfbgt} \cap \bfbm|/|\bfbgt\cup\bfbm| \geq \mathbb{L}_{\textsc{IoA}}(\bfbgt,\bfbm,\tau)>T$.
\end{enumerate} 
\vspace{-1.6em}
\end{proof}

With Lemma~\ref{lemma-iou} and Lemma~\ref{lemma2-iou}, the certification is straightforward: we only need to replace the certification condition with the new bound $ \mathbb{L}_{\textsc{IoU}}(\bfbgt,\bfb_\bfm,\tau)>T$ in Line~\ref{ln-certification-condition} of Algorithm~\ref{alg-certification}. 

\textbf{Implementation and performance evaluation.} In our implementation, we consider a box is a non-overlapping box if the smallest distance between the masked box and the mask along x and y axes is larger than 5\% of the range of the corresponding axis (i.e., height or width). We use the same setup as in Section~\ref{sec-eval} to evaluate Algorithm~\ref{alg-inference-iou}. We set $\tau_{\text{IoU}}=0.8$ and the IoU certification threshold $T_{\text{IoU}}=0.5$ and report the defense performance in Table~\ref{tab-iou-eval}. As shown in the table, Algorithm~\ref{alg-inference-iou} has similarly high clean AP as Algorithm~\ref{alg-inference} (recall Table~\ref{tab-main-eval}), and more importantly, achieves the first certified IoU robustness against patch hiding attacks.

\section{Taxonomy of Robustness Notions against Hiding Attacks}\label{apx-taxonomy}

In the main body of this paper, we focus on IoA robustness; in Appendix~\ref{apx-iou}, we further presented a \framework variant that has IoU robustness against far-patch attackers. In this section, we aim to provide a taxonomy of different robustness notions against hiding attacks to shed a light on future research. We summarize four major robustness notions in Table~\ref{tab-taxonomy}, which are categorized based on two important robustness factors as discussed next.

\textbf{Factor 1: attack detection vs. robust prediction.} The first factor is the defense format: attack detection versus robust prediction. An attack-detection defense aims to detect an attack: it alerts and abstains from making predictions when it detects an attack. That is, we allow the defense to output a special symbol $\bot$ for cases when it detects an attack. In contrast, a robust-prediction defense does not involve the abstention symbol $\bot$: it has to always mask robust predictions. Clearly, robust-prediction defenses achieve a stronger robustness notion -- a robust-prediction defense can directly reduce to an attack-detection defense that never alerts. We note that DetectorGuard~\cite{xiang2021detectorguard} is an attack-detection defense (Notion I) while \framework aims to build a robust-prediction defense (Notion II and IV).

\begin{table}[t]
    \caption{Taxonomy of robustness notions}
    \label{tab-taxonomy}
    \centering
        \vspace{-1em}
    \resizebox{\linewidth}{!}
    {\begin{tabular}{c|c|c}
    \toprule
         & Attack detection & Robust prediction \\
         \midrule
    \multirow{2}{*}{IoA robustness} & \textbf{Notion I} & \textbf{Notion II} \\
      & DetectorGuard~\cite{xiang2021detectorguard} & \framework (Section~\ref{sec-defense}) \\
    \midrule
    \multirow{2}{*}{IoU robustness} & \textbf{Notion III} & \textbf{Notion IV}\\
    &--&\framework (Appendix~\ref{apx-iou})\\
    \bottomrule
    \end{tabular}}

\end{table}

\textbf{Factor 2: IoA robustness vs. IoU robustness.} The second factor is about the guarantee for the box quality: IoA robustness versus IoU robustness. In the main body of this paper, we consider IoA robustness because it aligns with the defense objective against patch hiding attacks: we aim to detect at least part of the object, i.e., $\textsc{IoA}(\bfbgt,\bfb^\prime)=|\bfbgt\cap\bfb^\prime|/|\bfbgt|>T$.

However, a defense with IoA robustness might end up predicting large bounding boxes that cover the entire image. When this happens, the IoA robustness is satisfied (IoA equals to 1), but the defense output is not ideal: we only know that there are objects of certain classes but do not know the object locations. Therefore, we are motivated to consider the concept of IoU robustness, i.e., $\textsc{IoU}(\bfbgt,\bfb^\prime)=|\bfbgt\cap\bfb^\prime|/|\bfbgt\cup\bfb^\prime|>T$, which adds additional constraints on the sizes of predicted boxes. In Appendix~\ref{apx-iou}, we discuss a variant of \framework that can certify IoU robustness against far-patch attackers.

\textbf{Remark: non-triviality of IoA robustness and high clean performance.} Despite the subtle issues of IoA robustness discussed above, we note that building defense with IoA robustness and high clean performance is non-trivial. Yes, if defenders know that there is going to be an attack, they can trivially output large boxes for IoA robustness. However, in practice, we do not know if the input image is a benign normal image or an adversarially patched image. If defenders always output large boxes, the performance on clean images will be bad: note that we use the conventional IoU to evaluate clean performance (when no attacks happen); IoA is only for robustness evaluation. Moreover, IoA robustness requires the correctness of box class labels. Achieving label correctness is also non-trivial.

\textbf{Remark: the connection between Notion II and Notion I.} As shown in Table~\ref{tab-taxonomy}, both Notion I and Notion II use IoA robustness, and they differ in the defense formats of attack detection versus robust prediction. Here, we can demonstrate the equivalence of two notions in terms of certified robustness.

First, we can build a Notion I defense with a Notion II defense. This is directly implied by the definition: any Notion II defense is a valid Notion I defense that relinquishes the freedom of issuing alerts. 

Second, counterintuitively, we can also build a Notion II defense with a Notion I defense. The strategy is that, if the Notion I defense detects an attack, instead of issuing an alert, it outputs large bounding boxes of all different classes to cover the entire image. Since these large boxes have IoAs of 1 for any object of the same object class, the defense is considered IoA-robust.\footnote{We note that an extremely large box has a small IoU with the ground-truth box; therefore, this strategy of outputting large boxes does not apply to Notion III and Notion IV where we consider IoU robustness.}

\textit{Note 1: empirical advantage of Notion II defenses.} Despite the equivalence in certifiable robustness with Notion I defenses, Notion II defenses still have advantages in terms of empirical performance. For example, when there is a false alert in the clean setting, a Notion I based defense would either alert or output useless large boxes while a Notion II defense can still predict reasonable boxes (e.g. boxes that obtain high IoU with the ground-truth objects).

\textit{Note 2: the necessity of box unionizing in \framework.} Given the reduction strategies between Notion I and Notion II defenses, one might question the necessity and value of the box unionizing module in our \framework design (since it can be replaced with an attack alert module or a large box outputting module with no cost of certifiable robustness). Here, we want to reiterate that \framework is a defense framework and is compatible with different box similarity score functions $\mathbb{S}$ such as IoA and IoU, while the reduction problem discussed above is only for the IoA robustness. When we use IoU as the score function $\mathbb{S}$ and consider IoU robustness, the box unionizing is necessary since we can no longer output large boxes discussed for IoA. Another benefit of having box unionizing is the empirical advantage discussed in the paragraph above.

\textbf{Remark: IoA and IoU robustness in \framework.} As shown in Table~\ref{tab-taxonomy}, \framework is flexible and can be instantiated with either IoA or IoU. We focus on IoA robustness in the main body for the following reasons. First, IoA robustness is non-trivial and interesting to study (recall the first remark in this section); it intuitively aligns with the objective of mitigating hiding attacks. Second, focusing on IoA enables a fair comparison with DetectorGuard~\cite{xiang2021detectorguard}, whose robustness guarantee  is limited to IoA robustness with $T=0$. Third, though we have competitive IoU-CertR numbers with IoA-CertR against \textit{far-patch} (recall Appendix~\ref{apx-iou}), certifiable IoU robustness against over-patch and close-patch is significantly more challenging to achieve. We will need better pruning strategies to instantiate our framework for better IoU robustness in the future.

\textbf{Factor 3: Protect class labels or not.} In addition to two robustness factors presented in Table~\ref{tab-taxonomy}, we can further categorize robustness notions based on whether the defense protects the class label or not. For example, DetectorGuard~\cite{xiang2021detectorguard} is fundamentally limited in its design not to be able to certify class labels. In contrast, \framework is flexible for the class label certification.

\begin{figure*}
\centering
\begin{minipage}[b]{0.32\linewidth}
\includegraphics[width=\linewidth]{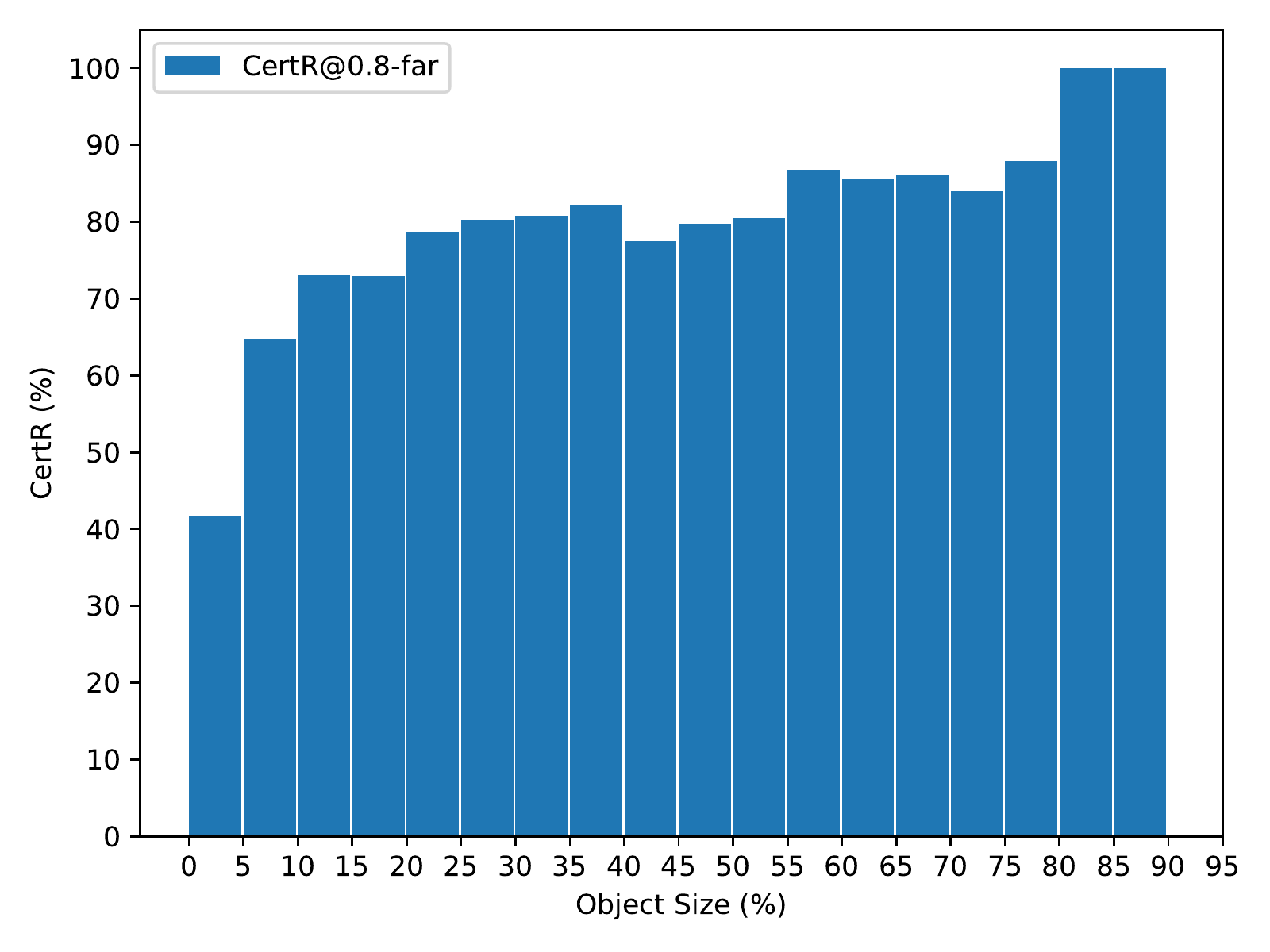}
\vspace{-2em}
\end{minipage}%
\quad
\begin{minipage}[b]{0.32\linewidth}
\includegraphics[width=\linewidth]{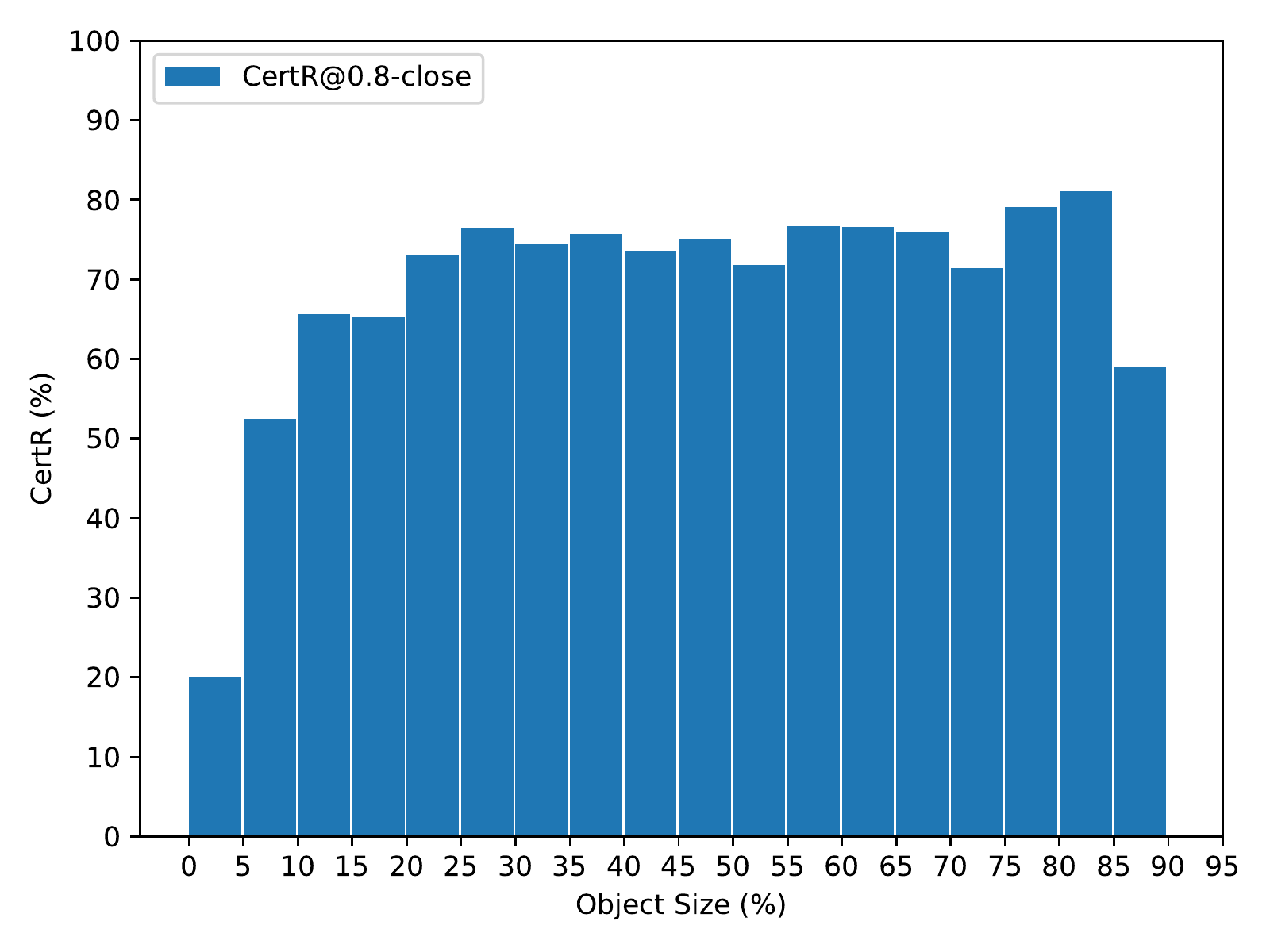}
\vspace{-2em}
\end{minipage}%
\quad
\begin{minipage}[b]{0.32\linewidth}
\includegraphics[width=\linewidth]{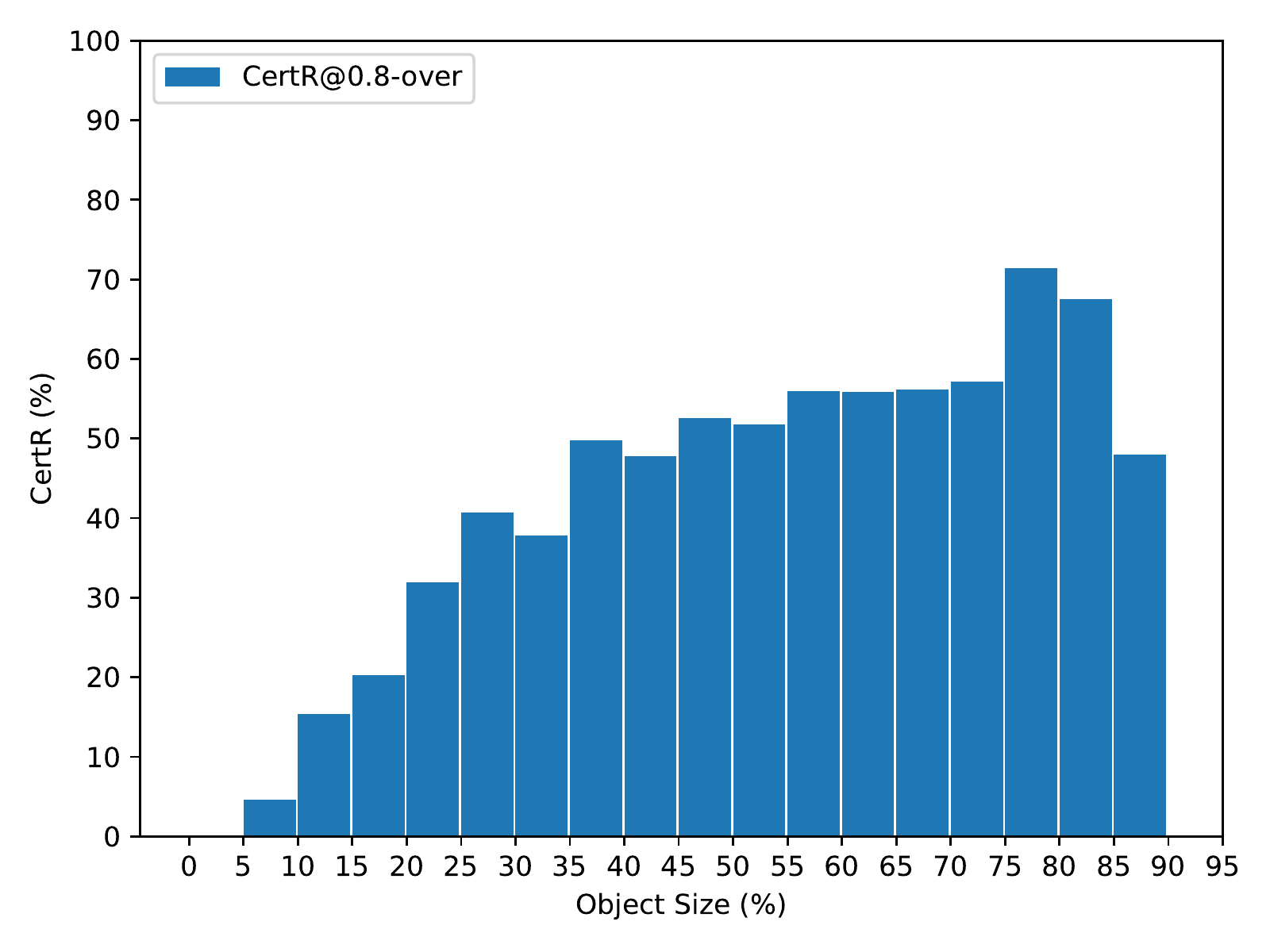}
\vspace{-2em}
\end{minipage}%
\caption{\framework robustness for VOC objects of different sizes (left to right: far-patch, close-patch, over-patch)}
\label{fig-obj-size-voc}
\end{figure*}

\begin{figure*}
\centering
\begin{minipage}[b]{0.32\linewidth}
\includegraphics[width=\linewidth]{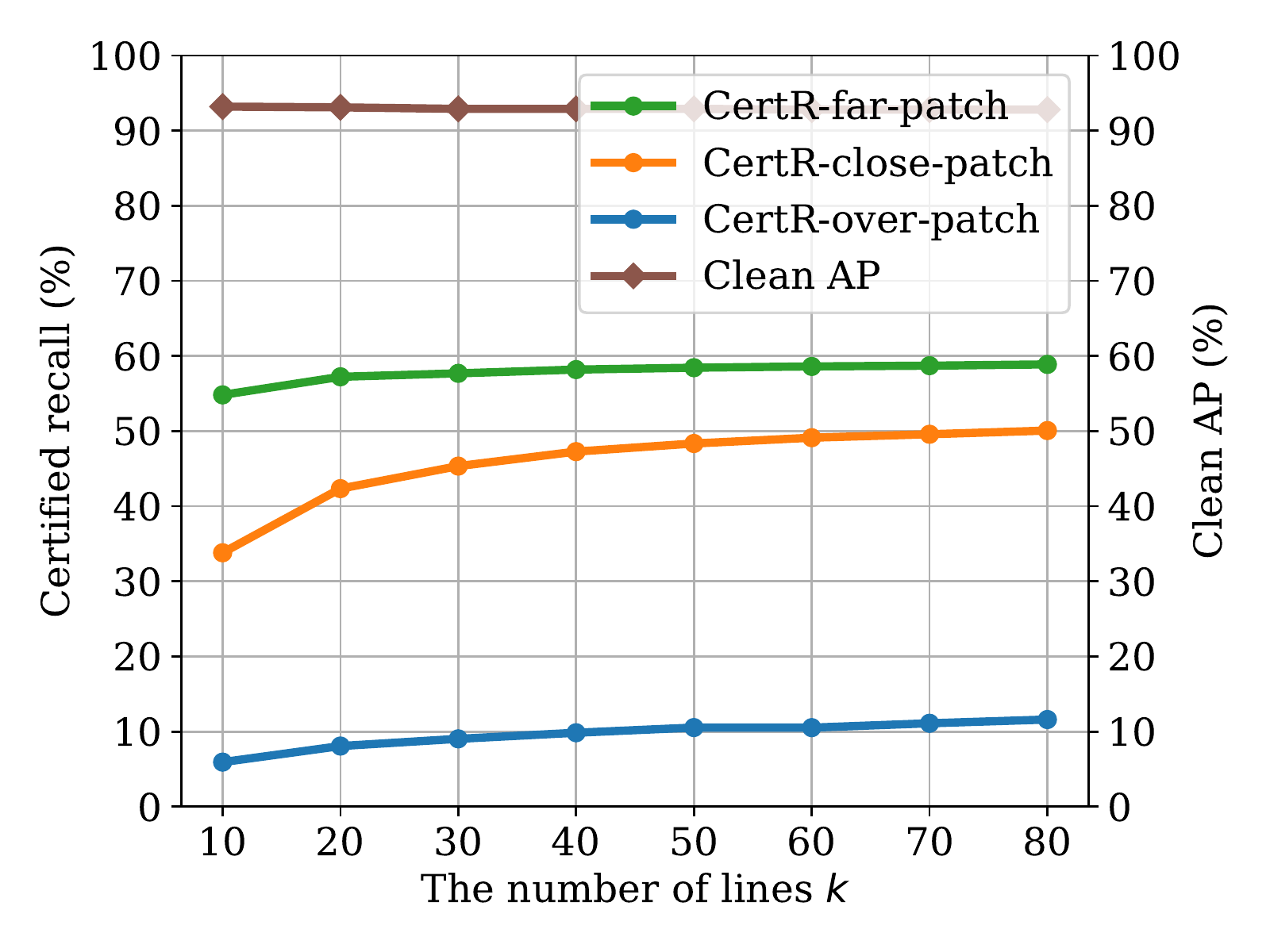}
\vspace{-2em}
\end{minipage}%
\quad
\begin{minipage}[b]{0.32\linewidth}
\includegraphics[width=\linewidth]{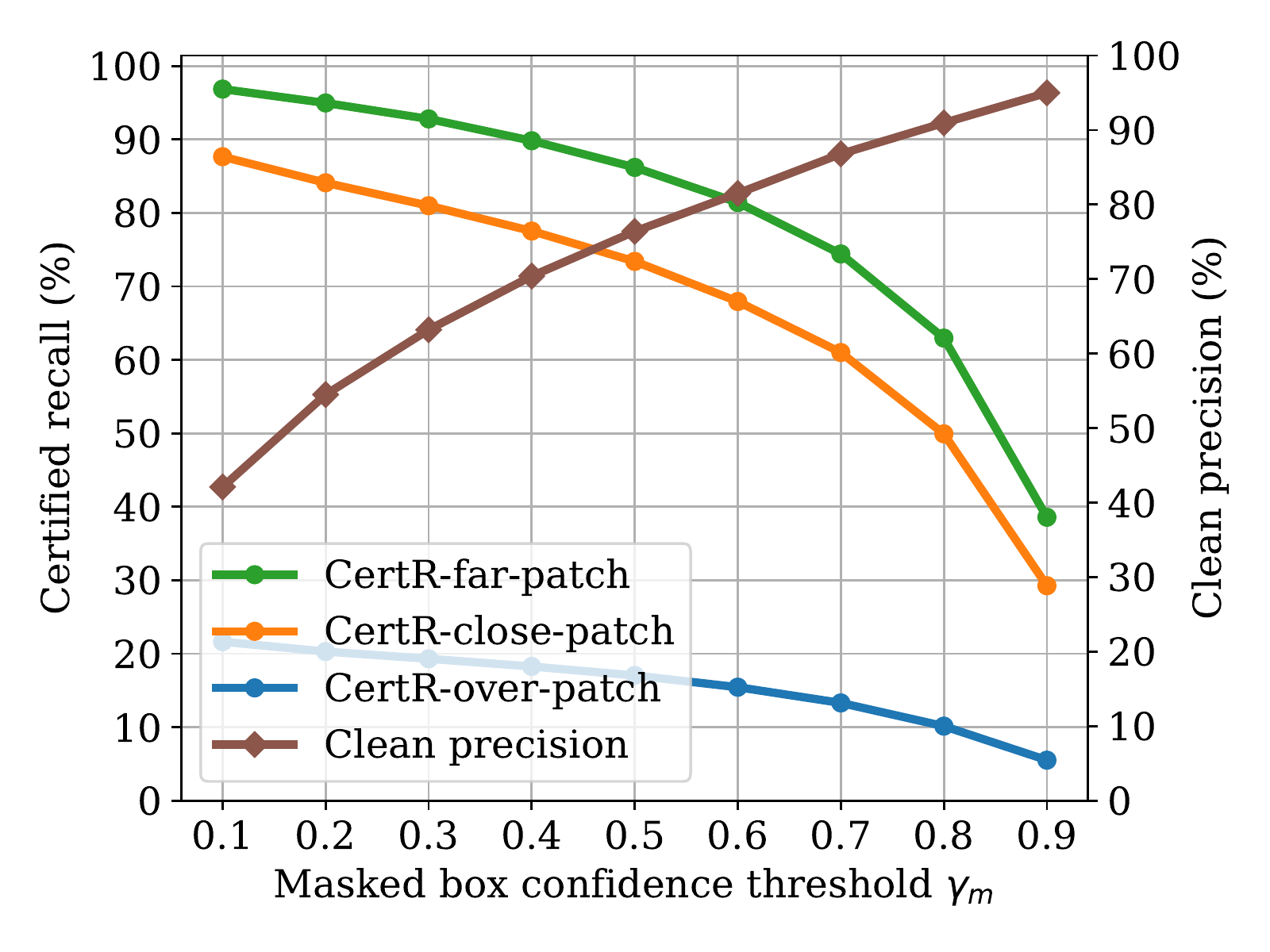}
\vspace{-2em}
\end{minipage}%
\quad
\begin{minipage}[b]{0.32\linewidth}
\includegraphics[width=\linewidth]{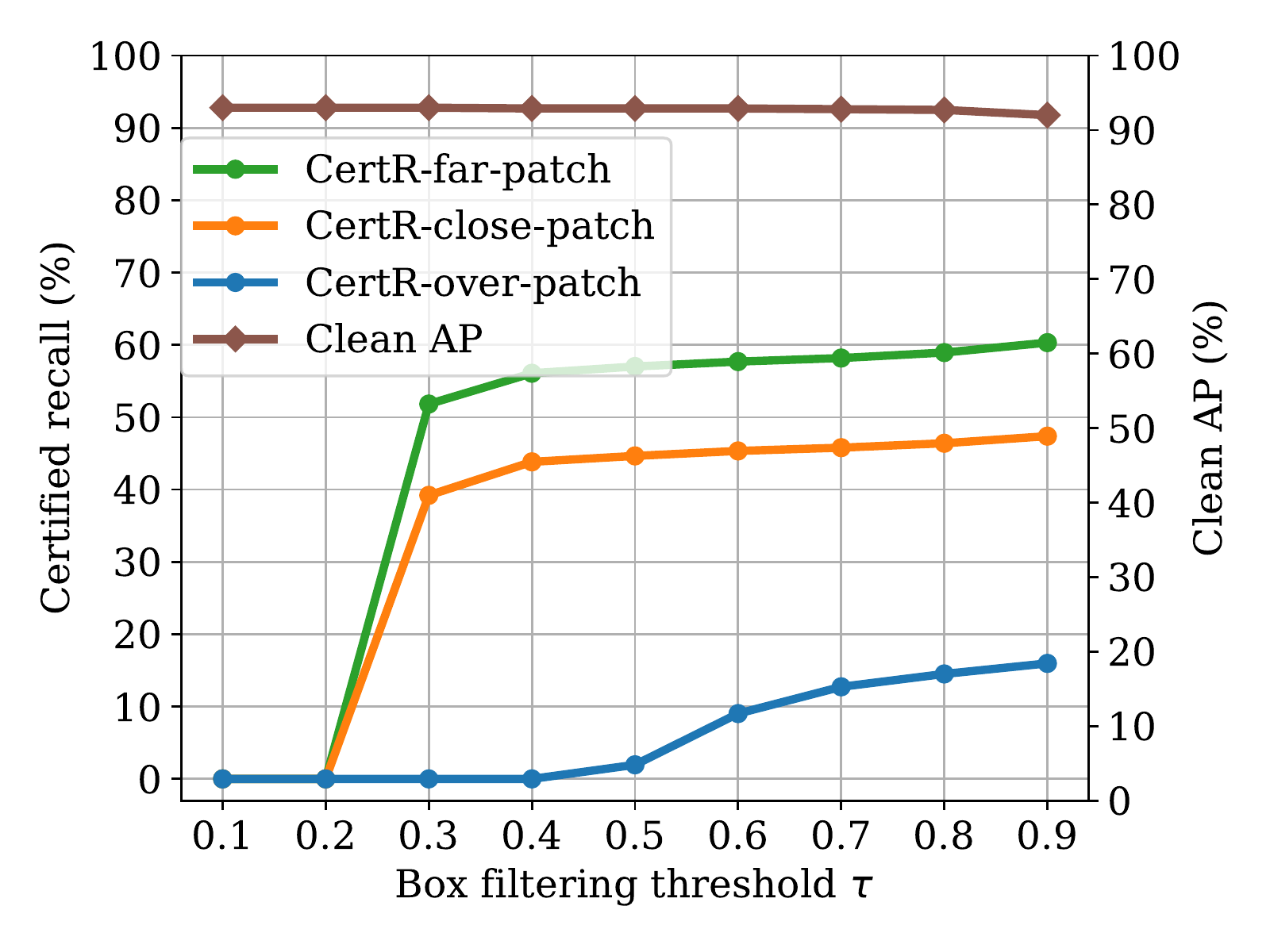}
\vspace{-2em}
\end{minipage}%
\caption{The effects of different hyperparameters on \framework with larger $T=0.2$ (VOC; left to right: $k$, $\gammam$, $\tau$)}
\label{fig-larger-T-voc}
\end{figure*}

\begin{figure}[t]
    \centering
    \includegraphics[width=0.8\linewidth]{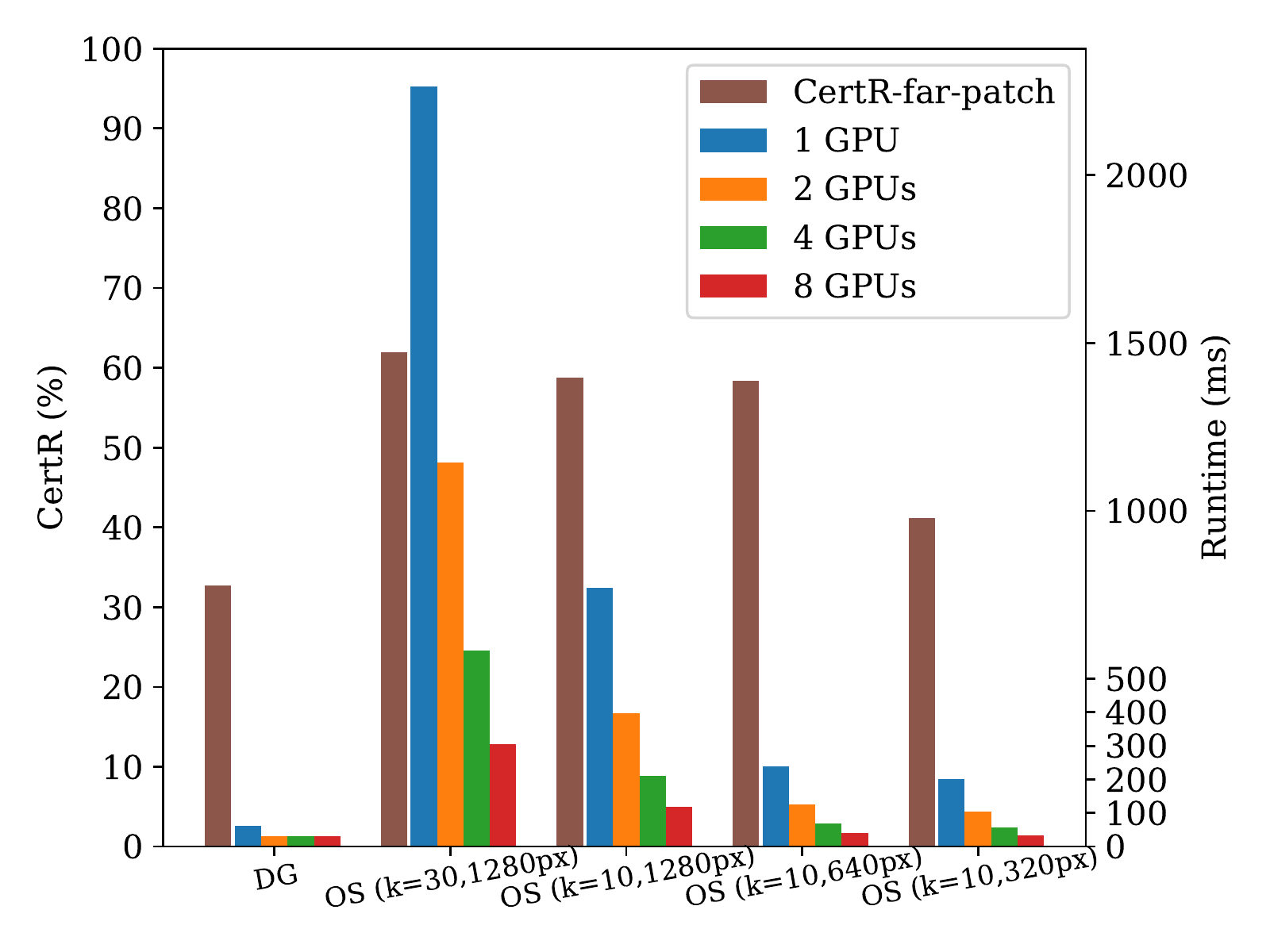}
    \vspace{-1.5em}
    \caption{Wall-lock per-image runtime of DetectorGuard (DG)~\cite{xiang2021detectorguard} and ObjectSeeker (OS) on VOC}
    \label{fig-runtime}
\end{figure}

\textbf{Future of robust object detection.} As a summary of this section, we believe that the ultimate defense objective is a robust-prediction defense with IoU robustness (Notion IV). The defense problem can then be interpreted as a list decoding problem: we aim to output a list of bounding boxes such that a subset of bounding boxes have high IoUs with all ground-truth boxes. Furthermore, to build even stronger defenses, we can also consider a stronger attacker who also wants to increase FP errors. We can deploy an empirical defense via classifying all predicted bounding boxes and removing boxes with inconsistent labels~\cite{xiang2021detectorguard}.

\section{Additional Experimental Discussions}\label{apx-more-discussion}
In this section, we provide quantitative discussions on objects of different sizes, absolute defense runtime, and certification threshold $T$.

\textbf{\framework's robustness for objects of different sizes.} In this analysis, we aim to understand the relationship between robustness and object sizes. We divide VOC objects into different groups based on their sizes (occupying 0-5\%, 5-10\%, $\cdots$, 95-100\% image pixels) and plot their CertR in Figure~\ref{fig-obj-size-voc}. As shown in the figure, larger objects tend to have higher robustness. This is because vanilla object detectors have a better chance to detect larger objects on masked images. We note that the CertR for over-patch is more sensitive to object sizes, partially due to that an over-patch can sometimes occlude the major part of small objects and make robust object detection hard or even impossible.

\textbf{Additional discussion on absolute runtime.} In Figure~\ref{fig-efficiency-voc}, we demonstrated that we could balance the trade-off between efficiency and robustness by tuning the parameter $k$. In this section, we provide further discussions on implementation-level optimizations for absolute runtime. Specifically, we consider using different input image sizes and different numbers of GPUs. In Figure~\ref{fig-runtime}, we report runtime results for DetectorGuard~\cite{xiang2021detectorguard} and \framework using different images sizes, different numbers of NVIDIA RTX A4000 GPUs, and different $k$ (for \framework). First, we can see that using multiple GPUs can significantly reduce wall-clock runtime, given that the inference on masked images is trivially parallelizable. For example, if we set $k=10$ and use an image size of 1280px, using 8 GPUs can reduce runtime from 770.6ms 117.4ms (6.6$\times$ speedup). In contrast, DetectorGuard's runtime improvement with multiple GPUs is limited (up to 2$\times$). Second, we can see that reducing the image size can also significantly improve runtime. For example, if we resize images from 1280px (the default value used in the paper) to 640px, the runtime for \framework with $k=10$ on 8 GPUs improves from 117.4ms to 39.6ms (3.0$\times$ speedup), while the robustness is only slightly affected (from 58.8\% to 58.4\%). However, we note that further reducing the image size (to 320px) can greatly hurt the robustness (to 41.2\%). We note that DetectorGuard~\cite{xiang2021detectorguard}'s robustness module uses fixed image size and has limited benefit from image resizing.

In summary, Figure~\ref{fig-runtime} demonstrates the feasibility to reduce absolute runtime via implementation-level optimizations. We can have a latency of 40.0ms (25fps) on VOC images using $k=10$, an image size of 640px, and 8 GPUs, while maintaining high CertR. In practice, we should carefully configure the \framework defense to meet computation constraints and robustness objectives, as discussed in Section~\ref{sec-discussion-limitation}.

\begin{figure*}
\centering
\begin{minipage}[b]{0.32\linewidth}
\includegraphics[width=\linewidth]{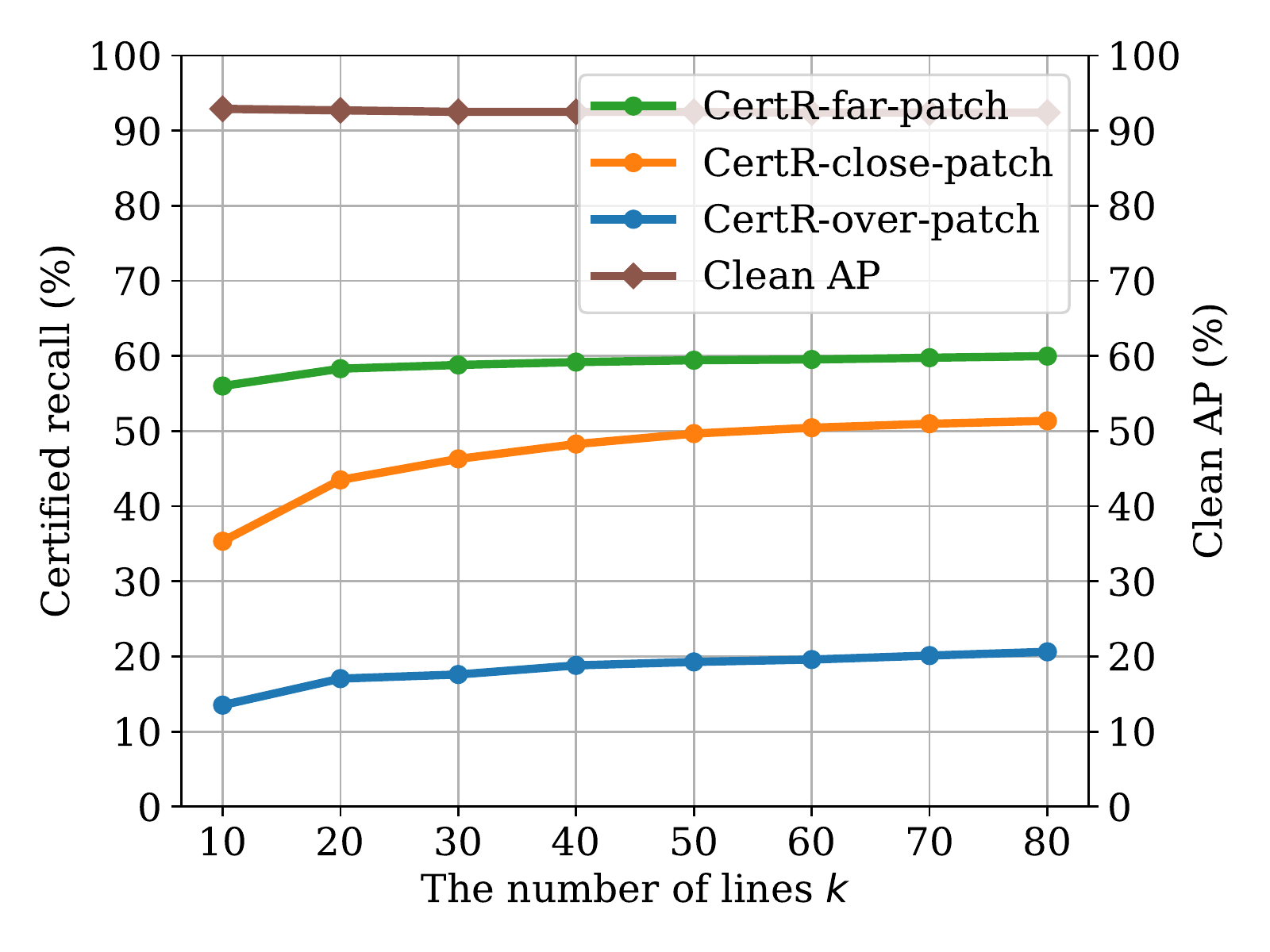}
\vspace{-1em}
\end{minipage}%
\quad
\begin{minipage}[b]{0.32\linewidth}
\includegraphics[width=\linewidth]{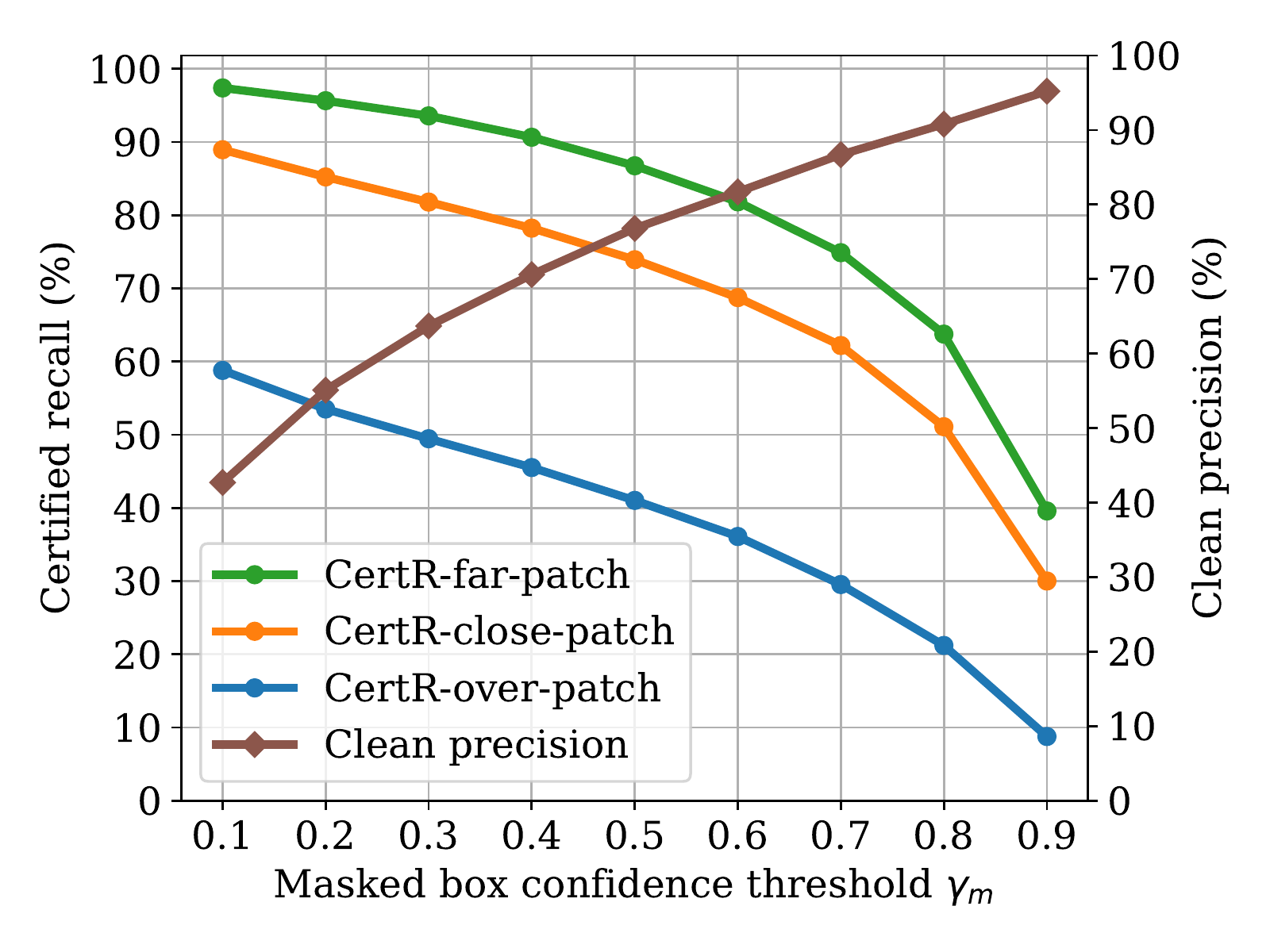}
\vspace{-1em}
\end{minipage}%
\quad
\begin{minipage}[b]{0.32\linewidth}
\includegraphics[width=\linewidth]{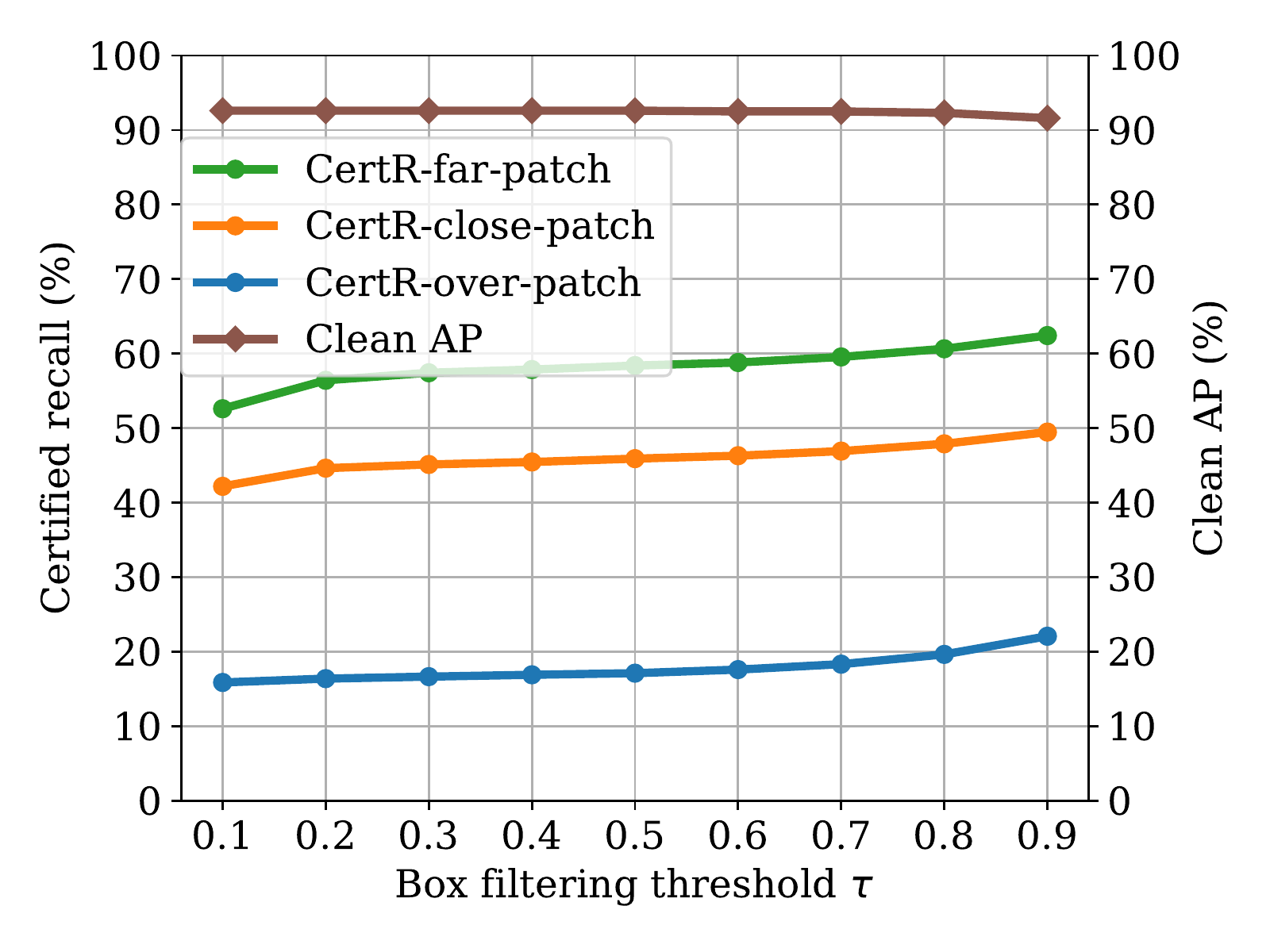}
\vspace{-1em}
\end{minipage}%

\begin{minipage}[b]{0.32\linewidth}
\includegraphics[width=\linewidth]{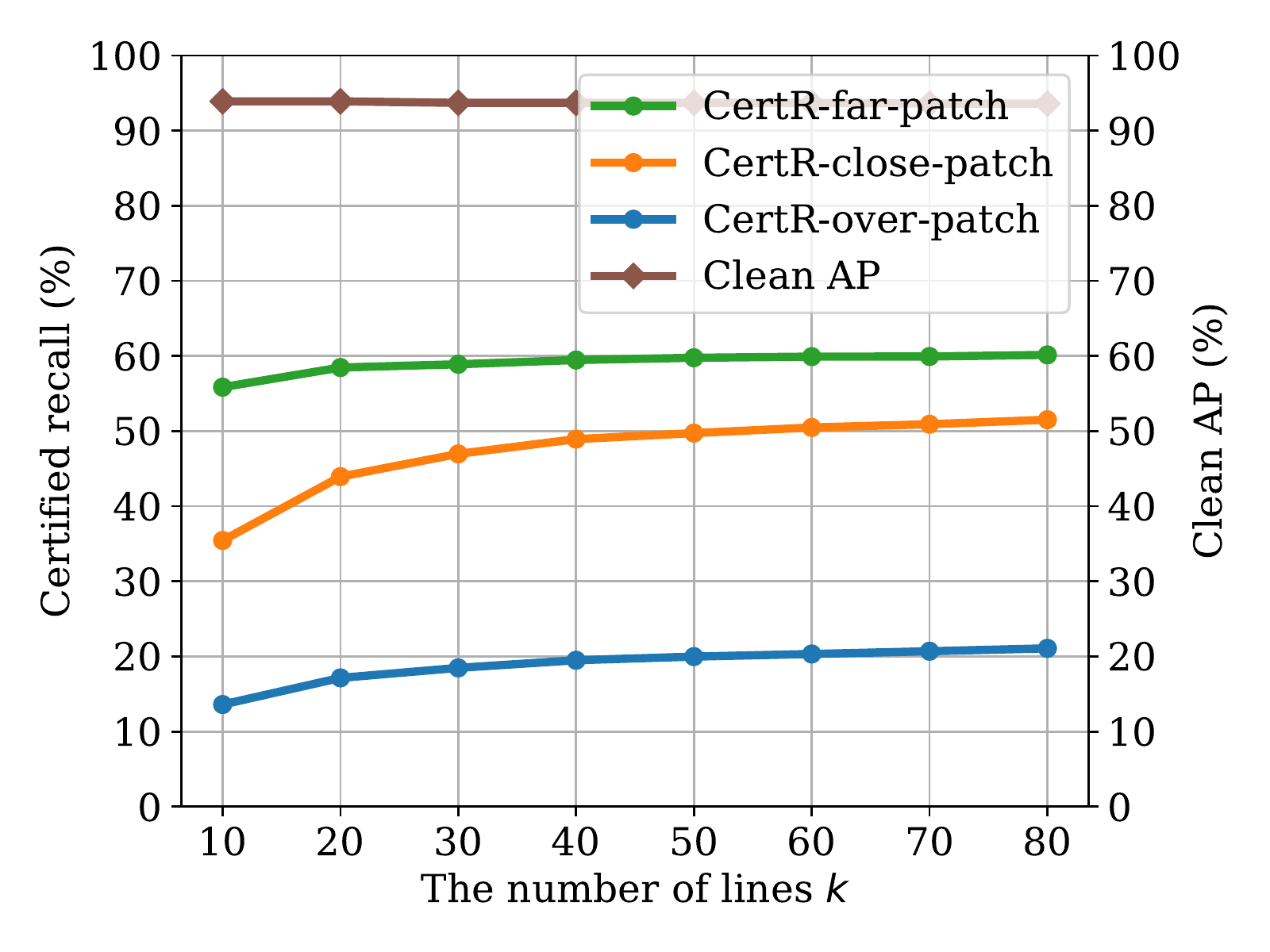}
\vspace{-2em}
\end{minipage}%
\quad
\begin{minipage}[b]{0.32\linewidth}
\includegraphics[width=\linewidth]{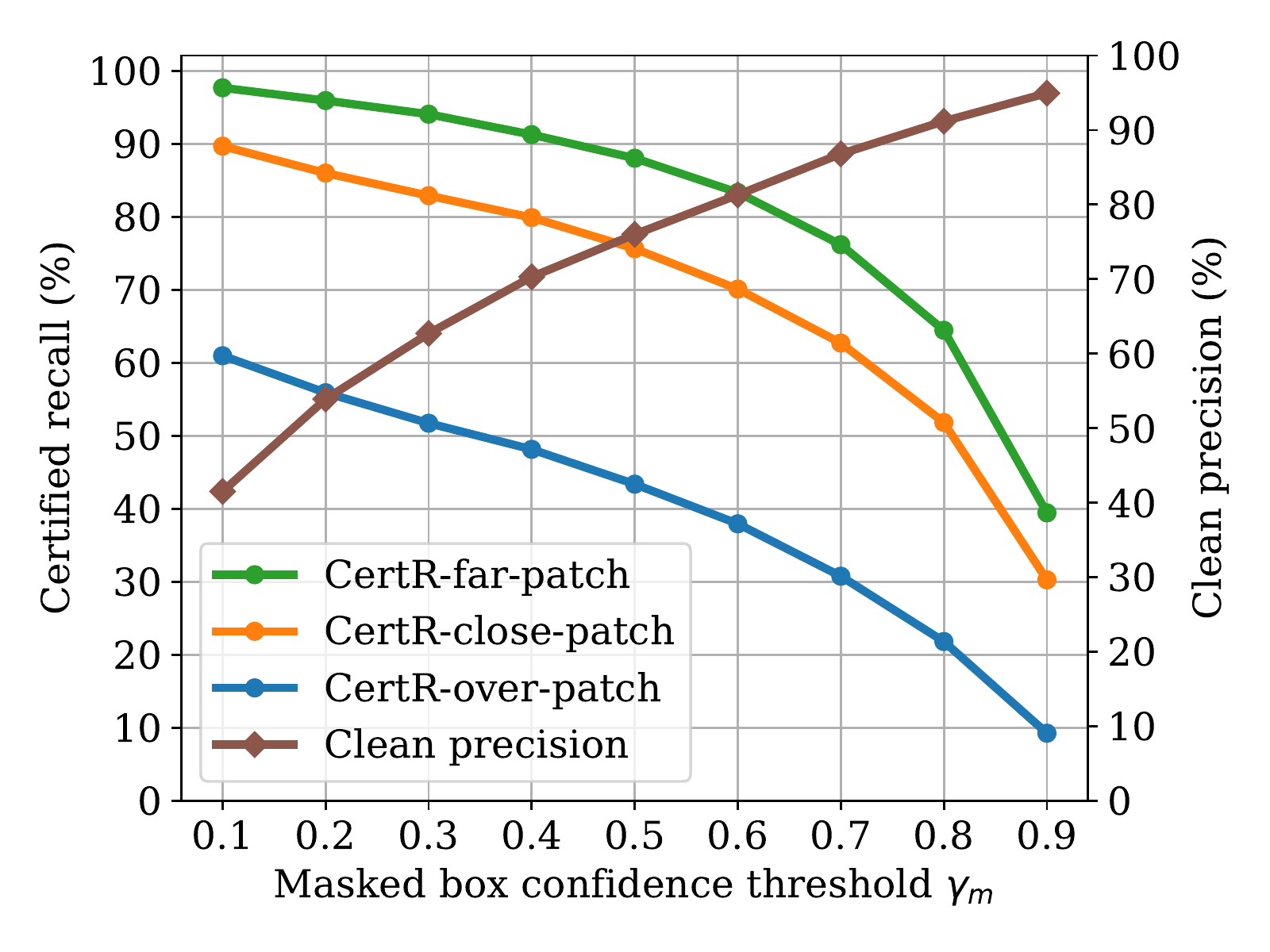}
\vspace{-2em}
\end{minipage}%
\quad
\begin{minipage}[b]{0.32\linewidth}
\includegraphics[width=\linewidth]{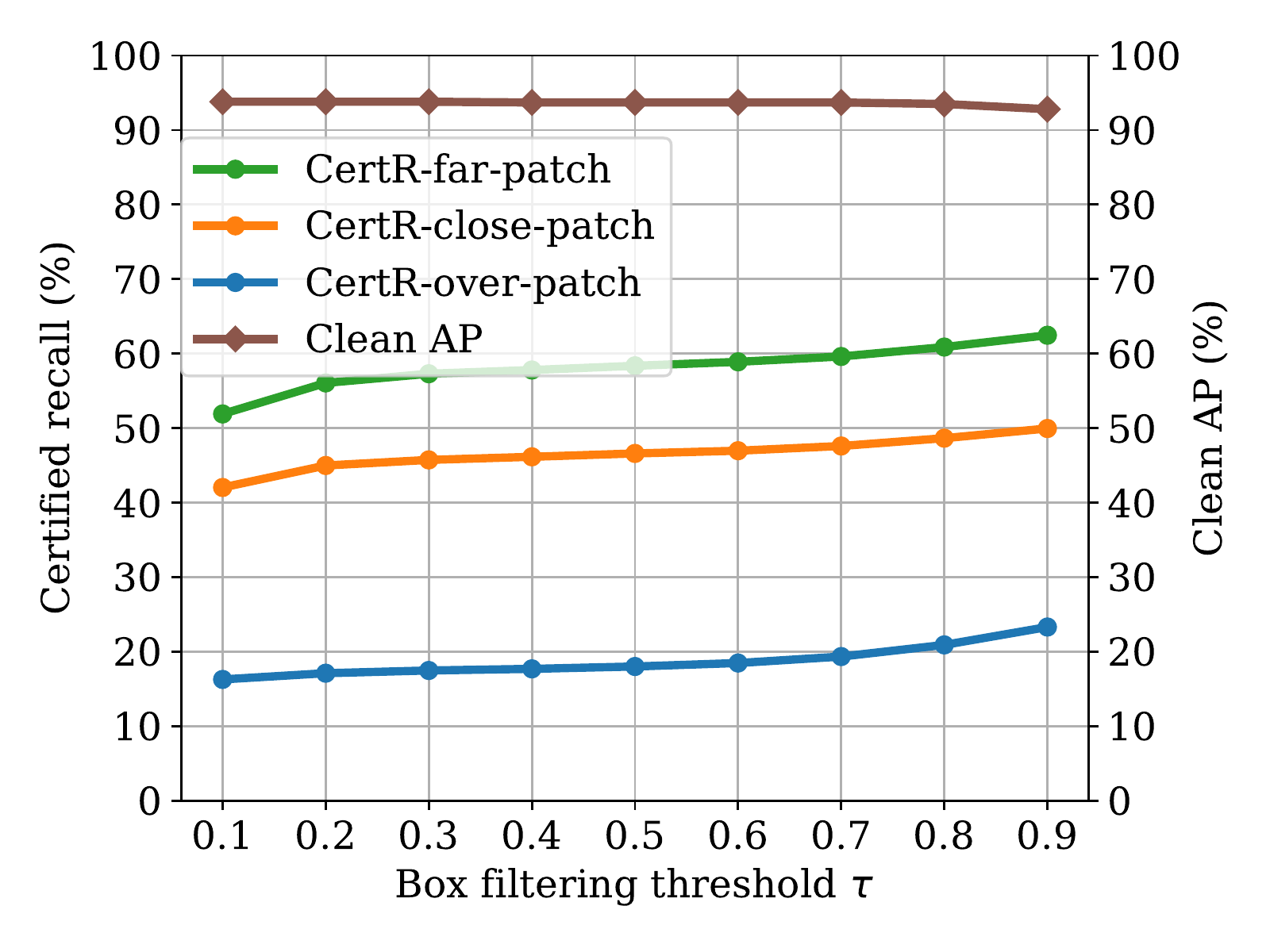}
\vspace{-2em}
\end{minipage}%
\caption{The effects of different hyperparameters on two disjoint data subsets (YOLOR and VOC). Upper row: subset 1; lower row: subset 2. From left to right: $k,\gammam,\tau$.}
\label{fig-split}
\end{figure*}

\textbf{Additional discussions on certification threshold $T$.} As discussed in Section~\ref{sec-formulation-defense} and Section~\ref{sec-defense-certification}, the threshold $T$ determines the strength of certification. In Section~\ref{sec-eval}, we set the default $T=0$ to enable a fair comparison with DetectorGuard~\cite{xiang2021detectorguard}. We note that Definition~\ref{dfn-certifiable-ioa} requires $\text{IoA}>T$ (strict inequality); thus, using $T=0$ is non-trivial as it requires that we can at least detect any tiny part of the object. We also note that DetectorGuard is limited to $T=0$ while \framework is compatible with non-zero T (we reported CertR with different T in Figure~\ref{fig-certify-voc} in Section~\ref{sec-eval-detailed}). Here, in Figure~\ref{fig-larger-T-voc}, we re-analyze the effects of the three most important defense hyperparameters ($k,\gammam,\tau$) on CertR with a larger T of 0.2. As shown in the figure, the trends of curves for $k$ and $\gammam$ largely stay the same as their counterparts of $T=0$ (Figures~\ref{fig-k-voc} and \ref{fig-mconf-voc}); the only difference is that the curves shift down a bit due to the stronger certification requirement. For the box filtering threshold $\tau$, we can see that CertR drops drastically when we use smaller $\tau$, compared to Figure~\ref{fig-tau-voc}. This is because the certification bound $ \mathbb{L}_{\textsc{IoA}}(\bfbgt,\bfbm,\tau) = ({|\bfbm|\cdot \tau - |\bfbm\setminus\bfbgt|})/{|\bfbgt|}$ becomes too low with a small $\tau$, and thus makes it hard to certify for a non-zero $T$.

\textit{Remark: the choice of $T$ in practice.} The semantic meaning of T for IoA robustness is how much of the object can be detected. In practice, the choice of $T$ for robustness \textit{evaluation} should depend on the application. For example, if we want to perform object counting or traffic sign detection/recognition, $T=0$ could be good enough. However, if we consider a robot safely navigating through a large number of large obstacles (without any collision), we might want to use a larger non-zero T. Moreover, it is also reasonable to consider different T at the same time. For example, we can report CertR for different T as done in Figure~\ref{fig-certify-voc}. Another option is to report the averaged CertR across different T. The reported number is approximately the area under the curve (AUC) for Figure~\ref{fig-certify-voc} (34.5\% for far-patch; 25.7\% for close-patch; 4.35\% for over-patch).

\textbf{Additional discussion on dataset splits and hyperparameter selection.} In this analysis, we discuss the hyperparameter selection process. First, we note that \framework does not involve any special training; we directly apply \framework to vanilla object detectors. Therefore, we only need to select hyperparameters to instantiate \framework. To select hyperparameters, we plot the curves similar to Figures~\ref{fig-k-voc} and \ref{fig-tau-voc} using the validation set, and then pick a reasonable point on the curve as the default hyperparameter.

Second, in Figure~\ref{fig-split}, we randomly split the test data into two disjoint subsets and report analysis results for hyperparameters $k,\gammam,\tau$. We can see that the plots for different hyperparameters look almost identical for two disjoint subsets, and these plots are almost identical to similar to Figures~\ref{fig-k-voc}, \ref{fig-mconf-voc}, and \ref{fig-tau-voc} in Section~\ref{sec-eval-detailed}. This demonstrates that the hyperparameters used in the paper are not overfitted to the test set. 

\section{Additional Evaluation for More Datasets}\label{apx-eval}

\begin{table*}[t]
    \centering
    \caption{Performance of vanilla undefended models, \framework, and DetectorGuard~\cite{xiang2021detectorguard} on KITTI~\cite{kitti}}  \label{tab-eval-kitti}
    \vspace{-1em}
    \resizebox{\linewidth}{!}
{ \small
\begin{tabular}{l|c|c|c|c|c|c|c|c|c|c|c}
    \toprule
   & &\multicolumn{5}{c|}{YOLOR~\cite{yolor}}&\multicolumn{5}{c}{Swin~\cite{liu2021swin}}\\
    \midrule
    
 \multirow{2}{*}{\diagbox{Model}{Metric}}&Certify &\multirow{2}{*}{$\text{AP}_{0.5}$} & \multirow{2}{*}{FAR}   &  \multicolumn{3}{c|}{Certified recall (@0.8)} &\multirow{2}{*}{$\text{AP}_{0.5}$} & \multirow{2}{*}{FAR}   &  \multicolumn{3}{c}{Certified recall (@0.8)}\\
 
  &class?&&&{far-patch}& {close-patch}& {over-patch} &&&{far-patch}& {close-patch}& {over-patch} \\
    \midrule
 Vanilla (undefended) & -- & 93.2\%&--&--&--&--&89.5\%&--&--&--&-- \\
 \framework&{\cmark}&92.9\%&--&82.3\% &48.9\% & 10.3\%&89.4\%&--&68.8\%&43.7\%&10.5\%\\
 \framework&{\xmark}&93.0\%&--&82.6\%& 49.0\%& 10.3\%&89.4\%&--&69.0\%& 43.8\%&10.6\%\\
  DetectorGuard~\cite{xiang2021detectorguard}&{\xmark}&93.0\%&0.2\%&24.2\%&8.3\%&0.7\%&89.4\%&0.1\%&23.9\%&8.3\%&0.7\%\\
  
      \bottomrule
    \end{tabular}}
  
\end{table*}

\begin{figure*}
\centering
\begin{minipage}[b]{0.32\linewidth}
\includegraphics[width=\linewidth]{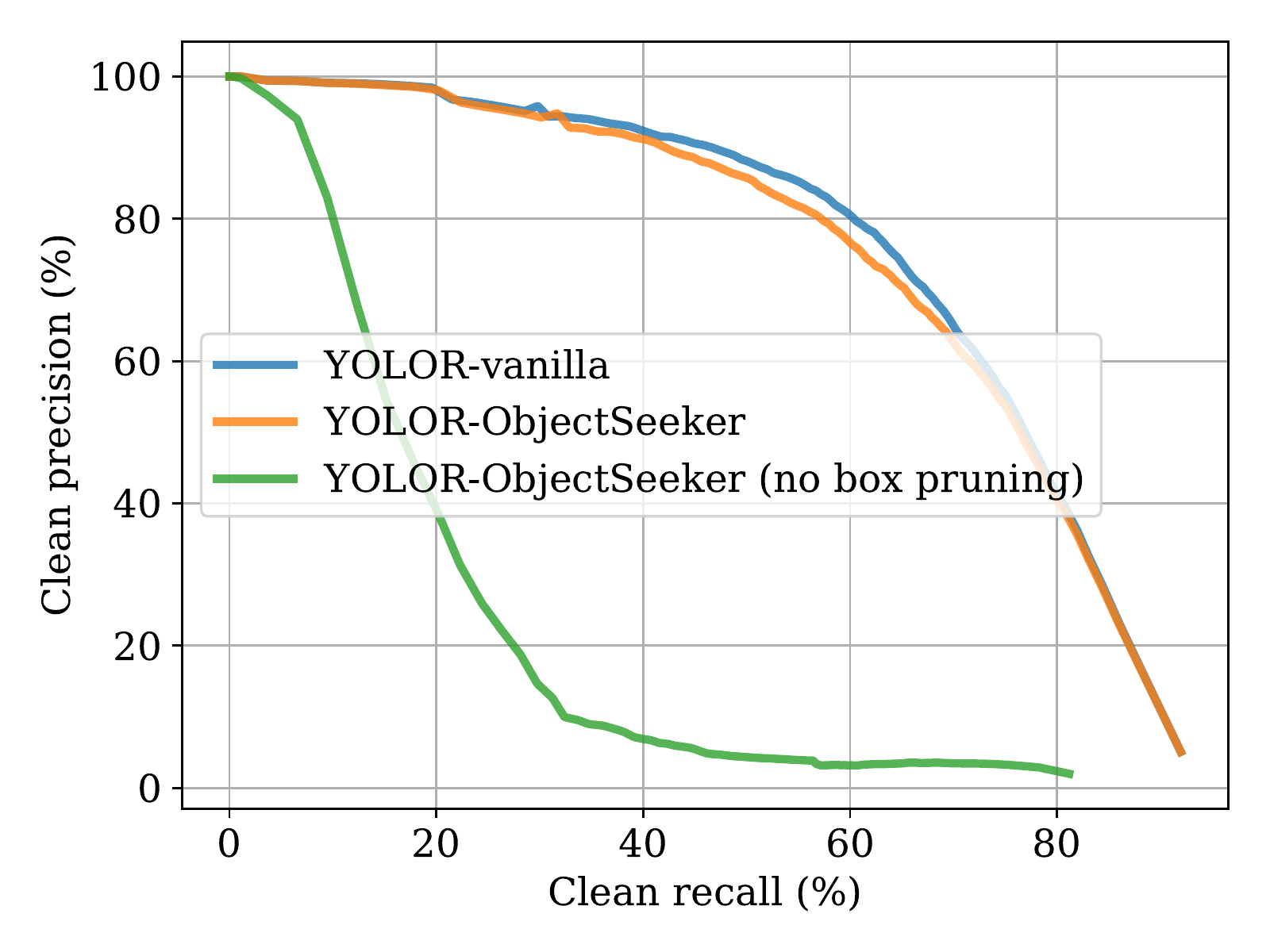}
\vspace{-2em}
    \caption{Clean precision vs. clean recall (COCO)}
    \label{fig-prec-coco}
\end{minipage}%
\quad
\begin{minipage}[b]{0.32\linewidth}
\includegraphics[width=\linewidth]{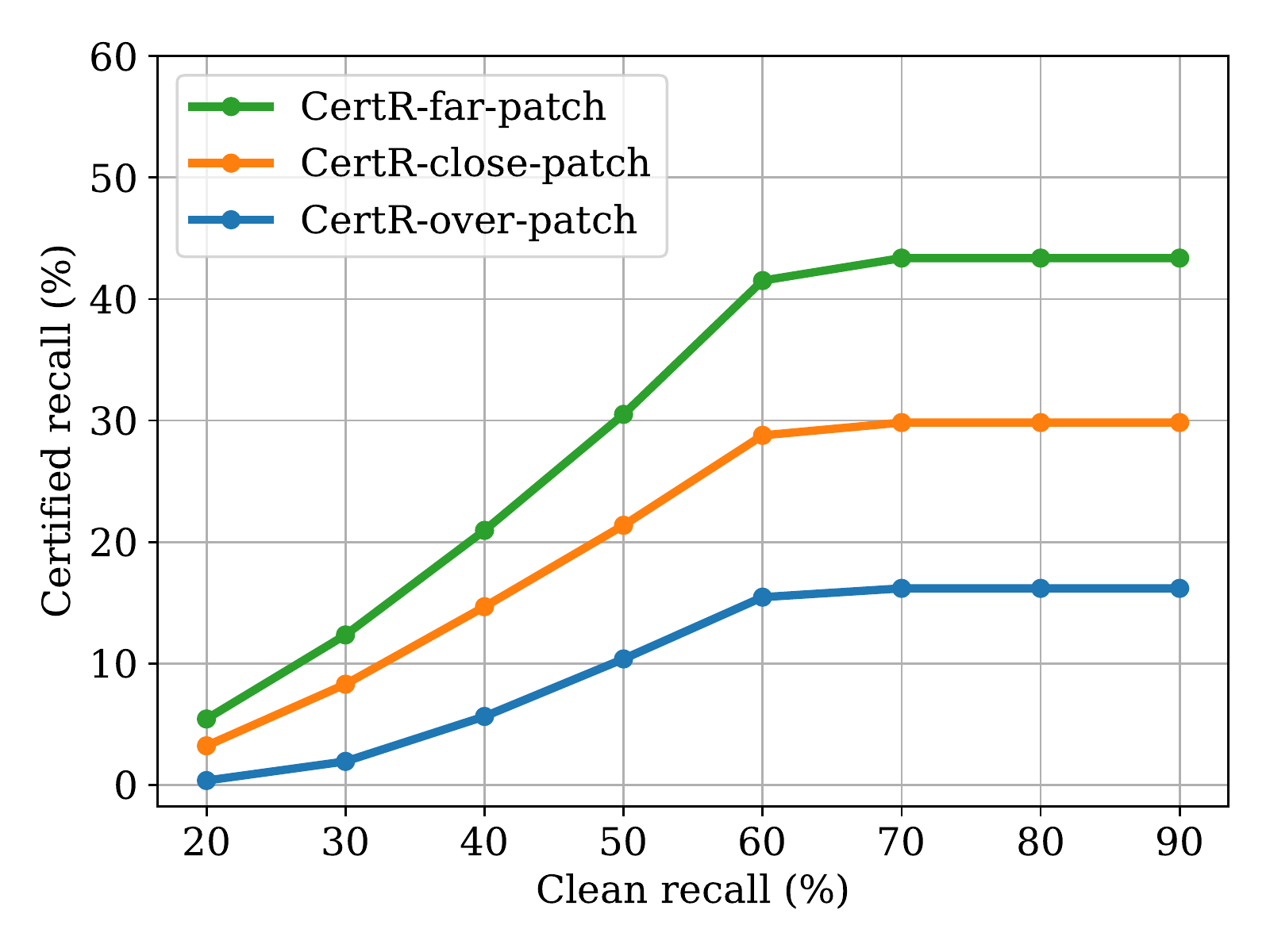}
\vspace{-2em}
    \caption{Certified recall vs. clean recall (COCO)}
    \label{fig-cr-coco}
\end{minipage}%
\quad
\begin{minipage}[b]{0.32\linewidth}
\includegraphics[width=\linewidth]{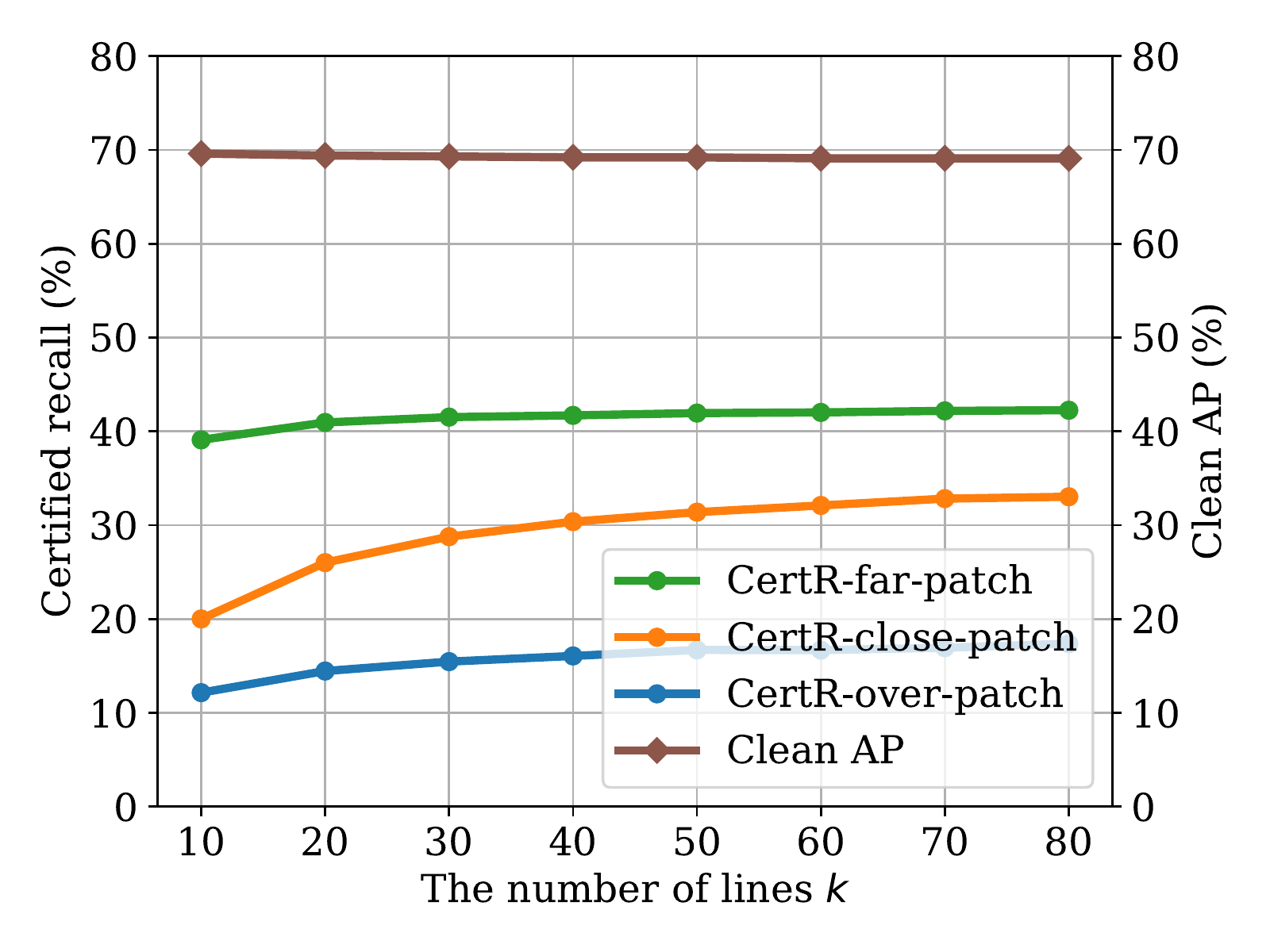}
\vspace{-2em}
    \caption{Effect of the number of lines $k$ (COCO)}
    \label{fig-k-coco}
\end{minipage}%
\end{figure*}

\begin{figure*}
\centering
\begin{minipage}[b]{0.32\linewidth}
\includegraphics[width=\linewidth]{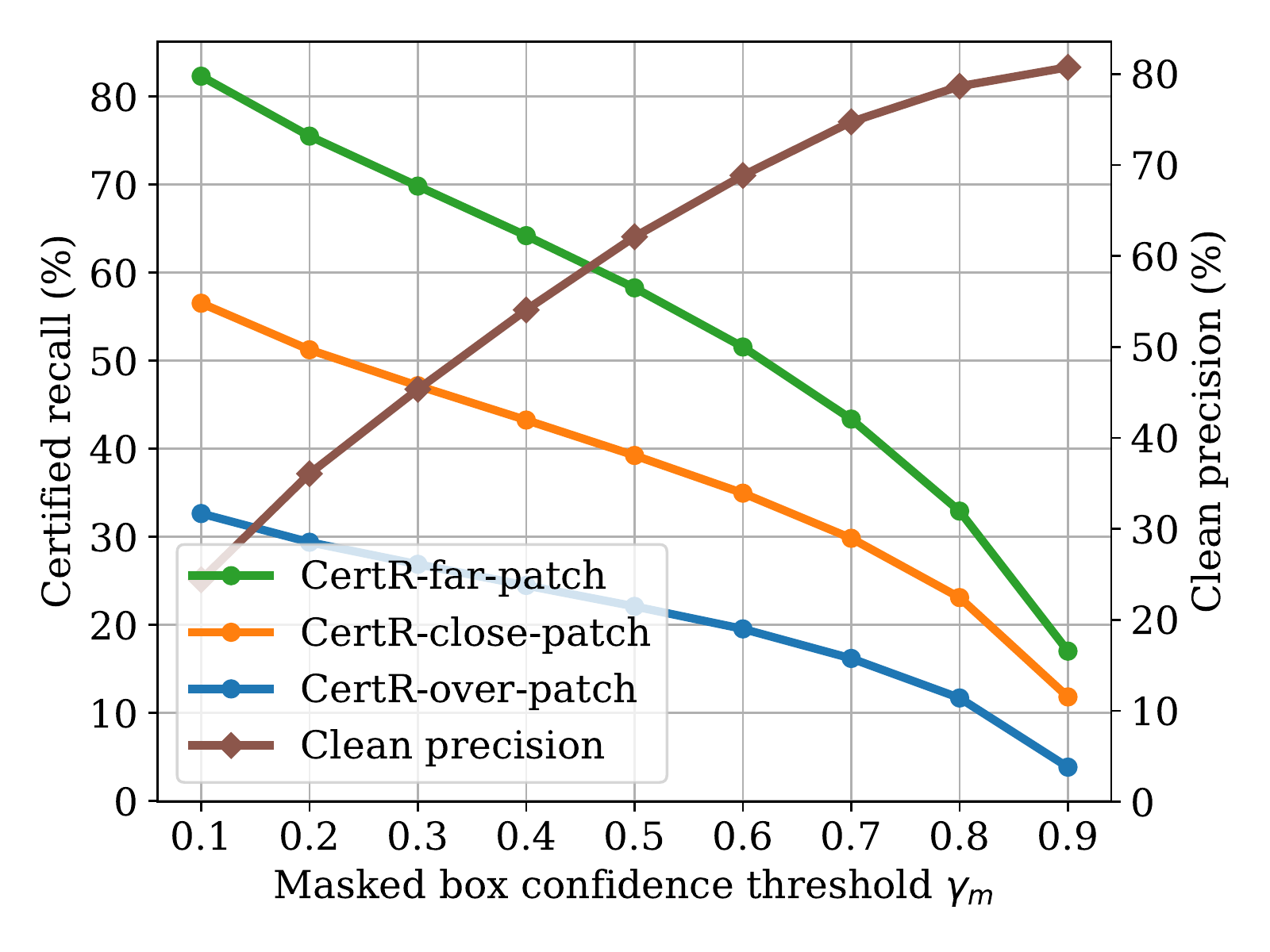}
\vspace{-2em}
    \caption{Effect of masked box confidence threshold $\gammam$ (COCO)}
    \label{fig-mconf-coco}
\end{minipage}%
\quad
\begin{minipage}[b]{0.32\linewidth}
\includegraphics[width=\linewidth]{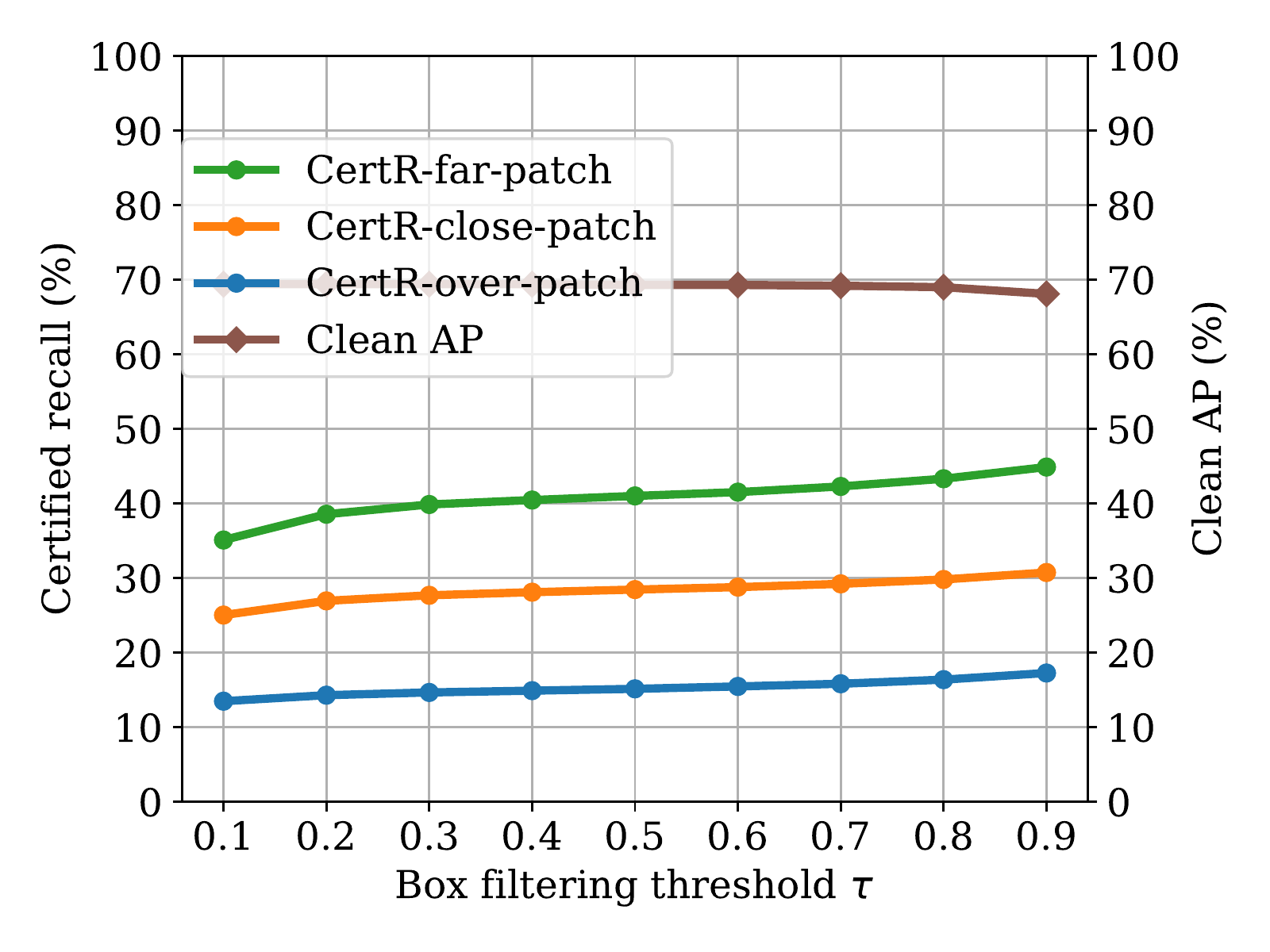}
\vspace{-2em}
    \caption{Effect of box filtering threshold $\tau$ (COCO)}
    \label{fig-tau-coco}
\end{minipage}%
\quad
\begin{minipage}[b]{0.32\linewidth}
\includegraphics[width=\linewidth]{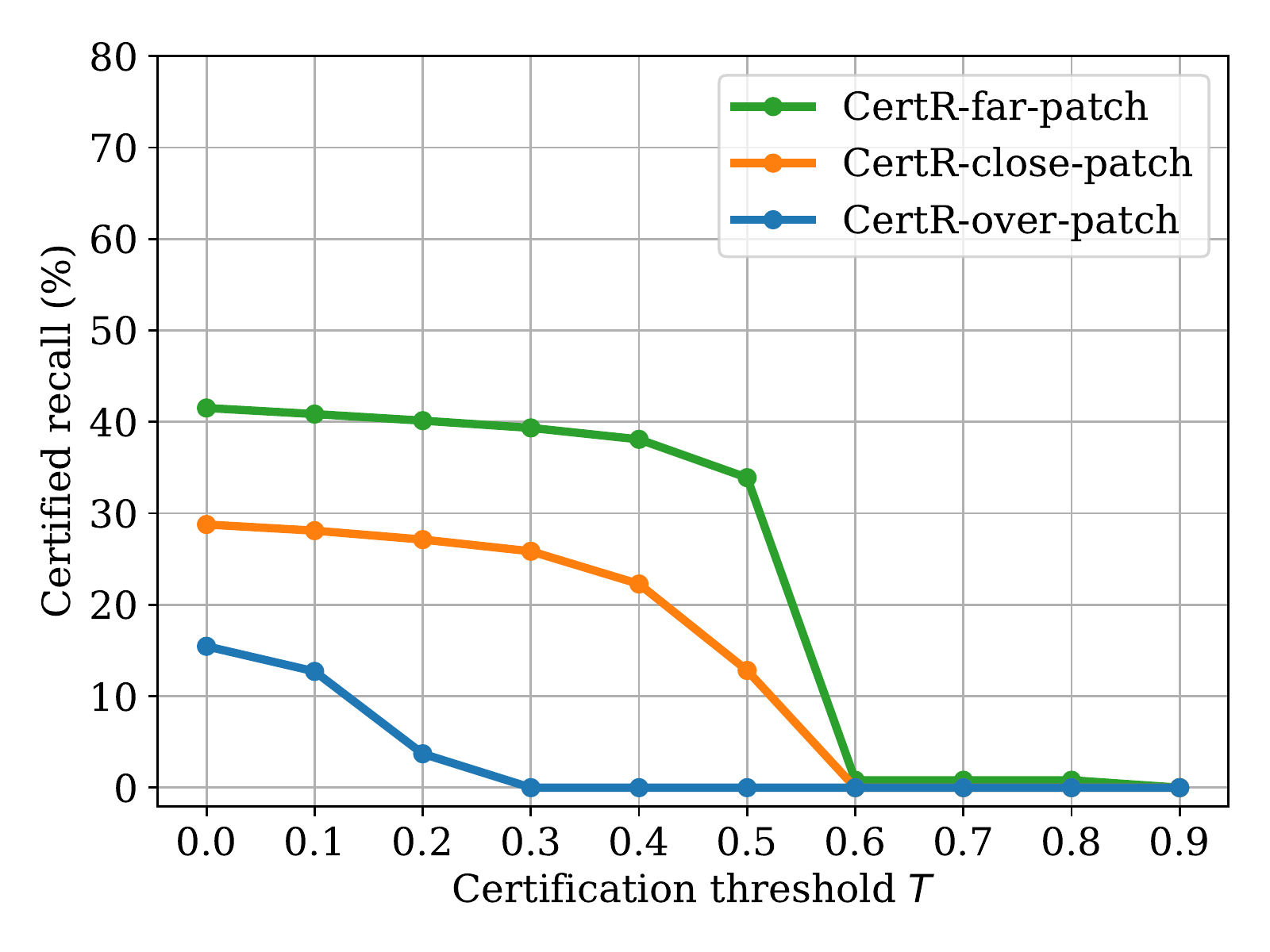}
\vspace{-2em}
    \caption{\framework performance for different certification thresholds $T$ (COCO)}
    \label{fig-certify-coco}
\end{minipage}%
\end{figure*}

\begin{figure}[t]
    \centering
    \includegraphics[width=0.8\linewidth]{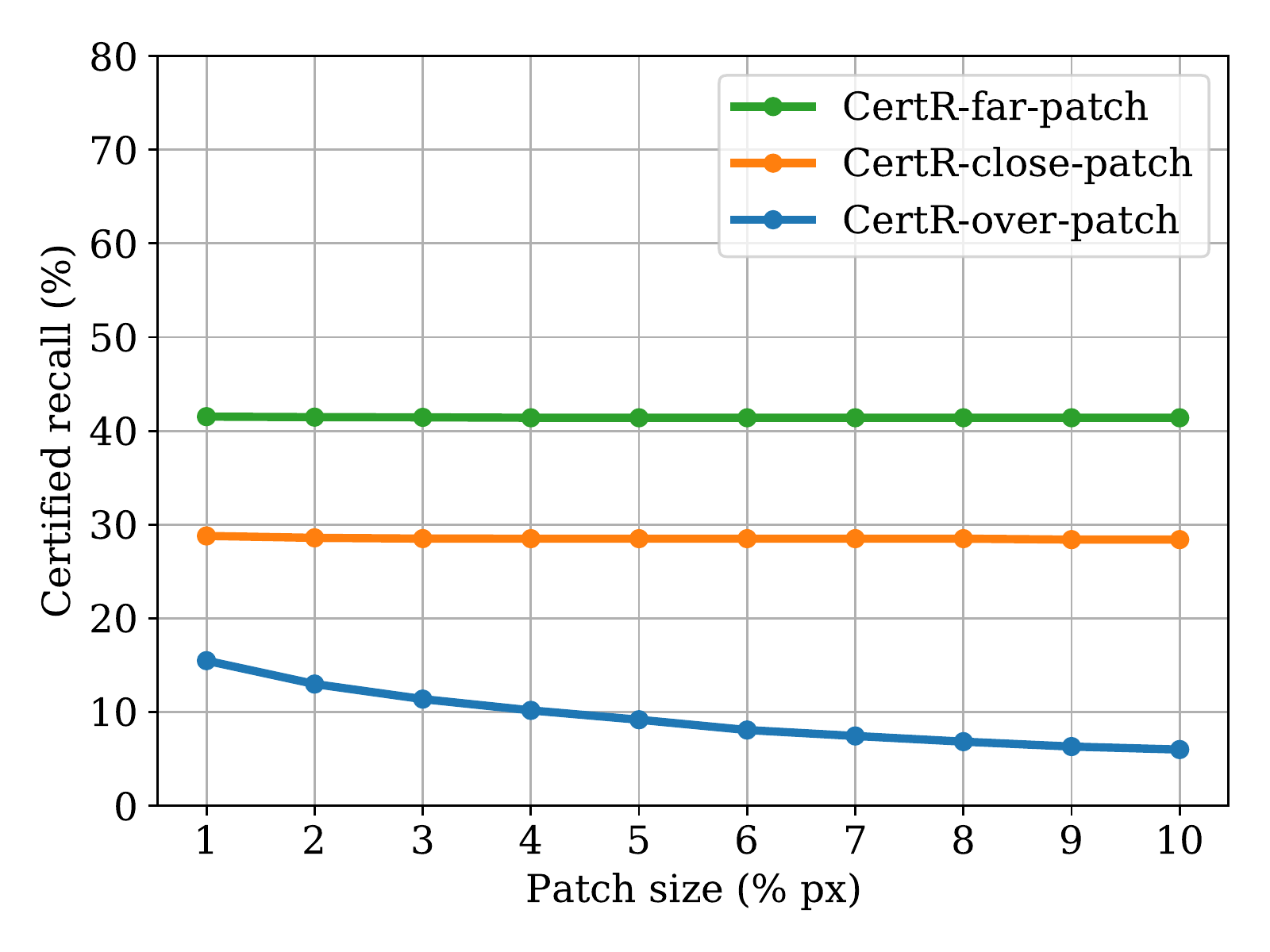}
    \vspace{-1em}
    \caption{\framework performance against different patch sizes (COCO)}
    \label{fig-patch-coco}
\end{figure}

In this section, we report additional experiment results for different datasets to further demonstrate the general applicability of \framework. We first report defense performance for an additional dataset KITTI~\cite{kitti}, we then include detailed analysis for YOLOR and COCO (similar to what we have in Section~\ref{sec-eval-detailed} for YOLOR and VOC).

\textbf{Additional Evaluation Results for KITTI.}
In Section~\ref{sec-eval}, we report evaluation results on the VOC~\cite{voc} and COCO~\cite{coco} datasets and demonstrate significant improvements in certified recalls over DetectorGuard~\cite{xiang2021detectorguard}. In this section, we report defense performance for an additional dataset KITTI~\cite{kitti}.

KITTI is a dataset for autonomous driving applications, which contains both 2D camera images and 3D point clouds. Following DetectorGuard~\cite{xiang2021detectorguard}. We use 80\% of its 7481 2D images for training and the remaining 20\% for validation. We merge all classes into three classes: \texttt{car} (all different classes of vehicles), \texttt{pedestrian}, \texttt{cyclist}. We use the same set of defense parameters of VOC to instantiate \framework for KITTI. We then report the defense performance of \framework and DetectorGuard~\cite{xiang2021detectorguard} in Table~\ref{tab-eval-kitti}. As shown in the table, \framework has a similarly high clean performance as vanilla undefended models and achieves significantly higher certified recall than DetectorGuard~\cite{xiang2021detectorguard}; the observation is similar to that of Table~\ref{tab-main-eval} in Section~\ref{sec-eval-main}. We note that the CertR@0.8 of \framework-YOLOR for far-patch becomes larger than 80\% (the clean recall value); this is possible because we are setting the certification threshold $T=0$, which is an easier condition to satisfy compared to the condition for clean evaluation (requiring IoU is larger than 0.5).

\textbf{Additional experiment results for COCO.} In Figure~\ref{fig-prec-coco} and Figure~\ref{fig-cr-coco}, we plot the clean precision-recall curve and the CertR-recall curve for COCO. The observation is similar to that for VOC in Section~\ref{sec-eval-detailed}. We note that the certified recall of COCO does not further increase as clean recall exceeds 70\%. This is because, the masked box confidence threshold $\gammam = \max(\alpha,\gammab + (1-\gammab)\cdot \beta)$ starts to take the value $\alpha$ when $\gammab$ takes a low value. 

In Figure~\ref{fig-k-coco}, Figure~\ref{fig-mconf-coco}, and Figure~\ref{fig-tau-coco}, we report defense performance with different defense parameters $k,\gammam,\tau$. In Figure~\ref{fig-certify-coco} and Figure~\ref{fig-patch-coco}, we further report certified robustness with a larger certification threshold $T$ and against a larger patch. The observations from these figures are similar to those reported for VOC in Section~\ref{sec-eval-detailed}. This further demonstrates that \framework works well for both easier and harder object detection tasks.

\end{document}